\documentclass[12pt]{article}
\usepackage{amsmath}
\usepackage{mathtools}
\usepackage{amsfonts}
\usepackage{tikz}
\usepackage{booktabs, multirow} 
\usepackage[letterpaper,top=1in,bottom=1in,left=1in,right=1in]{geometry}
\usepackage{bm}
\usepackage{float}
\usepackage{subfig}
\usepackage{graphicx}
\usepackage[shortlabels]{enumitem}
\usepackage{url}
\usepackage{xcolor}
\usepackage{amsthm}
\usepackage{multirow}

\renewcommand\a{\bm{a}}
\newcommand\bmu{\bm{\mu}}
\renewcommand\b{\bm{b}}

\newcommand\q{\bm{q}}
\newcommand\w{\bm{w}}
\newcommand\x{\bm{x}}
\newcommand\y{\bm{y}}

\renewcommand\v{\bm{v}}

\newcommand{\LF}{\mathrm{LF}}
\newcommand{\R}{\mathbb{R}}
\newcommand{\Diag}{\mathop{\mathrm{Diag}}}
\newcommand{\rank}{\mathop{\mathrm{rank}}}
\newcommand{\eps}{\epsilon}
\newcommand{\be}{\begin{equation}}
\newcommand{\ee}{\end{equation}}
\newcommand{\vol}{\mathrm{vol}}
\newcommand{\supp}{\mathrm{supp}}
\newcommand{\dist}{\mathrm{dist}}

\newcommand{\norm}[1] {\left \| #1 \right \|} 

\newtheorem{theorem}{Theorem}

\newtheorem{lemma}[theorem]{Lemma}

\newtheorem{claim}{Claim}

\newtheorem{manualtheoreminner}{Theorem}
\newenvironment{manualtheorem}[1]{%
  \renewcommand\themanualtheoreminner{#1}%
  \manualtheoreminner
}{\endmanualtheoreminner}

\title{Re-embedding data to strengthen recovery guarantees of clustering}

\author{Tao Jiang\thanks{School of Operations Research and Information Engineering, Cornell University, Ithaca, New York, USA 14850, {\tt tj293@cornell.edu}.} \and
Samuel Tan\thanks{School of Operations Research and Information Engineering, Cornell University, Ithaca, New York, USA 14850, {\tt sst76@cornell.edu}.} \and
Stephen Vavasis\thanks{Department of Combinatorics \&  Optimization, University of Waterloo, Waterloo, Ontario,  Canada N2L 3G1, {\tt vavasis@uwaterloo.ca}.  Research supported in part by a Discovery Grant from the Natural Science and Engineering Research Council (NSERC) of Canada.}}
\date{}

\begin{document}

\maketitle

\begin{abstract}
We propose a clustering method that involves chaining four known techniques into a pipeline yielding an algorithm with stronger recovery guarantees than any of the four components separately.  Given $n$ points in $\R^d$, the first component of our pipeline, which we call leapfrog distances, is reminiscent of density-based clustering, yielding an $n\times n$ distance matrix.  The leapfrog distances are then translated to new embeddings using multidimensional scaling and spectral methods, two other known techniques, yielding new embeddings of the $n$ points in $\R^{d'}$, where $d'$ satisfies $d'\ll d$ in general.  Finally, sum-of-norms (SON) clustering is applied to the re-embedded points.  Although the fourth step (SON clustering) can in principle be replaced by any other clustering method, our focus is on provable guarantees of recovery of underlying structure.
Therefore, we establish that the re-embedding improves recovery SON clustering, since SON clustering is a well-studied method that already has provable guarantees.

\end{abstract}

\section{Introduction}

Clustering is a well-established unsupervised learning technique that, informally speaking, partitions a group of $n$ points denoted $\a_1,\ldots,\a_n\in \mathbb{R}^d$ into $K$ ``clusters'' such that the intra-cluster distances are small while inter-cluster distances are large.

Sum-of-norms (SON) clustering is a method of formulating the above problem as a convex optimization problem, proposed independently by Pelckmans et al.~\cite{pelckmans}, Hocking et al.~\cite{hocking}, and Lindsten et al.~\cite{lindsten}. The formulation is the following:
\begin{equation}
    \min_{\x_1,\ldots,\x_n\in\R^d} \frac{1}{2}\sum_{i=1}^n \norm{\x_i-\a_i}^2 +\lambda\sum_{1 \le i<j \le n}\norm{\x_i-\x_j}.
    \label{eq:son-clustering}
\end{equation}

We interpret the solution $\x_1^*,\ldots,\x_n^*$ as the cluster centroids.  When $\x_i^* = \x_{i'}^*$  points $i$ and $i'$ are in the same cluster, and otherwise not. The parameter $\lambda>0$ in \eqref{eq:son-clustering} controls the number of clusters.  When $\lambda = 0$, every point is its own cluster.  In the opposite direction, there is a finite data-dependent $\bar\lambda$ such that the solution to \eqref{eq:son-clustering} for $\lambda\ge\bar\lambda$ yields a single cluster, i.e., a solution in which $\x_1^*=\cdots=\x_n^*$.  It is known \cite{chiquet} that clusters merge but never break apart as $\lambda$ increases from $0$ to $\bar\lambda$, and thus \eqref{eq:son-clustering} encodes a clustering hierarchy.

Much of the recent work on SON clustering is concerned with recovery guarantees. When clusters are generated by well-separated hypercubes, Zhu et al.~\cite{Zhu} shows perfect recovery. Tan and Witten~\cite{tan2015} draws connections between sum-of-norms clustering and single-linkage hierarchical and $k$-means clustering. Panahi et al.~\cite{Panahi} and Jiang et al.~\cite{jiang2020recovery} analyze the case when clusters are assumed to be mixtures of Gaussians. Efficient algorithms are developed in Chi and Lange~\cite{chiLange} and Sun et al.~\cite{dsun1}.

On the other hand, other recent work has pointed to limitations of SON clustering.
Dunlap and Mourrat~\cite{dunlap2021local} consider the asymptotic limit as the number of points scale to infinity, which yields the related problem of clustering of measures. They prove that even when the points are sampled uniformly at random from two disjoint unit closed balls, SON clustering is not able to distinguish the two balls unless there is a positive gap $\rho_d$ separating the balls.
Nguyen and Mamitsuka~\cite{2021nguyen} show that SON clustering can  find only convex clusters, i.e. clusters for which the interiors of their convex hulls are disjoint.

The negative results in previous paragraph motivate this paper.  We develop a method to re-embed $\a_1,\ldots,\a_n$ with new coordinates $\b_1,\ldots,\b_n\in\R^{d'}$ that unravels nonconvex clusters into convex ones and separates the peaks of points sampled from Gaussian mixtures so that SON clustering can succeed on the re-embedded points.

The first component of our approach is to make a density-based distance matrix using {\em leapfrog} distances, introduced in
Section \ref{sec:LF_dis}. Section \ref{sec:lf_property} develops bounds for the leapfrog distance of points in one dimension sampled from a distribution.  Section \ref{sec:LF_Rd} subsequently derives bounds on the leapfrog distance in $\mathbb{R}^d$. Section \ref{sec:dist2coord} describes the second and third components of our approach, namely multidimensional scaling and spectral coordinates.
Section \ref{sec:b_property} derives properties of the new embedded coordinates.  The final component of our pipeline, sum-of-norms clustering, has known theoretical guarantees described in
Section \ref{sec:chiquetproof}, which leads to our new results on guaranteed recovery for the re-embedded points.  Section \ref{sec:exper} shows numerical results of our pipeline on well known data sets such as the half-moons dataset and concentric circles, which cannot be recovered using \eqref{eq:son-clustering} directly.  Section \ref{sec:conc} concludes and comments on potential future work.




We conclude this introduction by noting that Dunlap and Mourrat~\cite{dunlap2021local} show that recovery of clusters supported on disjoint sets is also assured if exponentially decaying weights are introduced into the second summation of the right-hand side of \eqref{eq:son-clustering}.  However, introduction of weights apparently loses some of the advantages of SON clustering.  For example, it is not known whether exponential weights yield a hierarchical clustering as $\lambda$ increases; the ADMM algorithm of Chi and Lange~\cite{chiLange} can no longer use the efficient Sherman-Morrison-Woodbury formula and apparently must resort to an $O(n^3)$ operation on each iteration; and it is not known whether exponential weights can recover mixtures of Gaussians.

\section{From coordinates to distances: leapfrog distance}
\label{sec:LF_dis}
In this section, we introduce and motivate the ``leapfrog'' distance.
Leapfrog distance is similar in spirit to density scanning as in the well-known DBScan clustering algorithm \cite{ester1996density}.

Consider a complete graph $G=(\{\a_1,\ldots,\a_n\},E)$ constructed from the set of data points, and $E$, the set of undirected edges which connect every pair of distinct vertices. For each edge $(i,j) \in E$, we assign a cost of $c_{i,j}:= \norm{\a_i - \a_j}^2$ to be the squared Euclidean distance between $\a_i, \a_j$. We define the leapfrog distance between $i,j$ to be the total cost of the shortest path between nodes $i,j$ on the weighted graph $(G,c)$. Figure \ref{fig:twod} illustrates an example of the leapfrog distance for five points in two dimensions.

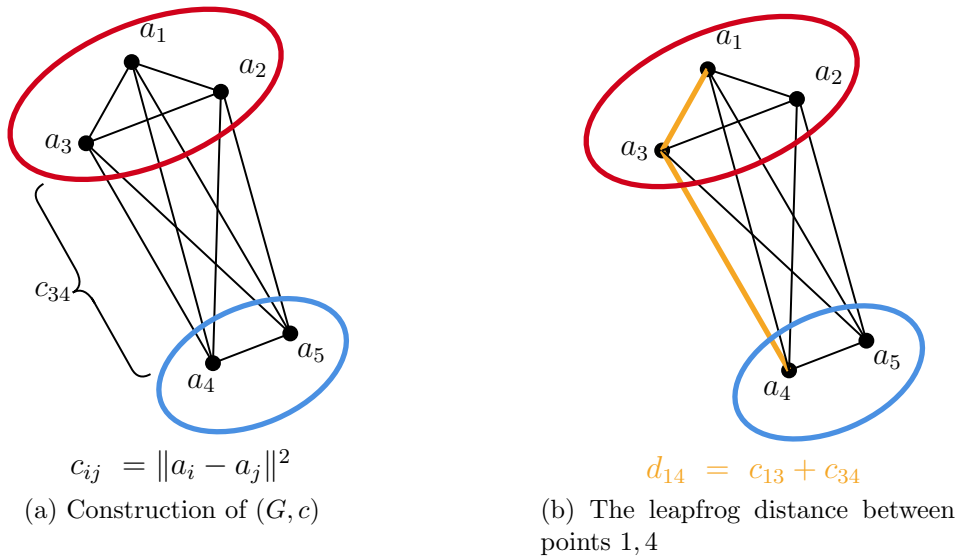
\begin{figure}[ht]
\vskip -0.05in
\centering
\subfloat[][Construction of $(G,c)$]{

\tikzset{every picture/.style={line width=0.75pt}} 
\scalebox{1.0}{
\begin{tikzpicture}[x=0.75pt,y=0.75pt,yscale=-1,xscale=1]

\draw  [fill={rgb, 255:red, 0; green, 0; blue, 0 }  ,fill opacity=1 ] (69,89.5) .. controls (69,87.57) and (70.57,86) .. (72.5,86) .. controls (74.43,86) and (76,87.57) .. (76,89.5) .. controls (76,91.43) and (74.43,93) .. (72.5,93) .. controls (70.57,93) and (69,91.43) .. (69,89.5) -- cycle ;
\draw  [fill={rgb, 255:red, 0; green, 0; blue, 0 }  ,fill opacity=1 ] (133,200.5) .. controls (133,198.57) and (134.57,197) .. (136.5,197) .. controls (138.43,197) and (140,198.57) .. (140,200.5) .. controls (140,202.43) and (138.43,204) .. (136.5,204) .. controls (134.57,204) and (133,202.43) .. (133,200.5) -- cycle ;
\draw  [fill={rgb, 255:red, 0; green, 0; blue, 0 }  ,fill opacity=1 ] (172,185.5) .. controls (172,183.57) and (173.57,182) .. (175.5,182) .. controls (177.43,182) and (179,183.57) .. (179,185.5) .. controls (179,187.43) and (177.43,189) .. (175.5,189) .. controls (173.57,189) and (172,187.43) .. (172,185.5) -- cycle ;
\draw  [fill={rgb, 255:red, 0; green, 0; blue, 0 }  ,fill opacity=1 ] (92,48.5) .. controls (92,46.57) and (93.57,45) .. (95.5,45) .. controls (97.43,45) and (99,46.57) .. (99,48.5) .. controls (99,50.43) and (97.43,52) .. (95.5,52) .. controls (93.57,52) and (92,50.43) .. (92,48.5) -- cycle ;
\draw  [color={rgb, 255:red, 0; green, 0; blue, 0 }  ,draw opacity=1 ][fill={rgb, 255:red, 0; green, 0; blue, 0 }  ,fill opacity=1 ] (137,63.5) .. controls (137,61.57) and (138.57,60) .. (140.5,60) .. controls (142.43,60) and (144,61.57) .. (144,63.5) .. controls (144,65.43) and (142.43,67) .. (140.5,67) .. controls (138.57,67) and (137,65.43) .. (137,63.5) -- cycle ;
\draw    (72.5,89.5) -- (136.5,200.5) ;
\draw    (72.5,89.5) -- (175.5,185.5) ;
\draw    (140.5,60) -- (175.5,185.5) ;
\draw    (136.5,200.5) -- (175.5,185.5) ;
\draw    (136.5,200.5) -- (95.5,48.5) ;
\draw    (72.5,89.5) -- (95.5,48.5) ;
\draw    (95.5,48.5) -- (140.5,63.5) ;
\draw    (140.5,63.5) -- (72.5,89.5) ;
\draw    (95.5,48.5) -- (175.5,185.5) ;
\draw    (140.5,63.5) -- (136.5,200.5) ;
\draw  [color={rgb, 255:red, 208; green, 2; blue, 27 }  ,draw opacity=1 ][line width=2.0]  (35.53,91.9) .. controls (28.09,73.5) and (52.05,46.45) .. (89.05,31.48) .. controls (126.05,16.51) and (162.08,19.29) .. (169.52,37.69) .. controls (176.96,56.08) and (153.01,83.14) .. (116.01,98.11) .. controls (79.01,113.08) and (42.98,110.3) .. (35.53,91.9) -- cycle ;
\draw  [color={rgb, 255:red, 74; green, 144; blue, 226 }  ,draw opacity=1 ][line width=2.0]  (110.7,221.2) .. controls (104.24,206.18) and (119.5,185.18) .. (144.8,174.3) .. controls (170.09,163.41) and (195.83,166.77) .. (202.3,181.8) .. controls (208.76,196.82) and (193.5,217.82) .. (168.2,228.7) .. controls (142.91,239.59) and (117.17,236.23) .. (110.7,221.2) -- cycle ;
\draw   (51,110) .. controls (46.93,112.29) and (46.04,115.46) .. (48.33,119.53) -- (67,152.72) .. controls (70.27,158.53) and (69.87,162.57) .. (65.8,164.86) .. controls (69.87,162.57) and (73.53,164.34) .. (76.8,170.15)(75.33,167.53) -- (95.47,203.33) .. controls (97.76,207.4) and (100.93,208.29) .. (105,206) ;

\draw (98,27.4) node [anchor=north west][inner sep=0.75pt]    {$a_{1}$};
\draw (148,45.4) node [anchor=north west][inner sep=0.75pt]    {$a_{2}$};
\draw (50,83.4) node [anchor=north west][inner sep=0.75pt]    {$a_{3}$};
\draw (122,204.4) node [anchor=north west][inner sep=0.75pt]    {$a_{4}$};
\draw (177.5,188.9) node [anchor=north west][inner sep=0.75pt]    {$a_{5}$};
\draw (63,242.4) node [anchor=north west][inner sep=0.75pt]    {$c_{ij} \ =\| a_{i} -a_{j} \| ^{2} \ $};
\draw (45,158.4) node [anchor=north west][inner sep=0.75pt]    {$c_{34}$};

\end{tikzpicture}
}
}
\hspace{2.0cm}
\subfloat[][The leapfrog distance between points $1,4$]{
\tikzset{every picture/.style={line width=0.75pt}} 
\scalebox{1.0}{
\begin{tikzpicture}[x=0.75pt,y=0.75pt,yscale=-1,xscale=1]

\draw  [fill={rgb, 255:red, 0; green, 0; blue, 0 }  ,fill opacity=1 ] (69,89.5) .. controls (69,87.57) and (70.57,86) .. (72.5,86) .. controls (74.43,86) and (76,87.57) .. (76,89.5) .. controls (76,91.43) and (74.43,93) .. (72.5,93) .. controls (70.57,93) and (69,91.43) .. (69,89.5) -- cycle ;
\draw  [fill={rgb, 255:red, 0; green, 0; blue, 0 }  ,fill opacity=1 ] (133,200.5) .. controls (133,198.57) and (134.57,197) .. (136.5,197) .. controls (138.43,197) and (140,198.57) .. (140,200.5) .. controls (140,202.43) and (138.43,204) .. (136.5,204) .. controls (134.57,204) and (133,202.43) .. (133,200.5) -- cycle ;
\draw  [fill={rgb, 255:red, 0; green, 0; blue, 0 }  ,fill opacity=1 ] (172,185.5) .. controls (172,183.57) and (173.57,182) .. (175.5,182) .. controls (177.43,182) and (179,183.57) .. (179,185.5) .. controls (179,187.43) and (177.43,189) .. (175.5,189) .. controls (173.57,189) and (172,187.43) .. (172,185.5) -- cycle ;
\draw  [fill={rgb, 255:red, 0; green, 0; blue, 0 }  ,fill opacity=1 ] (92,48.5) .. controls (92,46.57) and (93.57,45) .. (95.5,45) .. controls (97.43,45) and (99,46.57) .. (99,48.5) .. controls (99,50.43) and (97.43,52) .. (95.5,52) .. controls (93.57,52) and (92,50.43) .. (92,48.5) -- cycle ;
\draw  [color={rgb, 255:red, 0; green, 0; blue, 0 }  ,draw opacity=1 ][fill={rgb, 255:red, 0; green, 0; blue, 0 }  ,fill opacity=1 ] (137,63.5) .. controls (137,61.57) and (138.57,60) .. (140.5,60) .. controls (142.43,60) and (144,61.57) .. (144,63.5) .. controls (144,65.43) and (142.43,67) .. (140.5,67) .. controls (138.57,67) and (137,65.43) .. (137,63.5) -- cycle ;
\draw [color={rgb, 255:red, 245; green, 166; blue, 35 }  ,draw opacity=1 ][line width=2.0]    (72.5,89.5) -- (136.5,200.5) ;
\draw    (72.5,89.5) -- (175.5,185.5) ;
\draw    (140.5,60) -- (175.5,185.5) ;
\draw    (136.5,200.5) -- (175.5,185.5) ;
\draw    (136.5,200.5) -- (95.5,48.5) ;
\draw [color={rgb, 255:red, 245; green, 166; blue, 35 }  ,draw opacity=1 ][line width=2.0]    (72.5,89.5) -- (95.5,48.5) ;
\draw    (95.5,48.5) -- (140.5,63.5) ;
\draw    (140.5,63.5) -- (72.5,89.5) ;
\draw    (95.5,48.5) -- (175.5,185.5) ;
\draw    (140.5,63.5) -- (136.5,200.5) ;
\draw  [color={rgb, 255:red, 208; green, 2; blue, 27 }  ,draw opacity=1 ][line width=2.0]  (35.53,91.9) .. controls (28.09,73.5) and (52.05,46.45) .. (89.05,31.48) .. controls (126.05,16.51) and (162.08,19.29) .. (169.52,37.69) .. controls (176.96,56.08) and (153.01,83.14) .. (116.01,98.11) .. controls (79.01,113.08) and (42.98,110.3) .. (35.53,91.9) -- cycle ;
\draw  [color={rgb, 255:red, 74; green, 144; blue, 226 }  ,draw opacity=1 ][line width=2.0]  (110.7,221.2) .. controls (104.24,206.18) and (119.5,185.18) .. (144.8,174.3) .. controls (170.09,163.41) and (195.83,166.77) .. (202.3,181.8) .. controls (208.76,196.82) and (193.5,217.82) .. (168.2,228.7) .. controls (142.91,239.59) and (117.17,236.23) .. (110.7,221.2) -- cycle ;

\draw (98,27.4) node [anchor=north west][inner sep=0.75pt]    {$a_{1}$};
\draw (148,45.4) node [anchor=north west][inner sep=0.75pt]    {$a_{2}$};
\draw (50,83.4) node [anchor=north west][inner sep=0.75pt]    {$a_{3}$};
\draw (122,204.4) node [anchor=north west][inner sep=0.75pt]    {$a_{4}$};
\draw (177.5,188.9) node [anchor=north west][inner sep=0.75pt]    {$a_{5}$};
\draw (63,242.4) node [anchor=north west][inner sep=0.75pt]    {$\textcolor[rgb]{0.96,0.65,0.14}{d_{14} \ =\ c_{13} +c_{34}}$};
\end{tikzpicture}
}
}
\vskip 0.1in
\caption{Example of constructing the leapfrog distance for 5 datapoints in 2D. (a) Step 1: construct a complete graph with edge cost $c_{i,j}=\norm{\a_i-\a_j}^2$ to be the squared Euclidean distance between the two endpoints. (b) Step 2: define the leapfrog distance to be the total cost of the shortest path on $(G,c)$. In our example, even though the Euclidean distance between points 1 and 4 is $\norm{\a_1-\a_4}$ and corresponds to the length of the line segment between $\a_1,\a_4$, the leapfrog distance between point 1 and point 4 is $\norm{\a_1-\a_3}^2 + \norm{\a_3-\a_4}^2$ and corresponds to the path $(\a_1,\a_3),(\a_3,\a_4)$.}
\label{fig:twod}
\vskip 0.1in
\end{figure}

The name ``leapfrog'' comes from the following thought experiment. Imagine  a frog tries to get from point $i$ to $j$ on the weighted graph $(G,c)$. Due to the definition of our costs $c$, our frog prefers a path with many small leaps over one consisting of a few large leaps when minimizing the total cost of the trip. The leapfrog distance corresponds to the path of many small leaps that the frog takes. As the number of data points increases, the neighboring points get closer. As this happens, we have simultaneously that (1) the frog makes more leaps and (2) the leaps it takes also become smaller in order to keep the total cost low. However, when the frog tries to travel through a gap, a big leap is inevitable. Such gaps correspond to the frog jumping from one cluster to another, where this gap remains large even as the sample size grows. In this way, our constructed graph $(G,c)$ separates clusters so that the frog may find paths of lower cost within clusters relative to the cost of jumping between clusters. This intuition underpins the development of our re-embedded coordinates.

Figure \ref{fig:1d_disjoint_support} makes this intuition more concrete, in the particular case of uniformly distributed points on $\left[0,\frac{1}{3} \right]\cup \left[\frac{2}{3},1 \right]$ where the two disjoint parts are two clusters. There we see that the expected leapfrog distance for consecutive points within the same cluster is on the order $O\left(\frac{1}{n^2}\right)$, yielding an expected intra-cluster distance of at most $O\left(\frac{1}{n}\right).$ On the other hand, for points between the two clusters, the expected leapfrog distance is at least $\frac{1}{9} + O\left(\frac{1}{n}\right)$. By a simple concentration inequality argument, we may conclude that the intra-cluster leapfrog distances tend to 0 asymptotically, while the inter-cluster distance tends to $\frac{1}{9}$.  If the points are re-embedded so that their pairwise distances are leapfrog distances, then clustering becomes trivial.

\begin{figure}[ht]
\centering
\subfloat[][Intra-cluster LF distance]{
\tikzset{every picture/.style={line width=0.75pt}} 
\scalebox{1.0}{
\begin{tikzpicture}[x=0.75pt,y=0.75pt,yscale=-1,xscale=1]

\draw    (130,170) -- (310,170) ;
\draw  [color={rgb, 255:red, 74; green, 144; blue, 226 }  ,draw opacity=1 ][fill={rgb, 255:red, 74; green, 144; blue, 226 }  ,fill opacity=1 ] (130,160) -- (190,160) -- (190,170) -- (130,170) -- cycle ;
\draw  [color={rgb, 255:red, 74; green, 144; blue, 226 }  ,draw opacity=1 ][fill={rgb, 255:red, 74; green, 144; blue, 226 }  ,fill opacity=1 ] (250,160) -- (310,160) -- (310,170) -- (250,170) -- cycle ;
\draw  [fill={rgb, 255:red, 0; green, 0; blue, 0 }  ,fill opacity=1 ] (138,170) .. controls (138,168.9) and (138.9,168) .. (140,168) .. controls (141.1,168) and (142,168.9) .. (142,170) .. controls (142,171.1) and (141.1,172) .. (140,172) .. controls (138.9,172) and (138,171.1) .. (138,170) -- cycle ;
\draw  [fill={rgb, 255:red, 0; green, 0; blue, 0 }  ,fill opacity=1 ] (148,170) .. controls (148,168.9) and (148.9,168) .. (150,168) .. controls (151.1,168) and (152,168.9) .. (152,170) .. controls (152,171.1) and (151.1,172) .. (150,172) .. controls (148.9,172) and (148,171.1) .. (148,170) -- cycle ;
\draw  [fill={rgb, 255:red, 0; green, 0; blue, 0 }  ,fill opacity=1 ] (168,170) .. controls (168,168.9) and (168.9,168) .. (170,168) .. controls (171.1,168) and (172,168.9) .. (172,170) .. controls (172,171.1) and (171.1,172) .. (170,172) .. controls (168.9,172) and (168,171.1) .. (168,170) -- cycle ;
\draw  [fill={rgb, 255:red, 0; green, 0; blue, 0 }  ,fill opacity=1 ] (157,170) .. controls (157,168.9) and (157.9,168) .. (159,168) .. controls (160.1,170) and (161,168.9) .. (161,170) .. controls (161,171.1) and (160.1,172) .. (159,172) .. controls (157.9,172) and (157,171.1) .. (157,170) -- cycle ;
\draw  [fill={rgb, 255:red, 0; green, 0; blue, 0 }  ,fill opacity=1 ] (179,170) .. controls (179,168.9) and (179.9,168) .. (181,168) .. controls (182.1,168) and (183,168.9) .. (183,170) .. controls (183,171.1) and (182.1,172) .. (181,172) .. controls (179.9,172) and (179,171.1) .. (179,170) -- cycle ;
\draw  [fill={rgb, 255:red, 0; green, 0; blue, 0 }  ,fill opacity=1 ] (258,170) .. controls (258,168.9) and (258.9,168) .. (260,168) .. controls (261.1,168) and (262,168.9) .. (262,170) .. controls (262,171.1) and (261.1,172) .. (260,172) .. controls (258.9,172) and (258,171.1) .. (258,170) -- cycle ;
\draw  [fill={rgb, 255:red, 0; green, 0; blue, 0 }  ,fill opacity=1 ] (268,170) .. controls (268,168.9) and (268.9,168) .. (270,168) .. controls (271.1,168) and (272,168.9) .. (272,170) .. controls (272,171.1) and (271.1,172) .. (270,172) .. controls (268.9,172) and (268,171.1) .. (268,170) -- cycle ;
\draw  [fill={rgb, 255:red, 0; green, 0; blue, 0 }  ,fill opacity=1 ] (278,170) .. controls (278,168.9) and (278.9,168) .. (280,168) .. controls (281.1,168) and (282,168.9) .. (282,170) .. controls (282,171.1) and (281.1,172) .. (280,172) .. controls (278.9,172) and (278,171.1) .. (278,170) -- cycle ;
\draw  [fill={rgb, 255:red, 0; green, 0; blue, 0 }  ,fill opacity=1 ] (287,170) .. controls (287,168.9) and (287.9,168) .. (289,168) .. controls (290.1,168) and (291,168.9) .. (291,170) .. controls (291,171.1) and (290.1,172) .. (289,172) .. controls (287.9,172) and (287,171.1) .. (287,170) -- cycle ;
\draw  [fill={rgb, 255:red, 0; green, 0; blue, 0 }  ,fill opacity=1 ] (298,170) .. controls (298,168.9) and (298.9,168) .. (300,168) .. controls (301.1,168) and (302,168.9) .. (302,170) .. controls (302,171.1) and (301.1,172) .. (300,172) .. controls (298.9,172) and (298,171.1) .. (298,170) -- cycle ;
\draw  [fill={rgb, 255:red, 0; green, 0; blue, 0 }  ,fill opacity=1 ] (128,170) .. controls (128,168.9) and (128.9,168) .. (130,168) .. controls (131.1,168) and (132,168.9) .. (132,170) .. controls (132,171.1) and (131.1,172) .. (130,172) .. controls (128.9,172) and (128,171.1) .. (128,170) -- cycle ;
\draw  [fill={rgb, 255:red, 0; green, 0; blue, 0 }  ,fill opacity=1 ] (188,170) .. controls (188,168.9) and (188.9,168) .. (190,168) .. controls (191.1,168) and (192,168.9) .. (192,170) .. controls (192,171.1) and (191.1,172) .. (190,172) .. controls (188.9,172) and (188,171.1) .. (188,170) -- cycle ;
\draw  [fill={rgb, 255:red, 0; green, 0; blue, 0 }  ,fill opacity=1 ] (248,170) .. controls (248,168.9) and (248.9,168) .. (250,168) .. controls (251.1,168) and (252,168.9) .. (252,170) .. controls (252,171.1) and (251.1,172) .. (250,172) .. controls (248.9,172) and (248,171.1) .. (248,170) -- cycle ;
\draw  [fill={rgb, 255:red, 0; green, 0; blue, 0 }  ,fill opacity=1 ] (308,170) .. controls (308,168.9) and (308.9,168) .. (310,168) .. controls (311.1,168) and (312,168.9) .. (312,170) .. controls (312,171.1) and (311.1,172) .. (310,172) .. controls (308.9,172) and (308,171.1) .. (308,170) -- cycle ;
\draw   (131,196) .. controls (131,200.67) and (133.33,203) .. (138,203) -- (150.5,203) .. controls (157.17,203) and (160.5,205.33) .. (160.5,210) .. controls (160.5,205.33) and (163.83,203) .. (170.5,203)(167.5,203) -- (183,203) .. controls (187.67,203) and (190,200.67) .. (190,196) ;
\draw   (151,155) .. controls (151,153.21) and (150.11,152.32) .. (148.32,152.32) -- (148.32,152.32) .. controls (145.77,152.32) and (144.5,151.43) .. (144.5,149.65) .. controls (144.5,151.43) and (143.23,152.32) .. (140.68,152.32)(141.82,152.32) -- (140.68,152.32) .. controls (138.89,152.32) and (138,153.21) .. (138,155) ;
\draw   (302,158) .. controls (302.11,153.33) and (299.84,150.95) .. (295.17,150.84) -- (291.16,150.74) .. controls (284.5,150.58) and (281.22,148.17) .. (281.33,143.51) .. controls (281.22,148.17) and (277.84,150.42) .. (271.17,150.27)(274.17,150.34) -- (267.16,150.17) .. controls (262.5,150.06) and (260.11,152.33) .. (260,157) ;

\draw (122,173.4) node [anchor=north west][inner sep=0.75pt]  [font=\small]  {$0$};
\draw (312,173.4) node [anchor=north west][inner sep=0.75pt]  [font=\small]  {$1$};
\draw (186,175.4) node [anchor=north west][inner sep=0.75pt]  [font=\footnotesize]  {$\frac{1}{3}$};
\draw (249,175.4) node [anchor=north west][inner sep=0.75pt]  [font=\footnotesize]  {$\frac{2}{3}$};
\draw (74,208.4) node [anchor=north west][inner sep=0.75pt]  [font=\footnotesize]  {$\textcolor[rgb]{0.96,0.65,0.14}{\mathbb{E}[ LF( a_{1} ,a_{n/2})] =O\left(\frac{1}{n}\right)}$};
\draw (60,121.4) node [anchor=north west][inner sep=0.75pt]  [font=\footnotesize]  {$\textcolor[rgb]{0.96,0.65,0.14}{\mathbb{E}[ LF( a_{i} ,a_{i+1})] =O\left(\frac{1}{n^2}\right)}$};
\draw (219,121.4) node [anchor=north west][inner sep=0.75pt]  [font=\footnotesize]  {$\textcolor[rgb]{0.96,0.65,0.14}{\mathbb{E}[ LF( a_{i} ,a_{j})] =O\left(\frac{1}{n}\right)}$};

\end{tikzpicture}
}
}

\subfloat[][Inter-cluster LF distance]{

\tikzset{every picture/.style={line width=0.75pt}} 
\scalebox{1.0}{
\begin{tikzpicture}[x=0.75pt,y=0.75pt,yscale=-1,xscale=1]

\draw    (130,170) -- (310,170) ;
\draw  [color={rgb, 255:red, 74; green, 144; blue, 226 }  ,draw opacity=1 ][fill={rgb, 255:red, 74; green, 144; blue, 226 }  ,fill opacity=1 ] (130,160) -- (190,160) -- (190,170) -- (130,170) -- cycle ;
\draw  [color={rgb, 255:red, 74; green, 144; blue, 226 }  ,draw opacity=1 ][fill={rgb, 255:red, 74; green, 144; blue, 226 }  ,fill opacity=1 ] (250,160) -- (310,160) -- (310,170) -- (250,170) -- cycle ;
\draw  [fill={rgb, 255:red, 0; green, 0; blue, 0 }  ,fill opacity=1 ] (138,170) .. controls (138,168.9) and (138.9,168) .. (140,168) .. controls (141.1,168) and (142,168.9) .. (142,170) .. controls (142,171.1) and (141.1,172) .. (140,172) .. controls (138.9,172) and (138,171.1) .. (138,170) -- cycle ;
\draw  [fill={rgb, 255:red, 0; green, 0; blue, 0 }  ,fill opacity=1 ] (148,170) .. controls (148,168.9) and (148.9,168) .. (150,168) .. controls (151.1,168) and (152,168.9) .. (152,170) .. controls (152,171.1) and (151.1,172) .. (150,172) .. controls (148.9,172) and (148,171.1) .. (148,170) -- cycle ;
\draw  [fill={rgb, 255:red, 0; green, 0; blue, 0 }  ,fill opacity=1 ] (168,170) .. controls (168,168.9) and (168.9,168) .. (170,168) .. controls (171.1,168) and (172,168.9) .. (172,170) .. controls (172,171.1) and (171.1,172) .. (170,172) .. controls (168.9,172) and (168,171.1) .. (168,170) -- cycle ;
\draw  [fill={rgb, 255:red, 0; green, 0; blue, 0 }  ,fill opacity=1 ] (157,170) .. controls (157,168.9) and (157.9,168) .. (159,168) .. controls (160.1,170) and (161,168.9) .. (161,170) .. controls (161,171.1) and (160.1,172) .. (159,172) .. controls (157.9,172) and (157,171.1) .. (157,170) -- cycle ;
\draw  [fill={rgb, 255:red, 0; green, 0; blue, 0 }  ,fill opacity=1 ] (179,170) .. controls (179,168.9) and (179.9,168) .. (181,168) .. controls (182.1,168) and (183,168.9) .. (183,170) .. controls (183,171.1) and (182.1,172) .. (181,172) .. controls (179.9,172) and (179,171.1) .. (179,170) -- cycle ;
\draw  [fill={rgb, 255:red, 0; green, 0; blue, 0 }  ,fill opacity=1 ] (258,170) .. controls (258,168.9) and (258.9,168) .. (260,168) .. controls (261.1,168) and (262,168.9) .. (262,170) .. controls (262,171.1) and (261.1,172) .. (260,172) .. controls (258.9,172) and (258,171.1) .. (258,170) -- cycle ;
\draw  [fill={rgb, 255:red, 0; green, 0; blue, 0 }  ,fill opacity=1 ] (268,170) .. controls (268,168.9) and (268.9,168) .. (270,168) .. controls (271.1,168) and (272,168.9) .. (272,170) .. controls (272,171.1) and (271.1,172) .. (270,172) .. controls (268.9,172) and (268,171.1) .. (268,170) -- cycle ;
\draw  [fill={rgb, 255:red, 0; green, 0; blue, 0 }  ,fill opacity=1 ] (278,170) .. controls (278,168.9) and (278.9,168) .. (280,168) .. controls (281.1,168) and (282,168.9) .. (282,170) .. controls (282,171.1) and (281.1,172) .. (280,172) .. controls (278.9,172) and (278,171.1) .. (278,170) -- cycle ;
\draw  [fill={rgb, 255:red, 0; green, 0; blue, 0 }  ,fill opacity=1 ] (287,170) .. controls (287,168.9) and (287.9,168) .. (289,168) .. controls (290.1,168) and (291,168.9) .. (291,170) .. controls (291,171.1) and (290.1,172) .. (289,172) .. controls (287.9,172) and (287,171.1) .. (287,170) -- cycle ;
\draw  [fill={rgb, 255:red, 0; green, 0; blue, 0 }  ,fill opacity=1 ] (298,170) .. controls (298,168.9) and (298.9,168) .. (300,168) .. controls (301.1,168) and (302,168.9) .. (302,170) .. controls (302,171.1) and (301.1,172) .. (300,172) .. controls (298.9,172) and (298,171.1) .. (298,170) -- cycle ;
\draw  [fill={rgb, 255:red, 0; green, 0; blue, 0 }  ,fill opacity=1 ] (128,170) .. controls (128,168.9) and (128.9,168) .. (130,168) .. controls (131.1,168) and (132,168.9) .. (132,170) .. controls (132,171.1) and (131.1,172) .. (130,172) .. controls (128.9,172) and (128,171.1) .. (128,170) -- cycle ;
\draw  [fill={rgb, 255:red, 0; green, 0; blue, 0 }  ,fill opacity=1 ] (188,170) .. controls (188,168.9) and (188.9,168) .. (190,168) .. controls (191.1,168) and (192,168.9) .. (192,170) .. controls (192,171.1) and (191.1,172) .. (190,172) .. controls (188.9,172) and (188,171.1) .. (188,170) -- cycle ;
\draw  [fill={rgb, 255:red, 0; green, 0; blue, 0 }  ,fill opacity=1 ] (248,170) .. controls (248,168.9) and (248.9,168) .. (250,168) .. controls (251.1,168) and (252,168.9) .. (252,170) .. controls (252,171.1) and (251.1,172) .. (250,172) .. controls (248.9,172) and (248,171.1) .. (248,170) -- cycle ;
\draw  [fill={rgb, 255:red, 0; green, 0; blue, 0 }  ,fill opacity=1 ] (308,170) .. controls (308,168.9) and (308.9,168) .. (310,168) .. controls (311.1,168) and (312,168.9) .. (312,170) .. controls (312,171.1) and (311.1,172) .. (310,172) .. controls (308.9,172) and (308,171.1) .. (308,170) -- cycle ;
\draw   (268,152) .. controls (267.95,147.33) and (265.59,145.03) .. (260.92,145.08) -- (232.42,145.39) .. controls (225.75,145.46) and (222.4,143.17) .. (222.35,138.5) .. controls (222.4,143.17) and (219.09,145.54) .. (212.42,145.61)(215.42,145.58) -- (183.92,145.92) .. controls (179.25,145.97) and (176.95,148.33) .. (177,153) ;
\draw   (189,195) .. controls (189,199.67) and (191.33,202) .. (196,202) -- (209.5,202) .. controls (216.17,202) and (219.5,204.33) .. (219.5,209) .. controls (219.5,204.33) and (222.83,202) .. (229.5,202)(226.5,202) -- (243,202) .. controls (247.67,202) and (250,199.67) .. (250,195) ;

\draw (122,173.4) node [anchor=north west][inner sep=0.75pt]  [font=\small]  {$0$};
\draw (312,173.4) node [anchor=north west][inner sep=0.75pt]  [font=\small]  {$1$};
\draw (186,175.4) node [anchor=north west][inner sep=0.75pt]  [font=\footnotesize]  {$\frac{1}{3}$};
\draw (249,175.4) node [anchor=north west][inner sep=0.75pt]  [font=\footnotesize]  {$\frac{2}{3}$};
\draw (159,108.4) node [anchor=north west][inner sep=0.75pt]  [font=\footnotesize]  {$\textcolor[rgb]{0.25,0.46,0.02}{\mathbb{E}[ LF( a_{i} ,a_{j})] =\frac{1}{9} +O\left(\frac{1}{n}\right)}$};
\draw (159,205.4) node [anchor=north west][inner sep=0.75pt]  [font=\footnotesize]  {$\textcolor[rgb]{0.25,0.46,0.02}{\mathbb{E}[ LF( a_{n/2} ,a_{n/2+1})] =\frac{1}{9}}$};

\end{tikzpicture}
}
}

\subfloat[][Leapfrog distance]{

\tikzset{every picture/.style={line width=0.75pt}} 

\scalebox{1.0}{
\begin{tikzpicture}[x=0.75pt,y=0.75pt,yscale=-1,xscale=1]

\draw    (130,170) -- (310,170) ;
\draw  [color={rgb, 255:red, 74; green, 144; blue, 226 }  ,draw opacity=1 ][fill={rgb, 255:red, 74; green, 144; blue, 226 }  ,fill opacity=1 ] (130,160) -- (190,160) -- (190,170) -- (130,170) -- cycle ;
\draw  [color={rgb, 255:red, 74; green, 144; blue, 226 }  ,draw opacity=1 ][fill={rgb, 255:red, 74; green, 144; blue, 226 }  ,fill opacity=1 ] (250,160) -- (310,160) -- (310,170) -- (250,170) -- cycle ;
\draw  [fill={rgb, 255:red, 0; green, 0; blue, 0 }  ,fill opacity=1 ] (138,170) .. controls (138,168.9) and (138.9,168) .. (140,168) .. controls (141.1,168) and (142,168.9) .. (142,170) .. controls (142,171.1) and (141.1,172) .. (140,172) .. controls (138.9,172) and (138,171.1) .. (138,170) -- cycle ;
\draw  [fill={rgb, 255:red, 0; green, 0; blue, 0 }  ,fill opacity=1 ] (148,170) .. controls (148,168.9) and (148.9,168) .. (150,168) .. controls (151.1,168) and (152,168.9) .. (152,170) .. controls (152,171.1) and (151.1,172) .. (150,172) .. controls (148.9,172) and (148,171.1) .. (148,170) -- cycle ;
\draw  [fill={rgb, 255:red, 0; green, 0; blue, 0 }  ,fill opacity=1 ] (168,170) .. controls (168,168.9) and (168.9,168) .. (170,168) .. controls (171.1,168) and (172,168.9) .. (172,170) .. controls (172,171.1) and (171.1,172) .. (170,172) .. controls (168.9,172) and (168,171.1) .. (168,170) -- cycle ;
\draw  [fill={rgb, 255:red, 0; green, 0; blue, 0 }  ,fill opacity=1 ] (157,170) .. controls (157,168.9) and (157.9,168) .. (159,168) .. controls (160.1,170) and (161,168.9) .. (161,170) .. controls (161,171.1) and (160.1,172) .. (159,172) .. controls (157.9,172) and (157,171.1) .. (157,170) -- cycle ;
\draw  [fill={rgb, 255:red, 0; green, 0; blue, 0 }  ,fill opacity=1 ] (179,170) .. controls (179,168.9) and (179.9,168) .. (181,168) .. controls (182.1,168) and (183,168.9) .. (183,170) .. controls (183,171.1) and (182.1,172) .. (181,172) .. controls (179.9,172) and (179,171.1) .. (179,170) -- cycle ;
\draw  [fill={rgb, 255:red, 0; green, 0; blue, 0 }  ,fill opacity=1 ] (258,170) .. controls (258,168.9) and (258.9,168) .. (260,168) .. controls (261.1,168) and (262,168.9) .. (262,170) .. controls (262,171.1) and (261.1,172) .. (260,172) .. controls (258.9,172) and (258,171.1) .. (258,170) -- cycle ;
\draw  [fill={rgb, 255:red, 0; green, 0; blue, 0 }  ,fill opacity=1 ] (268,170) .. controls (268,168.9) and (268.9,168) .. (270,168) .. controls (271.1,168) and (272,168.9) .. (272,170) .. controls (272,171.1) and (271.1,172) .. (270,172) .. controls (268.9,172) and (268,171.1) .. (268,170) -- cycle ;
\draw  [fill={rgb, 255:red, 0; green, 0; blue, 0 }  ,fill opacity=1 ] (278,170) .. controls (278,168.9) and (278.9,168) .. (280,168) .. controls (281.1,168) and (282,168.9) .. (282,170) .. controls (282,171.1) and (281.1,172) .. (280,172) .. controls (278.9,172) and (278,171.1) .. (278,170) -- cycle ;
\draw  [fill={rgb, 255:red, 0; green, 0; blue, 0 }  ,fill opacity=1 ] (287,170) .. controls (287,168.9) and (287.9,168) .. (289,168) .. controls (290.1,168) and (291,168.9) .. (291,170) .. controls (291,171.1) and (290.1,172) .. (289,172) .. controls (287.9,172) and (287,171.1) .. (287,170) -- cycle ;
\draw  [fill={rgb, 255:red, 0; green, 0; blue, 0 }  ,fill opacity=1 ] (298,170) .. controls (298,168.9) and (298.9,168) .. (300,168) .. controls (301.1,168) and (302,168.9) .. (302,170) .. controls (302,171.1) and (301.1,172) .. (300,172) .. controls (298.9,172) and (298,171.1) .. (298,170) -- cycle ;
\draw  [fill={rgb, 255:red, 0; green, 0; blue, 0 }  ,fill opacity=1 ] (128,170) .. controls (128,168.9) and (128.9,168) .. (130,168) .. controls (131.1,168) and (132,168.9) .. (132,170) .. controls (132,171.1) and (131.1,172) .. (130,172) .. controls (128.9,172) and (128,171.1) .. (128,170) -- cycle ;
\draw  [fill={rgb, 255:red, 0; green, 0; blue, 0 }  ,fill opacity=1 ] (188,170) .. controls (188,168.9) and (188.9,168) .. (190,168) .. controls (191.1,168) and (192,168.9) .. (192,170) .. controls (192,171.1) and (191.1,172) .. (190,172) .. controls (188.9,172) and (188,171.1) .. (188,170) -- cycle ;
\draw  [fill={rgb, 255:red, 0; green, 0; blue, 0 }  ,fill opacity=1 ] (248,170) .. controls (248,168.9) and (248.9,168) .. (250,168) .. controls (251.1,168) and (252,168.9) .. (252,170) .. controls (252,171.1) and (251.1,172) .. (250,172) .. controls (248.9,172) and (248,171.1) .. (248,170) -- cycle ;
\draw  [fill={rgb, 255:red, 0; green, 0; blue, 0 }  ,fill opacity=1 ] (308,170) .. controls (308,168.9) and (308.9,168) .. (310,168) .. controls (311.1,168) and (312,168.9) .. (312,170) .. controls (312,171.1) and (311.1,172) .. (310,172) .. controls (308.9,172) and (308,171.1) .. (308,170) -- cycle ;
\draw   (268,152) .. controls (267.95,147.33) and (265.59,145.03) .. (260.92,145.08) -- (232.42,145.39) .. controls (225.75,145.46) and (222.4,143.17) .. (222.35,138.5) .. controls (222.4,143.17) and (219.09,145.54) .. (212.42,145.61)(215.42,145.58) -- (183.92,145.92) .. controls (179.25,145.97) and (176.95,148.33) .. (177,153) ;
\draw   (267,197) .. controls (267.12,201.67) and (269.51,203.94) .. (274.18,203.81) -- (276.19,203.76) .. controls (282.85,203.59) and (286.24,205.83) .. (286.37,210.49) .. controls (286.24,205.83) and (289.51,203.41) .. (296.18,203.23)(293.18,203.31) -- (298.19,203.18) .. controls (302.85,203.06) and (305.12,200.67) .. (305,196) ;

\draw (122,173.4) node [anchor=north west][inner sep=0.75pt]  [font=\small]  {$0$};
\draw (312,173.4) node [anchor=north west][inner sep=0.75pt]  [font=\small]  {$1$};
\draw (186,175.4) node [anchor=north west][inner sep=0.75pt]  [font=\footnotesize]  {$\frac{1}{3}$};
\draw (249,175.4) node [anchor=north west][inner sep=0.75pt]  [font=\footnotesize]  {$\frac{2}{3}$};
\draw (162,111.4) node [anchor=north west][inner sep=0.75pt]  [font=\footnotesize]  {$\textcolor[rgb]{0.25,0.46,0.02}{LF( a_{i} ,a_{j}) \to \frac{1}{9}}$};
\draw (242,209.4) node [anchor=north west][inner sep=0.75pt]  [font=\footnotesize]  {$\textcolor[rgb]{0.96,0.65,0.14}{LF( a_{i} ,a_{j}) \to 0}$};

\end{tikzpicture}
}
}
\vskip 0.1in
\caption{Example of intra-/inter-cluster leapfrog distances for points sampled from a uniform distribution supported on $[0,\frac{1}{3}]\cup [\frac{2}{3},1]$. (a) The expected leapfrog distance between two neighboring points is $O\left(\frac{1}{n^2}\right)$, hence the expected intra-cluster leapfrog distance is $O\left(\frac{1}{n}\right)$.  (b) The cost to travel through the gap between $\frac{1}{3}$ and $\frac{2}{3}$ is $\frac{1}{9}$. The inter-cluster distance is at least $\frac{1}{9}$. Hence, the expected leapfrog distance is $\frac{1}{9}+O\left(\frac{1}{n}\right)$. (c) Coupled with a concentration inequality, the intra-cluster leapfrog distance tends to 0 asymptotically while the inter-cluster leapfrog distance tends to $\frac{1}{9}$ asymptotically.}
\label{fig:1d_disjoint_support}
\vskip -0.05in
\end{figure}
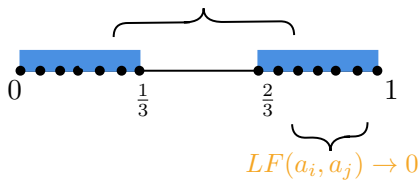

We also point out that the leapfrog distance $\mathrm{LF}(\cdot, \cdot)$ is indeed a metric on the set of data points.
The leapfrog distance between two fixed points changes as more points get sampled. In $\R^1$, the leapfrog distance between two arbitrary points $a,b \in \R^1$ is given by $\mathrm{LF}(a,b):=(a_1-a)^2 + \sum_{i=2}^m (a_i-a_{i-1})^2 + (b-a_m)^2$, where $a_1,\ldots,a_m$ are the points between $a,b$ and satisfy $a<a_1<\ldots<a_m<b$. In higher dimensions, the leapfrog distance does not have a closed-form expression.

\section{Characterization of the leapfrog distance in $\R^1$}

\label{sec:lf_property}
In this section, we analyze the leapfrog distance for data points which are $n$ i.i.d.\ samples from a probability density function $f$ in $\R^1$.
We assume that $f$ is continuous when restricted to its support.
These assumptions hold for the two main cases considered in this paper, that is, a Gaussian mixture model and also a function $f$ supported on a disjoint union of intervals that is positive on its support.

We also note that the bounds in this section all hold for $n$ sufficiently large, e.g., we often upper bound $an^c+bn^d$ by $(a+1)n^c$, where $a,b,c,d$ are all positive and $c>d$. The threshold value of $n$ that validates all of our bounds is an impractically large number.  However, the form of the bounds correctly describes the behavior we have observed in our experiments even for moderate $n$.  We suspect that similar bounds could be shown for more reasonable ranges of $n$ using more specialized techniques instead of general-purpose results like Hoeffding’s inequality.

\subsection{Expected value of leapfrog distance}
\label{subsec:explf}


Consider a probability density function $f$ supported on a union of finite intervals and continuous on each interval. Fix a value $\underline{f}>0$, and let $U$ be the set
\[
U:=\{x\in \R:f(x)\ge \underline{f}\}.
\]
By continuity, $U$ is a union of closed intervals.
Furthermore, its total measure is finite since $f$ is a probability density function.
Therefore, define
\[
v_U:=\vol(U).
\]

The following theorem describes the expected value of the leapfrog distance (See Appendix \ref{app:ELF} for detailed derivation of expected leapfrog distance).

\begin{theorem}
    The expectation of $\mathrm{LF}(a,b)$ has the following expression asymptotically as $n \to \infty$:
    \begin{equation}
    \mathbb E[\mathrm{LF}(a,b)]=\frac{2}{n}\int_a^b \frac{dx}{f(x)} + o\left(1/n\right),
    \label{eq:ELF_thm}
    \end{equation}
    for $a,b \in U$, $a<b$ such that $[a,b]\subseteq U$.
    The remainder is bounded in magnitude by
    $$
    10n^{-1.5}v_U^2
    $$
    if $f$ is locally $L$-Lipschitz continuous.
    \label{thm:ELF_U}
\end{theorem}

\subsection{High-probability bound for leapfrog distance}
\label{subsec:highprob}
For the remainder of the section, we present a concentration bound to quantify the tail behavior of leapfrog distance as follows.  Details of derivation of high-probability bound for leapfrog distance are shown in Appendix \ref{app:highprob}.
We assume that $f$ is Lipschitz continuous and is bounded below by a positive number on the interval of interest.

\begin{theorem}
 Assume that probability density function $f$ is Lipschitz continuous on its support.  Let $a,b$ be two data points drawn from an interval $I$ with the property that $f(x)\in[\underline{f},\overline{f}]$ for all $x\in I$ such that $\underline{f}>0$.
There exist constants $c_1, c_2, c_3>0$ and integer $n_0$, all of which may depend on $f$ and $I$, such that for all data points $a,b$ satisfying $a<b$, $a,b\in I$ and assuming $n>n_0$,
\begin{align*}
\mathbb P[|\mathrm{LF}(a,b) -  \mathbb E[\mathrm{LF}&(a,b)]| \ge C n^{-1.04}] \le c_1\cdot\exp(-c_2 n^{c_3}),
\end{align*}
where the coefficient $C$ is specified in \eqref{eq:mainhighprobineq} below.
\label{thm:LF_main}
\end{theorem}

\section{An upper bound on leapfrog distance in $\R^d$}
\label{sec:LF_Rd}





In this section, we show that the leapfrog distance between an arbitrary pair of points taken from a set of $n$ points in $\R^d$ tends to 0 like $O(n^{-1/d+\eta})$, where $\eta>0$ can be arbitrarily small,
and where the $n$ points are sampled from a PDF with compact, support composed of $K$ connected components (supports for $K$ clusters) and a positive lower bound on the support.
  The bound in this section is used in Section~\ref{sec:chiquetproof} to analyze clusters with disjoint supports.

We will say that the probability density function $f$
is {\em admissible} if it satisfies the following properties.
Let $S=\supp(f)$.
\begin{enumerate}
    \item
$S$ is compact,
\item
$S$ is Lebesgue-measurable,
\item
$S$ is the disjoint union of $K$ path-connected components $S_1$, \ldots, $S_K$, and the diameter of $S_i$ (with respect to paths inside $S_i$) is at most $\ell<\infty$ for each $i=1,\ldots,K$,
\item
There exists $\theta>0$ such that $f(\x)\ge \theta$ for all $\x\in S$.
\item
There exist $r_0>0$ and $p>0$ such that for every $\x\in S_i$, for every $i=1,\ldots,K$,
for every $r\in(0,r_0]$, $\vol(B(\x,r)\cap S_i)\ge pr^d$.
\end{enumerate}
The last assumption rules out the case that any $S_i$
has a cusp-like protrusion, or that a portion of $S_i$ has lower Hausdorff dimension than $d$.
A disjoint union of $K$ full-dimensional compact convex bodies always satisfy all of these condition, and the properties are preserved under bijective quasiregular mappings.

The main theorem of this section is as follows (see Appendix \ref{app:LF_main_proof} for detailed proof).

\begin{theorem}
Assume $f$ is admissible.  Let $n$ points in $\R^d$ be sampled according to $f$.  Then, with probability exponentially close to $1$,  for any $i=1,\ldots, K$, for any two samples $\a,\b\in S_i$, $\LF(\a,\b)\le 36\ell n^{-1/d+\eta}$
for an arbitrarily small $\eta>0$.
\label{thm:uppbdlf_dimd}
\end{theorem}

\noindent
{\bf Remark.} We conjecture that the true behavior of $\mathrm{LF}(\a,\b)$ is $\Theta(n^{-1/d})$.  This is proved for the $d=1$ case by Theorem \ref{thm:LF_main}, and in fact, that theorem pinpoints the leading coefficient.

\section{From distances to coordinates}
\label{sec:dist2coord}

So far, we have defined a new distance measure which we call the leapfrog distance. Under the admissibility assumption,
we have shown that simultaneously (1) each cluster is densely concentrated in the leapfrog distance, and (2) distinct clusters are well separated from each other. We exploit these properties by using leapfrog distances for clustering.
While some clustering algorithms only require pairwise distances, sum-of-norms clustering requires the coordinates of each data point. To find a cluster assignment using \eqref{eq:son-clustering}, one needs to construct a data set of coordinates from the pairwise leapfrog distances, then feed the data coordinates into a clustering algorithm.  In this section, we propose a method to obtain new coordinates.

When the original data set is already one-dimensional, an exact embedding for the pairwise leapfrog distances trivially exists; the construction and properties of the new dataset are illustrated in Subsection \ref{sec:1Dembedding}. In higher dimensions, however, it is often impossible to find an exact embedding. Instead, we use multidimensional scaling to find an inexact embedding. We present the construction in Subsection \ref{sec:MDS}.

\subsection{Trivial embedding in $\R^1$}
\label{sec:1Dembedding}
Without loss of generality, we may assume the original data is arranged in ascending order $a_1\le \dots \le a_n$. Under this assumption, it is easy to see that the closed-form expression for the leapfrog distance between any two points $i,j$ is $\LF(a_i,a_j)=\sum_{k=1}^{j-1} (a_{k+1}-a_{k})^2$. By setting $b_1=0$ and $b_i=\LF(a_1, a_i)$, we have constructed an exact embedding of the leapfrog distance, $b_1, \dots, b_n$, that satisfies $|b_i-b_j|=\LF(a_i,a_j)$. Since the embedding is isometric, the characterization of the leapfrog distance in $\R^1$ remains unchanged for the Euclidean distance between the new data points $\b_i$'s.

\begin{figure}[ht]
\vskip 0.1in
    \centering
\tikzset{every picture/.style={line width=0.75pt}} 

\begin{tikzpicture}[x=0.75pt,y=0.75pt,yscale=-1,xscale=1]

\draw    (100.5,130) -- (271,130) ;
\draw  [fill={rgb, 255:red, 0; green, 0; blue, 0 }  ,fill opacity=1 ] (98,130) .. controls (98,128.62) and (99.12,127.5) .. (100.5,127.5) .. controls (101.88,127.5) and (103,128.62) .. (103,130) .. controls (103,131.38) and (101.88,132.5) .. (100.5,132.5) .. controls (99.12,132.5) and (98,131.38) .. (98,130) -- cycle ;
\draw  [fill={rgb, 255:red, 0; green, 0; blue, 0 }  ,fill opacity=1 ] (147,129.5) .. controls (147,128.12) and (148.12,127) .. (149.5,127) .. controls (150.88,127) and (152,128.12) .. (152,129.5) .. controls (152,130.88) and (150.88,132) .. (149.5,132) .. controls (148.12,132) and (147,130.88) .. (147,129.5) -- cycle ;
\draw  [fill={rgb, 255:red, 0; green, 0; blue, 0 }  ,fill opacity=1 ] (192,129.5) .. controls (192,128.12) and (193.12,127) .. (194.5,127) .. controls (195.88,127) and (197,128.12) .. (197,129.5) .. controls (197,130.88) and (195.88,132) .. (194.5,132) .. controls (193.12,132) and (192,130.88) .. (192,129.5) -- cycle ;
\draw  [fill={rgb, 255:red, 0; green, 0; blue, 0 }  ,fill opacity=1 ] (259,129.5) .. controls (259,128.12) and (260.12,127) .. (261.5,127) .. controls (262.88,127) and (264,128.12) .. (264,129.5) .. controls (264,130.88) and (262.88,132) .. (261.5,132) .. controls (260.12,132) and (259,130.88) .. (259,129.5) -- cycle ;

\draw (91,132.4) node [anchor=north west][inner sep=0.75pt]  [font=\footnotesize]  {$0$};
\draw (97,110.4) node [anchor=north west][inner sep=0.75pt]  [font=\footnotesize]  {$b_{1}$};
\draw (142,110.4) node [anchor=north west][inner sep=0.75pt]  [font=\footnotesize]  {$b_{2}$};
\draw (188,110.4) node [anchor=north west][inner sep=0.75pt]  [font=\footnotesize]  {$b_{3}$};
\draw (219,120.4) node [anchor=north west][inner sep=0.75pt]    {$...$};
\draw (255,110.4) node [anchor=north west][inner sep=0.75pt]  [font=\footnotesize]  {$b_{n}$};
\draw (126,140) node [anchor=north west][inner sep=0.75pt]  [font=\tiny] [align=left] {$LF( a_{1} ,a_{2})$};
\draw (176,140.4) node [anchor=north west][inner sep=0.75pt]  [font=\tiny]  {$LF( a_{1} ,a_{3})$};
\draw (246,140) node [anchor=north west][inner sep=0.75pt]  [font=\tiny] [align=left] {$LF( a_{1} ,a_{n})$};

\end{tikzpicture}
    \vskip 0.1in
    \caption{Illustration of the exact embedding in $\R^1$}
    \label{fig:my_label}
\end{figure}
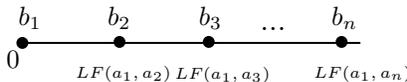

\subsection{Multidimensional scaling}
\label{sec:MDS}
Constructing the coordinates $\b_1, \dots, \b_n$ in higher dimensions is no longer trivial.  We blend two techniques from the previous literature, multidimensional scaling (MDS) \cite{kruskal1978multidimensional} and spectral embedding \cite{luo2003spectral}.
If the distance measure is Euclidean, MDS recovers the original dataset with a rigid transformation. Otherwise, the output of MDS is no longer an exact embedding of the distance matrix. However, the output still encodes useful information, hence MDS is often used as a visualization tool.  Although MDS is typically treated as a heuristic for the non-Euclidean case, we establish rigorous bounds on the MDS embedding that allow us to establish recovery guarantees.

We compute the squared leapfrog distance matrix $D$, where the $(i,j)$th entry of $D$ is given by $d_{i,j}=\mathrm{LF}(\a_i,\a_j)^2$. Our goal is to construct new data points $\b_i \in \R^L$ such that the pairwise Euclidean distance $\Vert\b_i-\b_j\Vert^2$ is well approximated by $d_{i,j}$.

Following the procedure of classical multidimensional scaling, we first compute the Gram matrix as $G(D) = D - D(1,:) \mathbf{1}^T - \mathbf{1} D(1,:)^T$, where $D(1,:)$ denotes the first row of $D$, and $\mathbf{1}$ denotes the vector of all 1's. The new data set $\b_i$'s is the $L$-dimensional spectral embedding of $G(D)$. In particular, we find the eigendecomposition of $G(D)=Q \Lambda Q^T$ where $\Lambda = \mathrm{Diag}(\lambda_1, \dots, \lambda_n)$. We then construct the new data set $\b_i$'s using the top $L$ (in magnitude) eigenpairs. We define
\begin{equation}
\hat G = Q_L \Lambda_L Q_L^T
\label{eq:hatGdef}
\end{equation}
to be the optimal rank-$L$ approximation to $G(D)$ with respect to Frobenius norm, where $\Lambda_L:=  \mathrm{Diag}(\lambda_1, \dots, \lambda_L)$ and $Q_L := Q(:,1:L)$ is the submatrix of the first $L$ columns of $Q$. Construct the re-embeddings $\b_i$'s as follows
\begin{equation}
[\b_1,\b_2,\dots,\b_n] = B := \mathrm{Diag}\left(\sqrt{|\lambda_1|},\ldots,\sqrt{|\lambda_L|}\right) Q_L^T.
\label{eq:b_construction}
\end{equation}

We choose the value $L$ by identifying a significant eigengap between the first $L$ eigenvalues and the last $n-L$ eigenvalues. We find the gap by computing $\lambda_\ell/n$ for all the eigenvalues and looking for the cutoff where the tail $\lambda_\ell/n \to 0$ if $\ell > L$. As such behavior is not observed in practice, a heuristic threshold is applied to find such $L$. With more tools introduced later in next section, we will justify that such $L$ exists and is achievable as the sample size $n$ grows. Moreover, it is guaranteed that $L \le K-1$. Hence, we can always find a low-dimensional embedding.



\section{Properties of $\b_i$'s} 
\label{sec:b_property}

In this section, we propose a framework to analyze the re-embeddings $\b_i$'s obtained from multidimensional scaling.
We analyzed the leapfrog distance for dimensions $d>1$ in Section~\ref{sec:LF_Rd} under the assumption that the probability density function $f$ is admissible: briefly, it is bounded below by a positive number on finite union of compact disjoint sets each of which satisfies a shape condition.  We continue to make these assumptions in this section.  As earlier, $S_1,\ldots,S_K$ denote the supports of the clusters, and let $C_k:=\{i:\a_i\in S_k\}$ denote the $k$th cluster for $k=1,\ldots,K$.

Our analysis uses the following steps:
\begin{enumerate}
    \item Decompose the leapfrog distance matrix $D$ into a low-rank leapfrog distance matrix $\bar D$ and a noise matrix $E$;
    \item Construct re-embeddings $\{\bar \b_i\}_{i=1}^n$ from the low-rank leapfrog distance matrix $\bar D$ using multidimensional scaling;
    \item Prove that the clusters of the re-embeddings $\{\bar \b_i\}_{i=1}^n$ are easily identifiable;
    \item Prove that the original re-embeddings $\{\b_i\}_{i=1}^n$ are a good proxy for $\{\bar \b_i\}_{i=1}^n$ with high probability, and conclude $\{\b_i\}_{i=1}^n$ is also 
    identifiable.
\end{enumerate}

Following these steps, we observed the following properties of $\b_i$'s. Details of our analysis are presented in Appendix \ref{app:bi}.

\begin{theorem} (Intracluster distance.)
Suppose $i,i'\in C_k$ for some $k=1,\ldots,K$. Then with probability exponentially close to $1$ as $n \to \infty$, there holds $\norm{\b_i - \b_{i'}}_2  = o(1)$.

\label{thm:intra_dist_b_disjoint}
\end{theorem}

\begin{theorem}
(Intercluster distance.)
Suppose $i \in C_m, j\in C_{m'}$ with $m\ne m'$. Then with probability exponentially close to $1$ as $n \to \infty$, there holds $\norm{\b_i - \b_{j}}_2 = \Omega(1)$.
\label{thm:inter_dist_X_disjoint}
\end{theorem}

\section{Recovery of clusters by sum-of-norms clustering}
\label{sec:chiquetproof}

So far we have constructed re-embeddings $\b_i$'s of the original data $\a_i$'s under the leapfrog metric and established desirable properties that the re-embeddings enjoy. In particular, for re-embeddings $\b_i$'s, their intra-cluster distances are much smaller than their inter-cluster distances as shown in Section \ref{sec:lf_property} and Section \ref{sec:b_property}. As a consequence, we can strengthen the recovery guarantee for sum-of-norms clustering using leapfrog re-embeddings $\b_i$'s.

We discuss stronger recovery guarantees under two settings. First, we consider data generated identically and independently by some common law supported on a union of disjoint, compact sets satisfying the admissibility conditions of Section~\ref{sec:LF_Rd}. In this setting, our method correctly clusters point when the sample size is sufficiently large. The second setting is when the data is generated by a mixture of Gaussians in $\R^1$. Note that in this case, as the sample size $n\to \infty$, some samples associated with one mean could be placed arbitrarily close to the mean of another Gaussian; as such, we can no longer hope to correctly label all the samples asymptotically. Instead, we settle for certifying correct clustering for samples that are within a fixed number of standard deviations from their respective means.

Previously, Chiquet et al.~\cite{chiquet} derived necessary and sufficient conditions for cluster recovery in terms of certain subgradients. Jiang et al.~\cite{jiang2020recovery} then obtained sufficient conditions for mixture of Gaussians with particular choice of $\b_i-\b_j$ for the Chiquet et al.\ subgradients.
The sufficient condition, which applies to arbitrary data, not just a mixture of Gaussians, is as follows:

\begin{theorem}
\label{thm:recovery}
Suppose SON clustering is applied to given data $\b_1, \dots, \b_n$. For any $k=1,\ldots,K$,
the points indexed by $C_k$ are in the same cluster
provided
\begin{equation}
    \lambda \ge
    \frac{\norm{\b_i - \b_j}}{|C_k|}, \qquad \forall i, j \in C_k
    \label{eq:son_lambda_lb}
\end{equation}
Furthermore, the cluster associated with $C_k$ is distinct from the cluster associated with $C_{k'}$, $1\le k<k'\le K$ provided that there exist $i \in C_k, j \in C_{k'}$ such that
\begin{equation}
\lambda <\frac{\norm{
\b_i - \b_j
}}{2(n-1)}.
    \label{eq:son_lambda_ub}
\end{equation}
\end{theorem}

\subsection{Recovery of non-convex clusters on disjoint supports}
Suppose the data is generated identically and independently by a common law on disjoint, compact sets satisfying the admissibility conditions of Section~\ref{sec:LF_Rd}, denoted by $S_1, \dots, S_K$. 

By Theorem \ref{thm:recovery}, proving perfect recovery reduces to showing that inter-cluster distances are much larger than intra-cluster distances (see Appendix \ref{app:recovery_disjoint} for detailed analysis).

\begin{theorem}
\label{thm:disjoint_recovery}
Suppose data $\a_1, \dots, \a_n$ are independent and identically distributed with a common law $f$, which is admissible and supported on the disjoint union of path-connected components $S_1, \dots, S_K$.
Then there exists $\lambda$ such \eqref{eq:son-clustering}
applied to the re-embeddings $\b_1, \dots, \b_n$
achieves perfect recovery of clusters $C_1, \dots, C_K$,
where $C_k=\{i:\a_i\in S_k\}$ for $k=1,\ldots,K$,
with probability exponentially close to $1$ as $n\rightarrow \infty$. 
\end{theorem}



Previously, Dunlap and Mourrat~\cite{dunlap2021sum} proved that sum-of-norms clustering fails to identify clusters for data consisting of a large number of independent random variables distributed on two disjoints balls that are close to one another. In addition, Nguyen and Mamitsuka~\cite{2021nguyen} showed that sum-of-norms clustering fails to find nonconvex clusters. Our results suggest a workaround: Theorem \ref{thm:disjoint_recovery} demonstrates that sum-of-norms clustering, coupled with our leapfrog re-embedding, can successfully complete the clustering task for datasets which were previously not amenable to sum-of-norms clustering.

\subsection{Recovery of mixture of Gaussians in 1D}

We now consider a second general setting for our method: a mixture of Gaussians on the real line $\mathbb{R}^1$. Clearly exact recovery is not possible in this setting since samples from one of the Gaussians can land arbitrarily close to the mean of another Gaussian because of the everywhere-positive support of Gaussians.
Therefore, the goal is to achieve perfect recovery for points within a fixed number of standard deviations of each mean.

We first state a theorem about arbitrary positive Lipschitz-continuous distributions, and then we show how it applies to a mixture of Gaussians via computations (see Appendix \ref{app:recovery_gaussian} for detailed analysis).

\begin{theorem}
\label{thm:1DGaussiansRecovery}
Let the vertices $a_1,\ldots,a_n$ be chosen i.i.d.~according to a Lipschitz-continuous probability density function $f(x)$ that is positive for all $x\in\R$.  Assume SON clustering is applied to this data after re-embedding according to new coordinates $b_1,\ldots,b_n$. Let $S_1$, \ldots, $S_K$ be disjoint bounded closed intervals ordered from left to right.
 Let $\rho_m$, $m=1,\ldots,K$, denote $\int_{S_m}f(x)\,dx$.
For a particular $m\in\{1,\ldots,K\}$, for any $\eps>0$, all the data points in $S_m$ will be clustered together with probability exponentially close to $1$ as $n\rightarrow\infty$ (that also depends on $\epsilon$) provided that
\begin{equation}
\lambda \ge      \frac{2\int_{S_m}(1/f(x))\,dx}{(\rho_m-\epsilon)n^2} + O(n^{-2.04}).
\label{eq:lambda_lb}
\end{equation}
Furthermore, the cluster associated with $S_m$ is distinct from the cluster associated with $S_{m'}$, $1\le m<m'\le K$, with probability exponentially close to $1$ as $n\rightarrow\infty$  provided that
\begin{equation}
\lambda < \frac{\min_{m=1,\ldots,K-1}\int_{T_m}(1/f(x))\,dx}{n^2} + O(n^{-2.04}),
    \label{eq:lambda_ub}
\end{equation}
where $T_m$ for $m\in\{1,\ldots,K-1\}$ denotes the interval comprising the gap between $S_m$ and $S_{m+1}$, i.e., $T_m=[\max_{x\in S_m} x, \min_{x\in S_{m+1}} x]$.
\end{theorem}

Next, we demonstrate that Theorem \ref{thm:1DGaussiansRecovery} strengthens the state-of-the-art recovery results for a mixture of Gaussians using sum-of-norms clustering \cite{jiang2020recovery}. In particular, we can specify particular values of $K$; $w_1,\ldots,w_K$; $\mu_1<\cdots<\mu_K$; and $\sigma_1,\ldots,\sigma_K$.

We next select a scalar $\theta>0$ and define intervals according to
\[
S_m=[\mu_m-\theta\sigma_m,\mu_m+\theta\sigma_m],
\]
for $m=1,\ldots,K$,
where we assume $\theta$ is sufficiently small so that the $S_m$'s are pairwise disjoint.  Observe that $S_m$ is defined to cover the centrally located points for the $m$th Gaussian.

Then we can compute the lower and upper bounds in \eqref{eq:lambda_lb} and \eqref{eq:lambda_ub}; if the lower bound is less than the upper bound, then recovery is guaranteed asymptotically.  These bounds do not involve generating actual data points $a_1,\ldots,a_n$ but only evaluating one-dimensional integrals.  We can compute the analogous bounds from formulas as in Theorem 3 of \cite{jiang2020recovery}, which uses SON clustering without re-embedding, to demonstrate the strengthening of the guarantee.

The comparison of ranges of $\lambda$ is in Table~\ref{tab:lambdacomp}. One sees from the table that recovery is better for the re-embedded coordinates than for the original coordinates in every computation we tried.  We do not have a theoretical result that the re-embedded coordinates outperform the original coordinates for all ranges of parameters; we return to this point in the conclusion.

\begin{table}[t]
    \caption{Comparison of $\lambda$ lower and upper bounds for original embedding versus re-embedded coordinates.  We took $K=2$ for all rows, and $\mu_1=0$, $\mu_2=1$ for all rows.  If the range of $\lambda$'s is indicated as the empty set, this means that the lower bound exceeded the upper bound, and hence recovery is not guaranteed for any value of $\lambda$.  Note that for the original embedding, we report on $n\lambda$ range since $\lambda$ scales as $\frac{1}{n}$ according to Theorem 3 of \cite{jiang2020recovery}.  In Theorem~\ref{thm:1DGaussiansRecovery}, $\lambda$ scales as $\frac{1}{n^2}$.}
    \label{tab:lambdacomp}
    \vskip 0.15in
    \centering
    \begin{small}
    \begin{sc}
    \begin{tabular}{lllll}
    \hline
     & & & {\bf Orig. data} & {\bf Re-embed.}\\
    $w_i$'s & $\sigma_i$'s & $\theta$ & $n\lambda$ range &
    $n^2\lambda$ range \\
    \hline
{[0.5,0.5]}&{[0.4,0.4]}&1&$\emptyset$&$\emptyset$\\
{[0.5,0.5]}&{[0.3,0.3]}&1&$\emptyset$&$\emptyset$\\
{[0.5,0.5]}&{[0.3,0.3]}&0.5&$\emptyset$&$\emptyset$\\
{[0.5,0.5]}&{[0.2,0.2]}&1&$\emptyset$&{$[ 2.3, 3.3]$}\\
{[0.5,0.5]}&{[0.2,0.2]}&0.5&$\emptyset$&{$[ 1.6, 3.6]$}\\
{[0.5,0.5]}&{[0.1,0.1]}&2&$\emptyset$&{$[ 1.9, 4500]$}\\
{[0.5,0.5]}&{[0.1,0.1]}&1&{$[ 0.47, 0.5]$}&{$[ 0.57, 4500]$}\\
{[0.5,0.5]}&{[0.1,0.1]}&0.5&{$[ 0.38 , 0.5]$}&{$[ 0.4, 4500]$}\\
{[0.9,0.1]}&{[0.3,0.3]}&1&$\emptyset$&$\emptyset$\\
{[0.9,0.1]}&{[0.2,0.2]}&1&$\emptyset$&$\emptyset$\\
{[0.9,0.1]}&{[0.1,0.1]}&1&$\emptyset$&{$[ 14, 7700]$}\\
    \hline
    \end{tabular}
    \end{sc}
\end{small}
\vskip -0.1in
\end{table}

\section{Computational experiments}
\label{sec:exper}
In this section, we demonstrate the value of our approach through a 
set of experiments. In particular, we compare the performance of our re-embedding versus the original features. For the specific case of mixtures of Gaussians, we show that our method tolerates a larger standard deviation while still maintaining perfect recovery.

\begin{center}
    \begin{figure}[h]
    \vskip -0.4in
    \subfloat[][Re-embedding]{
    \includegraphics[scale = 0.5]{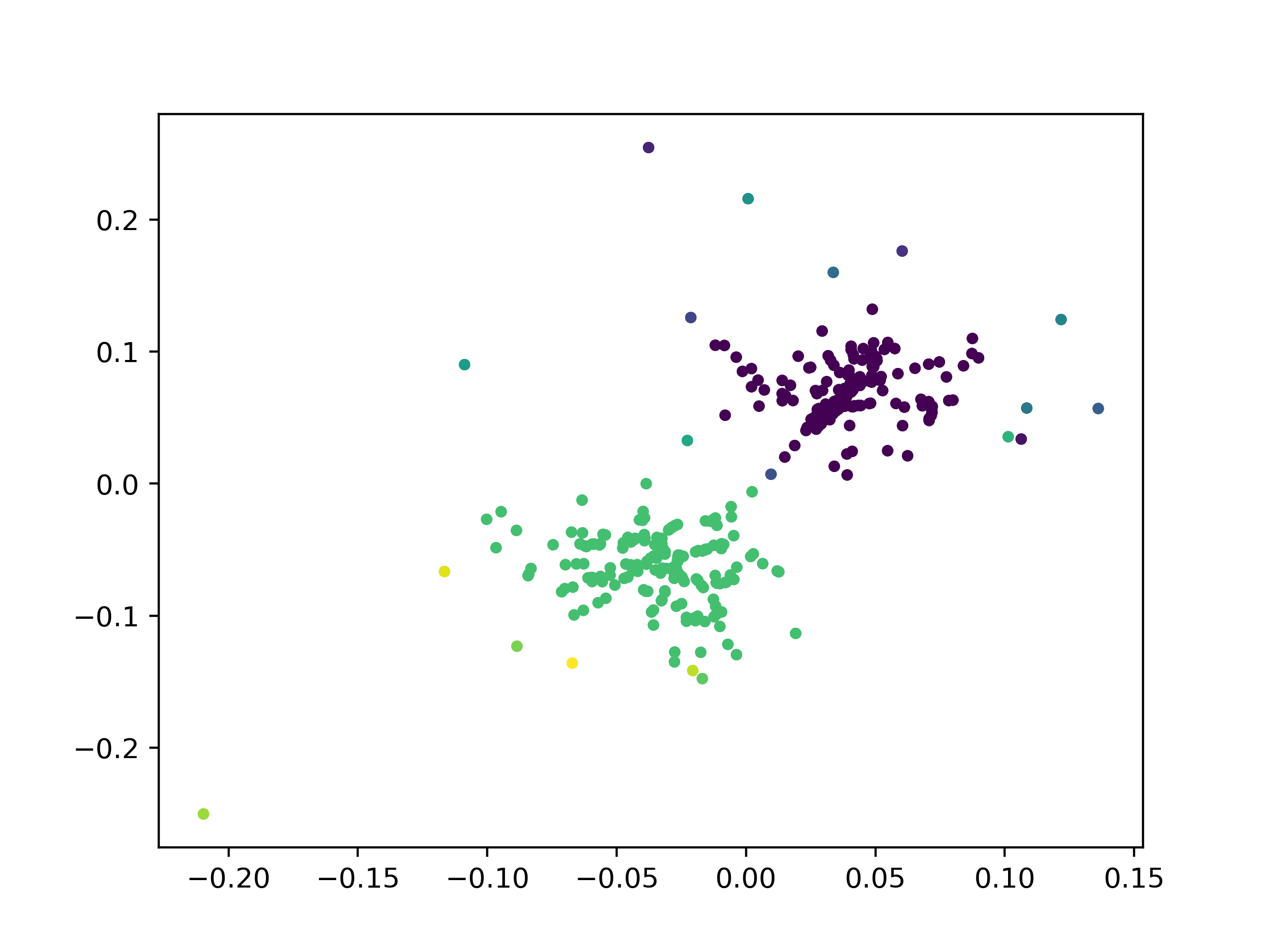}}
    \subfloat[][Original feature space]{
    \includegraphics[scale = 0.5]{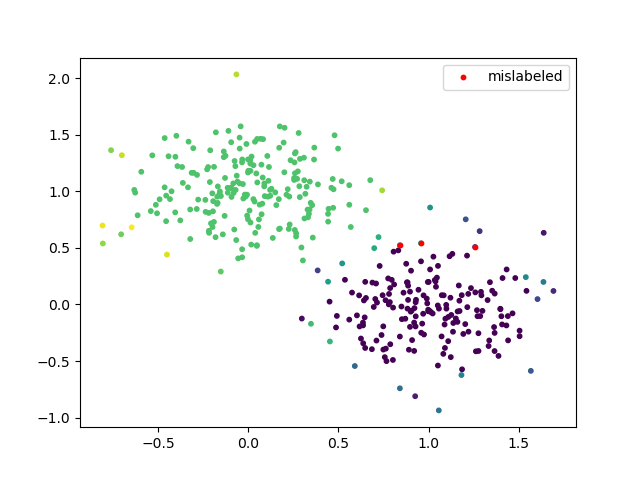}}
    \hfill
    \vskip 0.1in
    \caption{Mixture of two two-dimensional Gaussians with $\sigma = 0.29$. (a) For the re-embedding, we set $\lambda = 0.0038$ to achieve perfect recovery within $2$ standard deviations of the mean. (b) For the original feature space, no $\lambda$ was found at the $10^{-5}$ fidelity to recover all points within $2$ standard deviations perfectly. $\lambda = 0.00356$ is shown here, which misclassifies points shown in red.}
    \label{fig:2d}
    \end{figure}
\end{center}

For the setting of mixtures of Gaussians, we consider two settings, when $d= K = 2$ and $d=K = 6$ and $n=400$. In accordance with our theory, we wish to find settings in which there exists $\sigma$ large enough that we have perfect recovery (for some $\lambda$) near the means for our embedded approach but not for the original coordinates.  Note that our theory for Gaussian mixtures covers only the case $d=1$, but we conjecture that analogous results hold for higher dimensions.

To that end, we increased $\sigma$ until we were able to find a $\lambda$ for which the re-embedded coordinates were perfectly clustered for $V_m = \{i: |a_i - \mu_m| \leq 2\sigma_m\}, m=1,\dots,K $, but for which there were no values of $\lambda$ (to the $10^{-5}$ fidelity) which were able to recover the points in $V_m$ perfectly.

For $d=K=L=2$, we found that $\sigma = 0.29$ yielded perfect clustering for our embedded approach but not for the original approach. Figure \ref{fig:2d} visualizes the results for this setting.

\begin{table}[h]\centering
  \caption{Results for various settings of $\sigma$ for a two-dimensional mixture of Gaussians. The $\lambda$'s and the Rand indices are the best found.}
  \vskip 0.15in
  \begin{small}
    \begin{sc}
  \begin{tabular}{cccccc}\toprule
        &
      \multicolumn{2}{c}{\bf{Original data}} &
      \multicolumn{2}{c}{\bf{Re-embedding}} \\
  $\sigma$ &$\lambda$& Rand Index & $\lambda$ & Rand Index \\
  \hline
  $0.07$ & 0.002 & 1 & 0.0001  & 1 \\
  $0.1$ & 0.002 & 1 & 0.0002 & 1 \\
  $0.15$ & 0.003& 1 & 0.0004 & 1 \\
    $0.2$ & 0.0035& 1  & 0.0006 & 1 \\
$0.29$ & 0.00356& 0.92 & 0.00038 & 0.95 \\
  \bottomrule
  \end{tabular}
  \end{sc}
  \end{small}
  \vskip -0.1in
  \end{table}

For $d=K=L=6$, we found that $\sigma = 0.12$ yielded perfect clustering for our embedded approach but not for the original approach. As six dimensions is not easily visualized, we omit graphs for this result.

\begin{table}[h]\centering
  \caption{Results for various settings of $\sigma$ for a six-dimensional mixture of Gaussians. The $\lambda$s and the Rand indices are the best found.}
  \vskip 0.15in
  \begin{small}
    \begin{sc}
  \begin{tabular}{cccccc}\toprule
  &
      \multicolumn{2}{c}{\bf{Original data}} &
      \multicolumn{2}{c}{\bf{Re-embedding}} \\
  $\sigma$ &$\lambda$& Rand Index & $\lambda$ &  Rand Index \\
  \hline
  $0.06$ &  0.348 & 0.95 & 0.00285  & 1 \\
  $0.09$ & 0.355 & 0.83 & 0.00285 & 1 \\
  $0.12$ & 0.355 & 0.83 & 0.00285 & 0.99 \\
  \bottomrule
  \end{tabular}
  \end{sc}
  \end{small}
  \vskip -0.1in
\end{table}

\begin{figure}[h]
\vskip -0.3in
\subfloat[][Original points]{
\includegraphics[scale = 0.5]{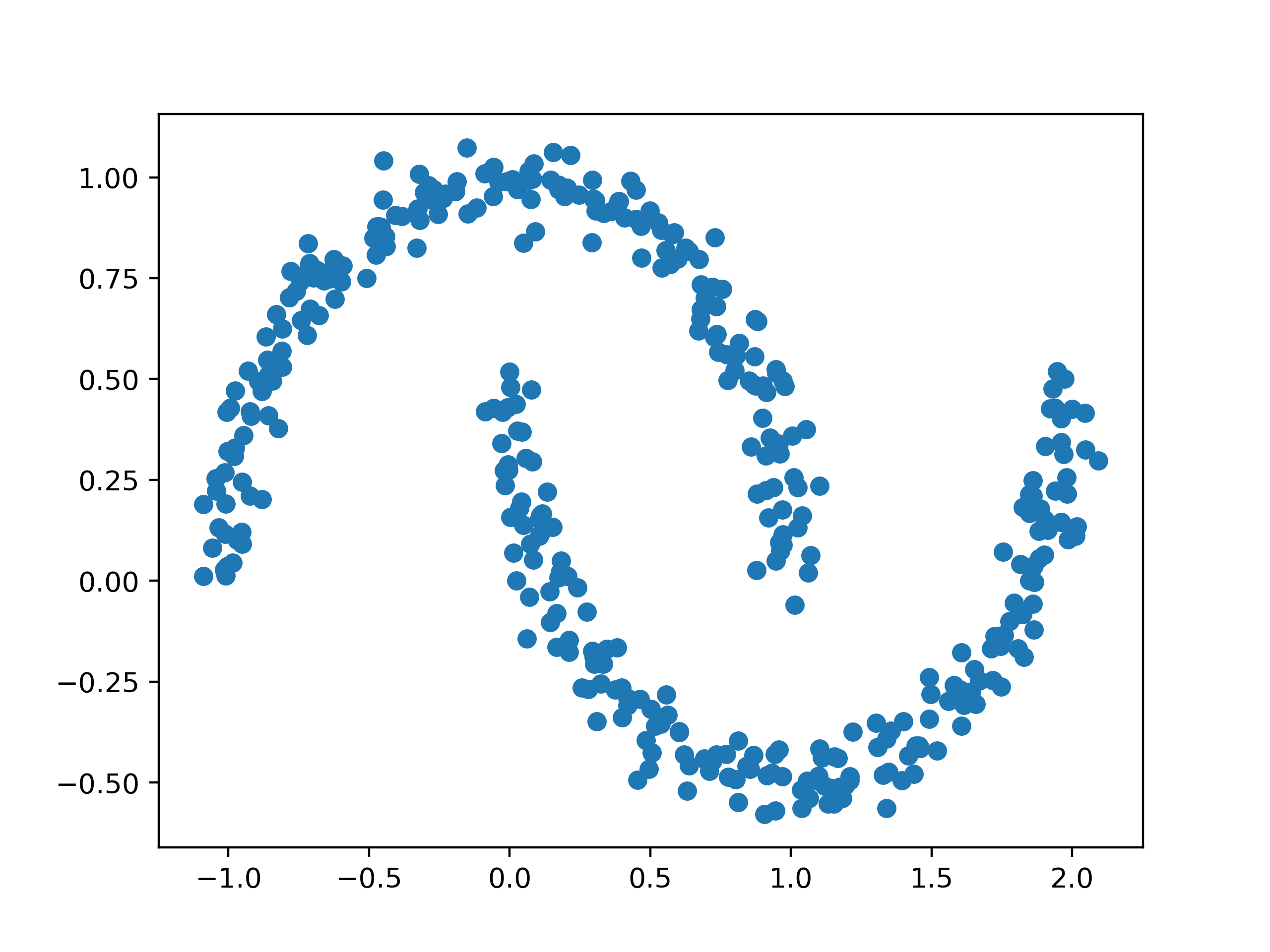}}
\subfloat[][Embedded points]{
\includegraphics[scale = 0.5]{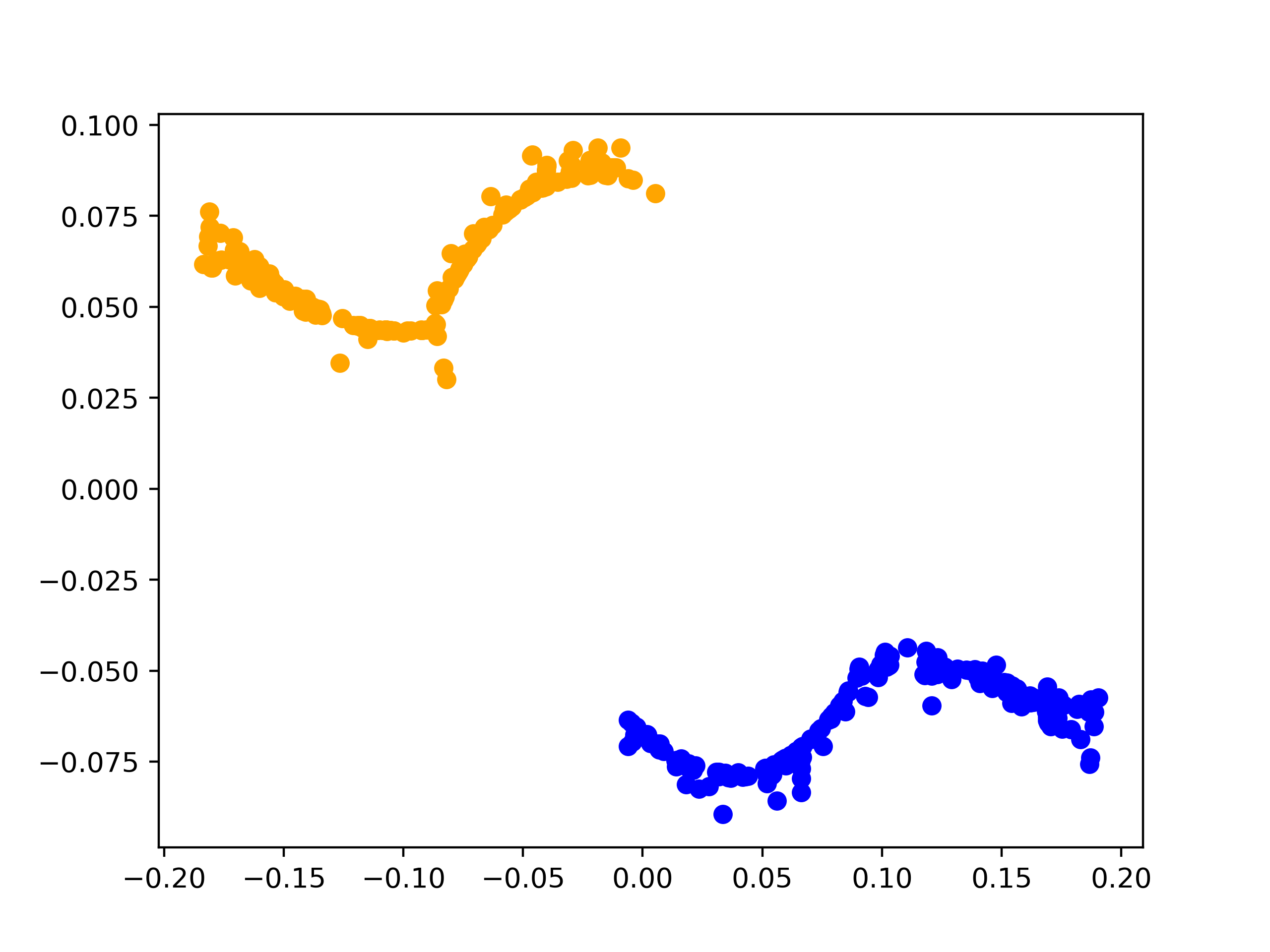}}
\hfill
\vskip 0.1in
\caption{We use the \texttt{make\_moons} function in \texttt{sci-kit learn} with 400 points, \texttt{noise = 0.05}. (a) The original moons in 2D space. (b) The re-embedded points, colored by the clusters as determined via sum-of-norms in this space. Perfect recovery is achieved.}
\label{fig:moon}
\end{figure}

Next, we investigate the performance of our algorithm on clustering problems which are known to be difficult. In particular, examples which are common in literature include anisomorphic blobs, half-moons, and concentric circles, and we demonstrate our method on them here. Figure \ref{fig:moon} shows results for the half-moons dataset (see Appendix \ref{app:figure_son} for anisomorphic blobs and concentric circles). We highlight the ability of sum-of-norms clustering to recover the correct clustering for the latter two, as we know sum-of-norms clustering is not able to cluster them in the original feature space due to nonconvexity. On the other hand, we see that our re-embedding achieves convex clusters, which are then recovered by standard sum-of-norms clustering.





\subsection{Other Clustering Algorithms}

Here we use the proposed pipeline approach, namely,
\begin{quotation}
leapfrog distances $\to$ MDS $\to$ dimension-reduced eigenvectors $\to$ clustering algorithm
\end{quotation}
with clustering algorithms other than SON clustering on the final step, namely $k$-means \cite{hartigan1979algorithm}, DBSCAN \cite{ester1996density} , and agglomerative clustering \cite{jain1988algorithms}.

Figure \ref{fig:moon_k} shows the result of $k$-means clustering (for $k=2$) on the half-moons dataset (identical to what was used in the sum-of-norms experiment). We see that $k$-means fails to recover the two moons in the original 2D space, but succeeds when using re-embedding (see Appendix \ref{app:figure_agg_clustering} for analogous results for agglomerative clustering when applied to concentric circles).

\begin{figure}[h]
\vskip -0.2in
\subfloat[][Original points]{
\includegraphics[scale = 0.5]{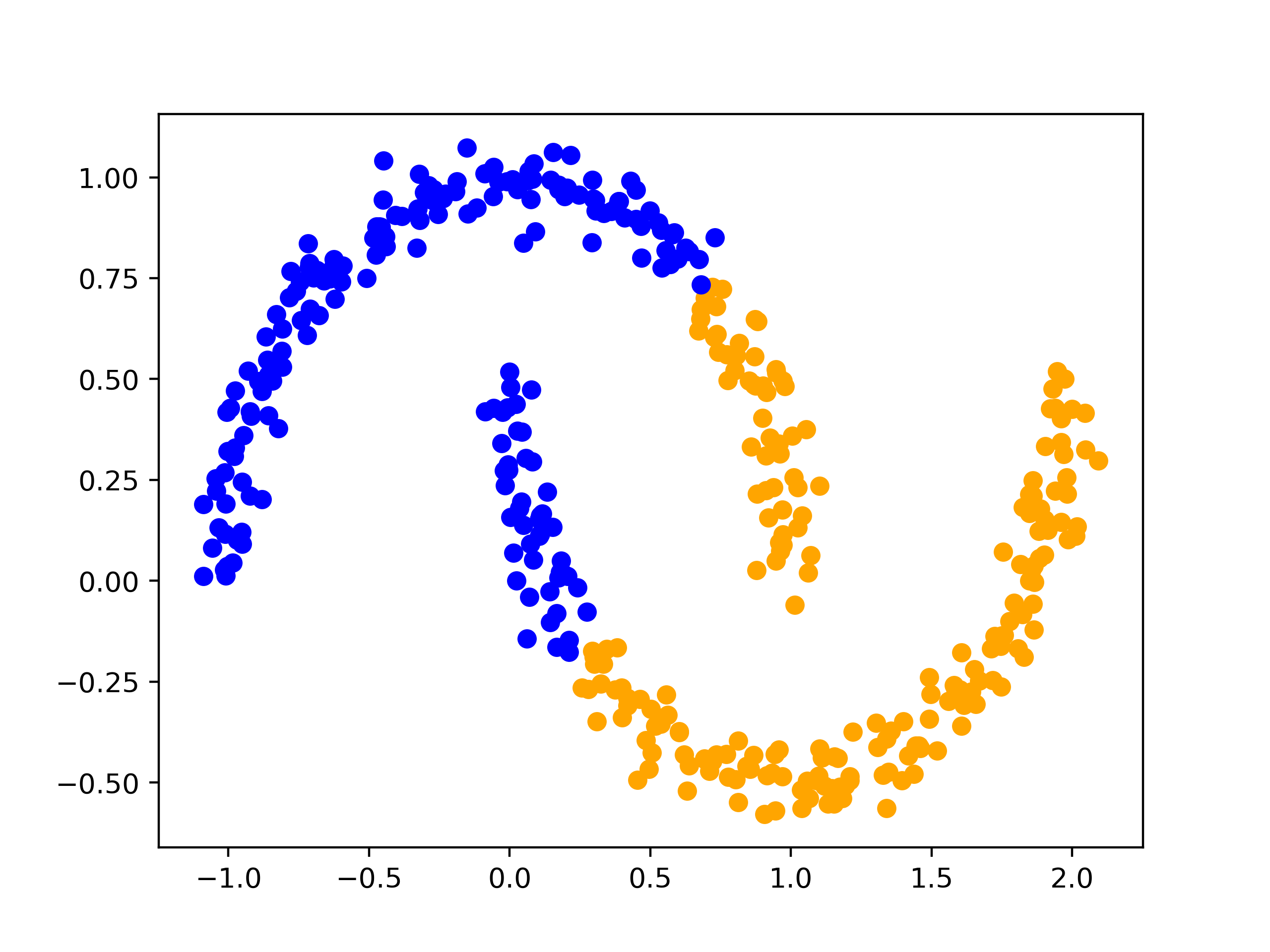}}
\subfloat[][Embedded points]{
\includegraphics[scale = 0.5]{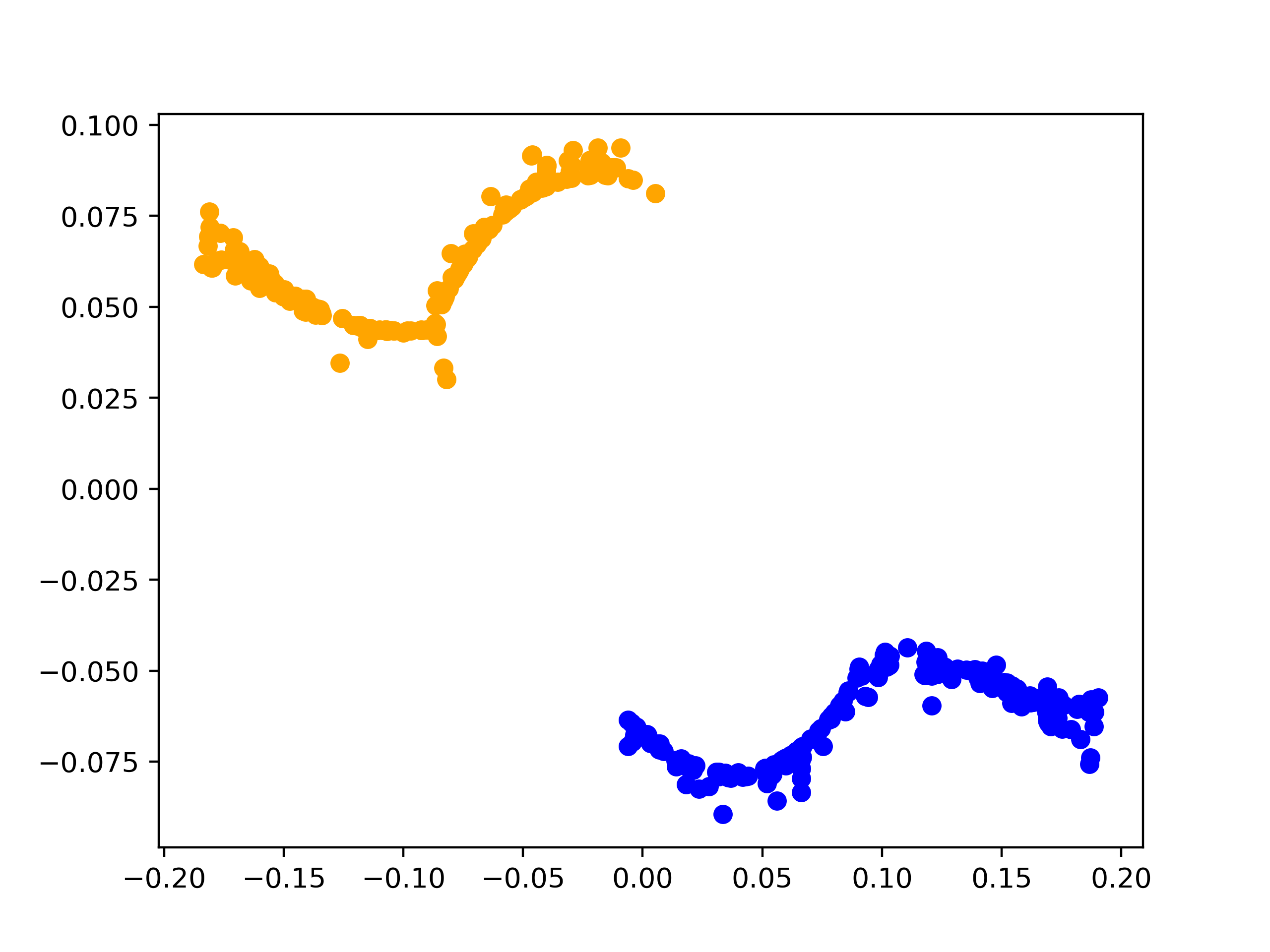}}
\hfill
\vskip 0.1in
\caption{The same moons dataset as in Figure \ref{fig:moon}, but we use $k$-means clustering instead in \texttt{scikit-learn}. (a) The original moons in 2D space, with colors denoting the two clusters. (b) The re-embedded points, colored by the clusters as determined via sum-of-norms in this space. Perfect recovery is achieved.}
\label{fig:moon_k}
\vskip 0.1in
\end{figure}

However, note that DBSCAN was able to recover clusters perfectly for both the half-moons and circles datasets in the original space (as well as the embedded space), and all three algorithms were able to recover the anisomorphic blob dataset perfectly in both cases. We conclude that our re-embedding can add value to use cases besides sum-of-norms clustering, as we see that it results in clustering results which are as good or better than the original space.

\section{Conclusion and future direction}
\label{sec:conc}
In this paper, we developed a method to adapt convex clustering techniques to discover non-convex clusters via a four-step pipeline. We do this by leveraging graphical structure between points to lift the dimension of the given dataset, thereby separating nonconvex clusters into convex clusters. We proved that our algorithm can successfully find clusters under reasonable assumptions about the underlying data, and showed its efficacy on some common non-convex clustering datasets.

Immediate open questions left by this research are as follows.  First, we do not have a precise characterization of the expected leapfrog distance for distributions in dimensions $d>1$.  As mentioned earlier, we conjecture that the behavior is $\Theta(n^{1/d})$. Second, we showed that for mixture of Gaussians even in one dimension, our proposed technique strengthens the sufficient conditions for every parameter set that we tried, but it would be interesting to prove that strengthening holds for all ranges of parameters.  However, prior to that, it would be useful to get stronger sufficient conditions or even exact conditions for recovery in both \cite{jiang2020recovery} and the new results herein.

Another immediate question is using different ingredients in the four-step pipeline. We already showed experiments with other clustering algorithms in Section~\ref{sec:exper}.  Other components are also replaceable, e.g. passing from distances to coordinates can be done using the basepoint technique rather than MDS and truncated eigenvectors.  The basepoint technique means selecting arbitrary distinct points from $\R^d$,  say $\y_1,\ldots,\y_L$, as basepoints and then defining coordinates of a point $\a_i$ to be the $L$-vector $(\mathrm{LF}(\a_i,\y_1)$,\ldots $\mathrm{LF}(\a_i,\y_L))$.  We conducted experiments with this technique (not reported herein) and found that MDS is superior, but we have not thoroughly explored all alternatives.

For future work, we would be interested in comparing our method with other unsupervised learning techniques, such as Isomap \cite{tenenbaum2000global}, which uses geodesic distances within a weighted graph to perform non-linear dimension reduction. Lifting the dimension of the data is also seen in kernel methods, e.g. in support vector machines as in \cite{amari1999improving}. Understanding the relationship between such methods and ours would also be of interest.

\bibliography{optimization}
\bibliographystyle{plain}

\newpage
\appendix
\appendix
\section{Derivation of Expectation of Leapfrog Distance in $\mathbb{R}^2$}
\label{app:ELF}
Let $X_1,\dots,X_n$ be $n$
independent samples from the probability density function $f$. For an arbitrary point $x \in \R$, define the random variable
\[
G(x) := \min \{X_k : X_k > x\} - \max\{X_k : X_k < x\}.
\]
In plain words, this is the distance between the consecutive points of our samples which are immediately to the left and right of $x,$ assuming both exist.
If either is missing, then $G(x)=\infty$.

\begin{lemma}
Let $X_1,\ldots,X_n$ be sampled according to $f$, and let $a,b$ be two samples  where $a < b$.
Then $\mathbb{E}[\mathrm{LF}(a,b)] = \int_a^b G(x) dx$.
\label{lem:LF_G}
\end{lemma}

\begin{proof}
Let $X(1), \ldots , X(n)$ be the order statistics. Say that $a=X(k)$ and $b=X(l)$ with $k<l$.

Then
\[
\mathrm{LF}(a,b) := \sum_{i=k+1}^l (X{(i)} - X{(i-1)})^2.
\]
This can be equivalently written as a sum of integrals:
\[
\mathrm{LF}(a,b) = \sum_{i=k+1}^l \int_{X{(i-1)}}^{X{(i)}} (X{(i)} - X{(i-1)}) \,dx.
\]
The latter can be expressed by $G(x)$ as
\begin{align*}
\mathrm{LF}(a,b)&= \sum_{i=k+1}^l \int_{X{(i-1)}}^{X{(i)}} G(x) \,dx =\sum_{i=k+1}^l  \int_{X{(k)}}^{X{(l)}} G(x)\, dx= \int_{a}^{b} G(x) \,dx,
\end{align*}
as desired. 
\end{proof}

Motivated by Lemma \ref{lem:LF_G}, we explore the properties of $G(x)$ in order to establish the behavior of the leapfrog distance. In particular, we find the closed-form expression for the expectation of $G(x)$ in Section~\ref{subsection:E[G(x)]},
for the expection of $\mathrm{LF}(a,b)$ in Sections~\ref{subsec:explf},
and use the results to establish a high-probability bound on the leapfrog distance in Section \ref{subsec:highprob}.

\subsection{Expectation of $G(x)$}
\label{subsection:E[G(x)]}
We now consider the cumulative distribution function of $G(x)$ as a means to establish the expectation of $G(x)$ that will guide later results in this section.  Our analysis is based on the following standard theorem (e.g., refer to Theorem 6.2 and Remark 6.4 in \cite{ghahramani2018fundamentals}).
\begin{equation}
    \mathbb E[G(x)]=\int_0^\infty \mathbb P[G(x)>\epsilon]\,d\epsilon.
    \label{eq:expec_from_cdf}
\end{equation}

We analyze $\mathbb{E}[G(x)]$ for $x\in U$. Of the ways to pick the $n$ data points according to $f$, we will construct a subset $\Pi$ that excludes cases whose probability is exponentially close to 0 as $n\rightarrow\infty$.  The excluded cases are identified in the lemma below.

Our approach of taking the expected value over all samples excluding an exponentially small set of bad cases, which may seem complicated to the reader,  is in fact necessary, as shown by the following example.  Consider evaluating $G(1)$
for $n$ samples of the normal distribution.  Regardless of how large $n$ is chosen, there is a small but positive probability that all samples will be less than 1; in these cases, $G(1)=\infty$.  Since $G(1)=\infty$ occurs with positive probability, $\mathbb E[G(1)]=\infty$.  In order to obtain a finite value of $\mathbb E[G(1)]$, bad cases must be excluded from the average.

\begin{lemma}
Let
\[
\sigma := n^{-0.75}v_U.
\]
Let
\[U_n := \{x\in U:[x-\sigma,x+\sigma]\subseteq U\}.\]
With probability exponentially close to $1$ as $n\rightarrow\infty$, $G(x)\le \sigma$ for all $x\in U_n$.
\label{lem:Gxupperbound}
\end{lemma}

\begin{proof}
Let $\ldots,x_{-2},x_{-1},x_0,x_1,x_2,\ldots$ be a grid of evenly spaced points in $\R$ separated by $\sigma/4$.
Let $\mathcal{I}_{i,n}$ be the event that there is no data point in $[x_i,x_{i+1}]$.  Let $P:=\{i:[x_i,x_{i+1}]\in U\}$  We see that for $i\in P$,
\begin{align*}
\mathbb{P}[\mathcal{I}_{i,n}] &=\left(1-\int_{x_i}^{x_{i+1}} f(x)\,dx\right)^n \\
&\le \left(1-\sigma\underline{f}/4\right)^n \\
&= \left(1- n^{-0.75}v_U\underline{f}/4\right)^n \\
&\le \left(\exp(- n^{-0.75}v_U\underline{f}/4)\right)^n \\
&= \exp(- n^{0.25}v_U\underline{f}/4).
\end{align*}
In the above chain of inequalities, we used the facts that $f(x)\ge \underline{f}$, $x_{i+1}-x_i=\sigma/4=n^{-0.75}v_U/4$, and, for any $t$, $1+t\le \exp(t)$.

Thus, for $i\in P$, the probability of $\mathcal{I}_{i,n}$ is exponentially small in $n$.  By the union bound, since $|P|=O(n^{0.75})$, it follows that the probability of $\bigcup_{i\in P}\mathcal{I}_{i,n}$ is exponentially small.  Assume therefore that none of these events occurs.

Consider an arbitrary $x\in U_n$, and suppose $x$ lies in the interval $[x_i,x_{i+1}]$.  By definition of $U_n$, it follows that the union of three intervals $[x_{i-1},x_{i+2}]$ lies in $U$, and therefore, by the assumption in the last paragraph, there is a data point in both $[x_{i-1},x_i]$ and $[x_{i+1},x_{i+2}]$.  Therefore, $G(x)\le x_{i+2}-x_{i-1}=0.75\sigma$.
\end{proof}

Let $\Pi$ denote the set of all ways to choose data points $a_1,\ldots,a_n$ according to $f$ such that the exponentially small set of bad cases in which $G(x)\ge \sigma(=n^{-0.75}v_u)$ for some $x\in U_n$ are excluded.  We adopt the notation $\mathbb{E}^\Pi[\cdot]$ and $\mathbb{P}^\Pi[\cdot]$ to indicate expectation or probability taken over $\Pi$ instead of over all ways to choose $a_1,\ldots,a_n$.
We evaluate $\mathbb{E}^\Pi[G(x)]$ using \eqref{eq:expec_from_cdf}.
Since we know that $G(x)\le \sigma$, we can replace the upper limit of integration with $\sigma$, i.e.,
\begin{equation}
\mathbb E^\Pi[G(x)]=\int_0^{\sigma} \mathbb P^\Pi[G(x)>\epsilon]\,d\epsilon.
\label{eq:expec_nh}
\end{equation}

Fix some $x\in U_n$.
To derive the closed-form expression for $\mathbb E^\Pi[G(x)]$, we make use of the following expression for the cumulative distribution function:
\begin{equation}
\begin{aligned}
    &\mathbb P^\Pi[G(x)>\epsilon]=\\
    &\int_{y=x-\frac{\epsilon}{2}}^x n f(y) \left (1 - \int_{z=y}^{y+\epsilon}f(z)\, dz\right)^{n-1} \,dy \\
    &+ \int_{y=x}^{x+\frac{\epsilon}{2}} n f(y) \left (1- \int_{z=y-\epsilon}^y f(z)\,dz\right)^{n-1} \,dy \\
    &\mbox{}+\left (1-\int_{y=x-\frac{\epsilon}{2}}^{x+\frac{\epsilon}{2}} f(y)\,dy \right)^n + R_n.
\end{aligned}
\label{eq:cdf_G}
\end{equation}
The first three terms on the right-hand side of \eqref{eq:cdf_G} sum exactly to $\mathbb P[G(x)>\epsilon]$ as follows.
The first term of \eqref{eq:cdf_G} comes from the case that there exists a point $y$ within $\frac{\epsilon}{2}$ of $x$ to the left and all other points either are to the left of $y$ or lie at least $\epsilon$ away from $y$  to the right. 
The second comes from the analogous case that there exists a point within $\frac{\epsilon}{2}$ of $x$ to the right and all other points are either to the right of $y$ or lie at least $\epsilon$ away from $y$ to the left. The third term comes from the case that no points exist within $\frac{\epsilon}{2}$ of $x$ to the left or right.
Finally, the last term $R_n$ arises from the difference between $\mathbb P[\cdot]$ and $\mathbb P^\Pi[\cdot]$.  For the cases excluded from $\Pi$, it is possible that $G(x)>\epsilon$ for all of them or for none of them, but this alters $\mathbb P[G(x)>\epsilon]$ only by the probability that the case is among the excluded cases.  In other words, $|R_n|\le O(\exp(-\mathrm{const}\cdot n^{0.25})).$

If we substitute \eqref{eq:cdf_G} into \eqref{eq:expec_nh}, then we obtain
\[
\mathbb E^\Pi[G(x)]=T_1+T_2+T_3+R_n',
\]
where $T_1$ stands for the integral with respect to $\epsilon$ of the first term on the right-hand side of \eqref{eq:cdf_G} over $[0,\sigma]$, and similarly for $T_2$ and $T_3$, while $|R_n'|\le O(n^{-0.75}\exp(-\mathrm{const}\cdot n^{0.25}))$.

We now approximate the integrands of $T_1,T_2,T_3$ using the mean value theorem for integrals.
Consider the first term of \eqref{eq:cdf_G}; by
the continuity of $f$, we may replace the inner integral as follows
\begin{align*}
&\int_{y=x-\frac{\epsilon}{2}}^x n f(y) \left (1 - \int_{z=y}^{y+\epsilon}f(z)\, dz\right)^{n-1}\, dy = \int_{y=x-\frac{\epsilon}{2}}^x n f(y) \left (1 - f(\xi(y)) \epsilon \right)^{n-1} \,dy
\end{align*}
for $\xi(y) \in [y, y+\epsilon]$. Using the mean value theorem again, we get the first term of the right-hand side of \eqref{eq:cdf_G} equals
$$n f(\eta) (1-f(\xi) \epsilon)^{n-1} \frac{\epsilon}{2},$$
for some $\eta \in [x-\frac{\epsilon}{2}, x]$ and $\xi \in [x-\frac{\epsilon}{2},x+\epsilon]$. Now finally we can
write
\[
T_1:=\int_{0}^\sigma n f(\eta) (1-f(\xi)\epsilon)^{n-1} \frac{\epsilon}{2} \, d\epsilon
\]
We obtain the identical formula for $T_2$ except that $\eta$ and $\xi$ are different functions of $\epsilon$ in the right-hand side.

We may also use the mean value theorem for the third term:
\begin{align*}
    \left (1-\int_{y=x-\frac{\epsilon}{2}}^{x+\frac{\epsilon}{2}} f(y)dy \right)^n = (1-f(\xi)\epsilon)^n
\end{align*}
for some $\xi \in [x-\frac{\epsilon}{2}, x+\frac{\epsilon}{2}]$
and thus
\[
T_3:=\int_0^\sigma(1-f(\xi)\epsilon)^n \,d\epsilon.
\]

We can now derive upper and lower bounds for all three terms $T_1, T_2, T_3$.
Start with $T_1$. Denote $\overline{f}(x):= \sup\{ f(y): y \in [x-\sigma, x+\sigma]\}$ and analogously $\underline{f}(x) = \inf\{f(y) :y \in [x-  \sigma, x+ \sigma]\} .$  Note that $\underline{f}(x)\ge\underline{f}$ by definition of $U_n$.
\begin{align*}
    T_1 &\leq \int_0^\sigma n \overline{f}(x)(1-\underline{f}(x)\epsilon)^{n-1} \frac{\epsilon}{2} d\epsilon\\
    &= -\frac{\overline{f}(x)}{\underline{f}(x)} (1 - \underline{f}(x) \epsilon)^n \frac{\epsilon}{2} \Bigg|_0^\sigma + \int_{0}^\sigma \frac{\overline{f}(x)}{\underline{f}(x)} (1 - \underline{f}(x)\epsilon)^n \frac{1}{2} d\epsilon \\
    &= - \frac{\overline{f}(x)}{\underline{f}(x)}\left(1 - \underline{f}(x)\sigma\right)^n \frac{\sigma}{2} - \frac{\overline{f}(x)}{2(n+1)\underline{f}(x)^2}(1 - \underline{f}(x)\epsilon)^{n+1} \Bigg|_0^\sigma\\
    &= -\frac{\overline{f}(x)}{\underline{f}(x)}\left(1 - \underline{f}(x)\sigma\right)^n \frac{\sigma}{2} - \frac{\overline{f}(x)}{2(n+1)\underline{f}(x)^2}\left(1 - \underline{f}(x)\sigma\right)^{n+1} + \frac{\overline{f}(x)}{2(n+1)\underline{f}(x)^2}.
\end{align*}
Note that integration by parts was used on the second line.

We observe the first two terms vanish exponentially fast as $n\to \infty$ since $\sigma=\mathrm{const}\cdot n^{-0.75}$.  Furthermore, $\overline{f}(x) \to f(x)$ and $\underline{f}(x) \to f(x)$ by the continuity of $f$ and the fact that $n^{-0.75}\rightarrow 0$, and therefore the last term behaves like $1/(2(n+1)f(x))$ as $n\rightarrow \infty$.

The above chain of inequalities may be reversed if we interchange the roles of $\underline{f}(x)$ and $\overline{f}(x)$.  Therefore,
\[
T_1 = \frac{1}{2(n+1)f(x)}+o\left(\frac{1}{n}\right).
\]
The speed of decay of the remainder term $o\left(\frac{1}{n}\right)$ in this bound depends on the strength of the continuity assumption imposed on $f$. For example, if we assume that $f$ is locally Lipschitz-continuous at $x$, then $\underline{f}(x)\ge f-L\sigma$ and $\overline{f}(x)\le f+L \sigma$. Analyzing the above remainders, and assuming $n$ is sufficiently large, we can derive the formula
that the remainder is more strongly bounded as $3\cdot L v_U/(n^{1.75}f(x)^2)$.
Since the formula for $T_2$ is identical to that of $T_1$, the same bound applies to $T_2$.

For $T_3$, note that
\begin{align*}
    T_3 \leq \frac{-1}{(n+1)\underline{f}(x)} (1 - \underline{f}(x) \epsilon)^{n+1} \Bigg|_0^\sigma    \leq \frac{1}{(n+1)\underline{f}(x)},
\end{align*}
where $\underline{f}(x) \to f(x)$ as $n \to \infty$
by construction and by continuity of $f$.  We obtain the opposite inequality by replacing $\underline{f}(x)$ with $\overline{f}(x)$.  Therefore, as in the previous analysis,
\[
T_3 = \frac{1}{(n+1)f(x)} + o\left(\frac{1}{n}\right),
\]
where, as above, we can more strongly bound the remainder as $3 Lv_U/(n^{1.75}f(x)^2)$ if we assume local Lipschitz continuity.

Adding the contributions from $T_1$, $T_2$, and $T_3$ leads to the following expression for $\mathbb E^\Pi[G(x)]$.
\begin{lemma}
The expectation of $G(x)$ for $x\in U_n$ has the following expression asymptotically as $n \to \infty$:
\begin{equation}
    \mathbb E^\Pi[G(x)] = \frac{2}{n f(x)} + o\left(\frac{1}{n}\right).
    \label{eq:E[G(x)]}
\end{equation}
The remainder is bounded in magnitude by
\[
\frac{9Lv_U}{ n^{1.75}f(x)^2}\]
if $f$ is locally $L$-Lipschitz continuous.
\label{lem:E[G(x)]}
\end{lemma}

Note that we changed $n+1$ to $n$ in the denominator since the difference between $1/n$ and $1/(n+1)$ is already smaller than the remainder $O(n^{-1.75})$.

\subsection{Expectation of Leapfrog distance}
\begin{manualtheorem}{\ref{thm:ELF_U}}
    The expectation of $\mathrm{LF}(a,b)$ has the following expression asymptotically as $n \to \infty$:
    \begin{equation}
    \mathbb E[\mathrm{LF}(a,b)]=\frac{2}{n}\int_a^b \frac{dx}{f(x)} + o(1/n),
    \end{equation}
    for $a,b \in U$, $a<b$ such that $[a,b]\subseteq U$.
    The remainder is bounded in magnitude by
    $$
    10n^{-1.5}v_U^2
    $$
    if $f$ is locally $L$-Lipschitz continuous.
\end{manualtheorem}

We start with deriving the expectation of $\mathrm{LF}(a,b)$ in a more restricted setting that $[a,b]\in U_n$, which implies that there is a positive distance between $a,b$ and the boundary $U$. The expectation of $\mathrm{LF}(a,b)$ follows the following lemma.

\begin{lemma}
    The expectation of $\mathrm{LF}(a,b)$ has the following expression asymptotically as $n \to \infty$:
    \begin{equation*}
    \mathbb E[\mathrm{LF}(a,b)]=\frac{2}{n}\int_a^b \frac{dx}{f(x)} + o(1/n),
    \end{equation*}
    for $a,b \in U_n$, $a<b$ such that $[a,b]\subseteq U_n$.
    The remainder is bounded in magnitude by
    $$
    10L v_U(b-a)/(n^{1.75}\underline{f}^2)
    $$
    if $f$ is locally $L$-Lipschitz continuous.
    \label{lem:ELF_Un}
\end{lemma}

\begin{proof}[Proof of Lemma \ref{lem:ELF_Un}]
    Let us first derive the expected value over $\Pi$ of $\mathrm{LF}(a,b)$ by combining Lemmas~\ref{lem:LF_G} and \ref{lem:E[G(x)]} with linearity of expectation:
    \begin{equation}
    \begin{aligned}
    \mathbb{E}^\Pi[\mathrm{LF}_n(a,b)]&= \mathbb{E}^\Pi\left [\int_a^b G(x)\, dx \right]
    = \int_a^b \mathbb{E}^\Pi[G(x)]\,dx
    = \frac{2}{n} \int_a^b \frac{dx}{f(x)} + R,
    \end{aligned}
    \label{eq:ELF1}
    \end{equation}
    where $|R|\le 9Lv_U(b-a)/(n^{1.75}\underline{f}^2)$ for a Lipschitz-continuous function $f$.

    Now, we consider the cases when the samples may lie outside the set $\Pi$.  Let $\bar\Pi$ denote the complement of $\Pi$ (i.e., ways to choose $a_1,\ldots,a_n$ so that for at least one $x\in U_n$, $G(x)>\sigma$).    For a sample of $\bar \Pi$, we have only the weak estimate $\mathrm{LF}(a,b)\in[0,v_U^2]$ since the length of $[a,b]$ is at most $v_U$.  The lower bound applies because distance is nonnegative, and the upper bound applies to the worst case that no data points lie in the interior of $[a,b]$, and therefore the leapfrog distance involves a single hop from $a$ to $b$. Thus, we have
    \begin{align}
    \mathbb E[\mathrm{LF}(a,b)] &=
    \mathbb P[\Pi]\mathbb\cdot \mathbb E^\Pi[\mathrm{LF}(a,b)] +
    \mathbb P[\bar\Pi]\cdot\mathbb E^{\bar\Pi}[\mathrm{LF}(a,b)] \notag\\
    &=(1-O(\exp(-O(n^{0.25})))) \cdot
    \left(\frac{2}{n}\int_a^b \frac{dx}{f(x)} + R
    \right) + O(\exp(-O(n^{0.25})))X\notag \\
    &=\frac{2}{n}\int_a^b \frac{dx}{f(x)} + R,\label{eq:ELF2}
    \end{align}
    where $|R|$ is bounded by $10L v_U(b-a)/(n^{1.75}\underline{f}^2)$ again for a Lipschitz function $f$.
    In this derivation, $X$ denotes an unknown number in $[0,v_U^2]$.
    We applied \eqref{eq:ELF1}, the fact that $\mathbb P[\bar\Pi]=O(\exp(-O(n^{0.25})))$, and the fact that exponentially small remainder terms can be absorbed asymptotically by the polynomially small remainder $O(n^{-1.75})$ by increasing the coefficient from 9 to 10.
\end{proof}

We can lift the restriction that $a,b\in U_n$ and still obtain a uniform estimate of $\mathbb E[\mathrm{LF}(a,b)]$.  We need to lift this restriction for the result in Subsection \ref{subsec:highprob}.

\begin{proof}[Proof of Theorem \ref{thm:ELF_U}]
    Focus on one of the intervals in $U$, say $[a_0,b_0]$. Assume two data points $a,b$ satisfy $a_0\le a < b\le b_0$.

    Let $\Omega$ denote the set of ways to choose $a_1,\ldots,a_n$ such that there is at least one data point in $[a_0+\sigma,a_0+2\sigma]$ and at least one in $[b_0-2\sigma,b_0-\sigma].$ (Note: assume that $n$ is chosen sufficiently large so that $4\sigma < b_0-a_0$.)
    Call these data points $a',b'$. Using the same analysis as in the proof of Lemma~\ref{lem:Gxupperbound}, these points $a',b'$ exist with probability exponentially close to $1$ as $n\rightarrow\infty$, and thus the complement of $\Omega$, say $\bar\Omega$, is exponentially small.

    Returning now to $a,b\in[a_0,b_0]$,  if both $a,b$ lie in $[a',b']$, then the equation \eqref{eq:ELF2} in the previous subsection shows that that
    $$\mathbb E^\Omega[\mathrm{LF}(a,b)]=\frac{2}{n}\int_a^b \frac{dx}{f(x)} + R,$$
    where $R$, as in the last subsection, satisfies
    $|R|\le 10Lv_U(b-a)/(n^{1.75}\underline{f}^2)$ for a Lipschitz-continuous function $f$, and otherwise $R=o(1/n)$.

    If $a<a'$, then
    \[
    \mathrm{LF}(a,a') \le |a-a'|^2\le 4\sigma^2=4n^{-1.5}v_U^2.
    \]
    The same holds for $\mathrm{LF}(b',b)$ in the case that $b>b'$.
    Thus, the leapfrog distance is increased by at most $8n^{-1.5}v_U^2$ compared to $\mathrm{LF}(a',b')$.  We will account for this increase by adding it to the remainder estimate $R$.  Note that $n^{-1.5}$ asympotically dominates $n^{-1.75}$ from \eqref{eq:ELF2}, so we will account for the sum of both terms by increasing the coefficient $8$ to $9$.  Thus, for $a_0\le a < b\le b_0$,
    \[
    \mathbb E^\Omega[\mathrm{LF}(a,b)]=\frac{2}{n}\int_a^b \frac{dx}{f(x)} + R
    \]
    where $R=o(1/n)$, and for a Lipschitz function,
    $|R|\le 9n^{-1.5}v_U^2$.

    Finally, adding in a term for the ways to choose $a_1,\ldots,a_n$ not in $\Omega$ can increase the expected value by only an exponentially small amount as in the last subsection, so we absorb this into the remainder by increasing the coefficient as in last subsection to estimate
    \[
    \mathbb E[\mathrm{LF}(a,b)]=\frac{2}{n}\int_a^b \frac{dx}{f(x)} + R
    \]
    where $|R|\le 10n^{-1.5}v_U^2$ for a Lipschitz function, else $R=o(1/n)$ in general.
\end{proof}

\section{Proof of high probability bound for leapfrog distance (Theorem \ref{thm:LF_main})}
\label{app:highprob}
\begin{manualtheorem}{\ref{thm:LF_main}}
 Assume that probability density function $f$ is Lipschitz continuous on its support.  Let $a,b$ be two data points drawn from an interval $I$ with the property that $f(x)\in[\underline{f},\overline{f}]$ for all $x\in I$ such that $\underline{f}>0$.
There exist constants $c_1, c_2, c_3>0$ and integer $n_0$, all of which may depend on $f$ and $I$, such that for all data points $a,b$ satisfying $a<b$, $a,b\in I$ and assuming $n>n_0$,
\[
\mathbb P\left[|\mathrm{LF}_n(a,b) - \mathbb E\left[\mathrm{LF}_n(a,b)\right]| \ge C n^{-1.04}\right] \le c_1 \cdot \exp(-c_2 n^{c_3}),
\]
where the coefficient $C$ is specified in \eqref{eq:mainhighprobineq} below.
\end{manualtheorem}
\begin{proof}
Divide the interval $[a,b]$ into $n^{0.9}$ subintervals, each of length $(b-a) n^{-0.9}$.
(Assume $n^{0.9}$ is an integer to simplify notation.) Let the $i$th subinterval be $[a_i, a_{i+1}]$, so that $a_1=a$ and $a_{n^{0.9}+1}=b$, and let $n_i$ denote the number of samples lying in  $[a_i, a_{i+1}]$.

We intend to apply Hoeffding's inequality to sum the leapfrog contribution from the subintervals.
Recall Hoeffding's inequality:  If $X_1,\ldots,X_n$ are independent nonnegative random variables all
bounded above by $u$, then for any $t>0$,
\[
\mathbb P\left[\left|\sum_{i=1}^n X_i-\sum_{i=1}^n\mathbb E[X_i]\right|>t\right]\le 2\exp(-2t^2/(nu^2)).
\]
Hoeffding's inequality requires that the summands be independent, but there is a weak coupling between the $n^{0.9}$ subintervals defined above due to the fact that the total number of samples $n$ is prespecified.  Therefore, the analysis requires additional arguments to assert independence.

Consider the following way to draw $n$ samples. Given prespecified nonnegative integers $n_1,\ldots,n_{n^{0.9}}$ that sum to at most $n$, choose precisely $n_1$ points at random from $[a_1,a_2]$ according to $f$ (restricted to this interval), then $n_2$ from $[a_2,a_3]$, etc., and the remaining $n-n_1-\cdots-n_{0.9}$ from the support of $f$ outside $[a,b]$. Let this sampling rule be denoted $\mathcal{S}(n;n_1,\ldots,n_{n^{0.9}})$.  The original sampling rule ($n$ points chosen according to $f$) is equivalent to choosing $n_1,\ldots, n_{n^{0.9}}$ at random according to the appropriate multinomial distribution and then sampling according to $\mathcal{S}(n;n_1,\ldots,n_{n^{0.9}})$.
Focus for now on a particular
$\mathcal{S}(n;n_1,\ldots,n_{n^{0.9}})$, where we assume that all of $n_1,\ldots,n_{n^{0.9}}$ are positive.  In other words, we discard from the analysis those cases in which any $n_i$ is zero.  It will be shown below (when we discard many more cases) that the case when an $n_i$ is zero has exponentially small probability.

Let $l_1,\ldots,l_{n^{0.9}}$ be the leftmost data points in $[a_1,a_2]$, \ldots, $[a_{n^{0.9}}, a_{n^{0.9}+1}]$, and similarly let $r_1,\ldots,r_{n^{0.9}}$ be the rightmost data points.
Recall that the two outer endpoints are assumed to be data points so that $l_1=a_1=a$ and $r_{n^{0.9}}=a_{n^{0.9}+1}=b$.  Then
\begin{align}
    \mathrm{LF}(a,b)&=\mathrm{LF}(l_1,r_1)+\mathrm{LF}(r_1,l_2)+\mathrm{LF}(l_2,r_2)+\mathrm{LF}(r_2,l_3)+\cdots+\mathrm{LF}(l_{n^{0.9}},r_{n^{0.9}})\notag\\
    &= \sum_{i=1}^{n^{0.9}} \mathrm{LF}(l_i,r_i) + \sum_{i=1}^{n^{0.9}-1} \mathrm{LF}(r_i,l_{i+1}).\label{eq:twolf}
\end{align}
We now present an analysis of the first summation based on Hoeffding's inequality.  Hoeffding's inequality is applicable because the data points in $[a_i,a_{i+1}]$ are independent of those from $[a_j,a_{j+1}]$ for $1\le i<j\le n^{0.9}$ since $n_i,n_j$ have been prespecified. Notice that $\mathrm{LF}(l_i,r_i)\le (a_{i+1}-a_i)^2=(b-a)^2n^{-1.8}$ since the $a_i\le l_i\le r_i\le a_{i+1}$ and the worst case for leapfrog distance is a hop from $l_i$ to $r_i$ with no intermediate points.
Thus, selecting $t=(b-a)^2n^{-1.1}/2$ in Hoeffding's inequality, we have
\begin{align}
&\mathbb P\left[\left|\sum_{i=1}^{n^{0.9}}\left(\mathrm{LF}(l_i,r_i)-\mathbb E[\mathrm{LF}(l_i,r_i)]\right)\right| > \frac{(b-a)^2n^{-1.1}}{2}\right]\le \exp\left(\frac{-2(b-a)^4n^{-2.2}}{4n^{0.9}(b-a)^4n^{-3.6}}\right) \\
&=\exp(-n^{0.5}/2).
\end{align}
We now discard the exponentially small number of cases when the inequality inside $\mathbb P[\cdot]$ holds.  In other words, for the remainder of the analysis, we assume that
\begin{equation}
\left|\sum_{i=1}^{n^{0.9}}\left(\mathrm{LF}(l_i,r_i)-\mathbb E[\mathrm{LF}(l_i,r_i)]\right)\right| \le \frac{(b-a)^2n^{-1.1}}{2}.
\label{eq:hoef1}
\end{equation}

We next analyze the term $\mathbb E[\mathrm{LF}(l_i,r_i)]$, which appears in \eqref{eq:hoef1}, using \eqref{eq:ELF_thm}.
Let $\mu_i$, $i=1,\ldots,n^{0.9}$, be the probability that a sample lies in $[a_i,a_{i+1}]$, i.e.,
\[
\mu_i:=\int_{a_i}^{a_{i+1}} f(\xi)\,d\xi.
\]
Note that
\begin{equation}
    \mu_i\in [n^{-0.9}(b-a)\underline{f},n^{-0.9}(b-a)\overline{f}].
    \label{eq:mubds}
\end{equation}
The $n_i$ samples in the interval $[a_i,a_{i+1}]$ are equivalently chosen from a restricted distribution whose PDF has the following form:
$$f_i(x)=\left\{\begin{array}{ll}
f(x)/\mu_i, & x\in [a_i,a_{i+1}], \\
0, & x\notin [a_i,a_{i+1}].\end{array}\right.$$
In applying \eqref{eq:ELF_thm}, we note that `$n$' in the bound now refers to $n_i$, while `$v_U$' refers to $a_{i+1}-a_i=(b-a)n^{-0.9}$
It follows from \eqref{eq:ELF_thm} that
\begin{equation}
\mathbb E[\mathrm{LF}(l_i,r_i)] =
\frac{2\mu_i}{n_i}\int_{l_i}^{r_i}\frac{dx}{f(x)}
+ R,
\label{eq:ELF4}
\end{equation}
with
\begin{align}
|R| &\le 10v_{[a_i,a_{i+1}]}^2n_i^{-1.5}=10(b-a)^2n^{-1.8}n_i^{-1.5} \label{eq:Rbound}
\end{align}
The analysis so far has been valid for any positive choice of $n_1,\ldots,n_{n^{0.9}}$.
In order to make progress on \eqref{eq:ELF4}, we need tighter bounds on $n_i$, so we next show that cases when any $n_i$ is far from its mean value occur with exponentially small probability and hence may be discarded.  First, we write down the mean and a lower bound on the mean as follows.
\begin{align}
\mathbb E[n_i] &= (n-2)\mu_i \label{eq:Eni}\\
&\ge (n-2)(a_{i+1} - a_{i}) \underline f \notag \\
&= (n-2) n^{-0.9} (b-a) \underline f \notag \\
&\ge (b-a) \underline f (n-2)^{0.09}. \label{eq:Enilb}
\end{align}

Note that $n_i$ is the outcome of $n-2$ independent Bernoulli trials with probability $\mu_i$.  Therefore, we can analyze the departure from its mean using the Chernoff bound, which is as follows.
\begin{lemma}
Let $X_i, \ldots, X_N \sim \text{Bernoulli}(p)$ be independent Bernoulli random variables with parameter $p$. Consider their sum $S_N = \sum_{i=1}^N X_i$. Then, for any $\delta \in (0,1)$, we have
\[
\mathbb P\left[|S_N - Np| \ge \delta Np\right] \le 2\exp(-Np \delta^2/3).
\]
\label{lemma:Chernoff}
\end{lemma}

A direct application of the Chernoff bound yields
the following.
\begin{claim}
The probability
\[
\mathbb P\left[|n_i - \mathbb E[n_i]| \ge \mu_i(n-2)^{0.96}\right]
\]
decays to zero exponentially fast with $n$.
\label{claim:ni}
\end{claim}
\begin{proof}[Proof of Claim \ref{claim:ni}]
We prove the claim by applying the Chernoff bound with $N=(n-2)$, $p=\mu_i$ and $\delta = (n-2)^{-0.04}$. Then
\begin{align*}
&\mathbb{P}\left[|n_i - \mathbb{E}[n_i]|\geq \delta \mathbb{E}[n_i]\right]\\
    &\leq 2\exp(-\mathbb{E}[n_i] \delta^2 /3) \\
    &= 2\exp(-\mathbb{E}[n_i](n-2)^{-0.08} /3)\\
    &\le 2 \exp(-(b-a) (n-2)^{0.09}  \underline{f}(n-2)^{-0.08}/3)\\
    &= 2 \exp(-(n-2)^{0.01} (b-a) \underline{f}/3),
\end{align*}
which decays to zero exponentially fast as desired.  The last line was obtained from \eqref{eq:Enilb}
\end{proof}

Let us now discard all choices of $n_1,\ldots,n^{0.9}$ in which any $n_i$ differs from $\mathbb E[n_i]$
by more than the amount in Claim~\ref{claim:ni}, so for the remainder of the proof, we assume that
\begin{align}
|n_i-\mu_i n| &\le \mu_i(n-2)^{0.96} \label{eq:for21}\\
&\le (b-a)n^{0.06}{\overline{f}},\label{eq:nimui}
\end{align}
where we used \eqref{eq:mubds} on the last line, and we replace $n-2$ with $n$ since the difference vanishes asymptotically.

Now we localize $l_i$ and $r_i$ as follows.
\begin{claim}
For each $i=1,\ldots,n^{0.9}$,
\[\mathbb P[l_i-a_i> (b-a)n^{-0.97}]\]
and
\[\mathbb P[a_{i+1}-r_i> (b-a)n^{-0.97}]\]
decay to zero exponentially fast with $n$.
\label{claim:bar_a}
\end{claim}

\begin{proof}[Proof of Claim \ref{claim:bar_a}]
Let $I$ denote the interval $[a_i, a_i + n^{-0.97}(b-a)]$. Then the probability that a sample misses $I$ is
\[
1 - \int_I f(x)\, dx \le 1 - |I| \underline f = 1 - n^{-0.97}(b-a) \underline f,
\]
so the probability that all samples miss $I$ is $\left(1 - n^{-0.97} \underline f(b-a)\right)^{n-2}$, which tends to 0 exponentially fast with $n$ as desired. Analogously, the probability that all samples miss the interval $[a_i - n^{-0.97}(b-a), a_i]$ also tends to 0 exponentially fast with $n$.
\end{proof}

Thus, we discard an exponentially small number of cases so that we can assume
\begin{equation}
l_i-a_i\le (b-a)n^{-0.97}\quad\mbox{and}\quad
a_{i+1}-r_i\le (b-a)n^{-0.97}.
\label{eq:liairi}
\end{equation}

Now we can estimate $\mathbb E[\mathrm{LF}(l_i,r_i)]$:
\begin{align}
&\left| \mathbb E[\mathrm{LF}(l_i,r_i)] -
\frac{2}{n}\int_{a_i}^{a_{i+1}}\frac{dx}{f(x)}\right| \notag\\
&\le
\left| \mathbb E[\mathrm{LF}(l_i,r_i)] -
\frac{2\mu_i}{n_i}\int_{l_i}^{r_i}\frac{dx}{f(x)}\right|
+\left|\frac{2\mu_i}{n_i}\int_{l_i}^{r_i}\frac{dx}{f(x)}-\frac{2}{n}\int_{a_i}^{a_{i+1}}\frac{dx}{f(x)}\right| \notag\\
&=
R+\left|\frac{2\mu_i}{n_i}\int_{l_i}^{r_i}\frac{dx}{f(x)}-\frac{2}{n}\int_{a_i}^{a_{i+1}}\frac{dx}{f(x)}\right| \notag\\
&\le
 R+\left|\frac{2\mu_i}{n_i}\int_{l_i}^{r_i}\frac{dx}{f(x)}-
 \frac{2}{n}\int_{l_i}^{r_i}\frac{dx}{f(x)}\right|
 +
 \left|
  \frac{2}{n}\int_{l_i}^{r_i}\frac{dx}{f(x)} -
 \frac{2}{n}\int_{a_i}^{a_{i+1}}\frac{dx}{f(x)}\right| \notag \\
 &=
 R+\left|\frac{2\mu_i}{n_i}-\frac{2}{n}\right|\cdot \int_{l_i}^{r_i}\frac{dx}{f(x)}
 +
 \frac{2}{n}
 \left|
 \int_{a_i}^{l_i}\frac{dx}{f(x)}+
 \int_{r_i}^{a_{i+1}}\frac{dx}{f(x)}
 \right|. \label{eq:longchain}
\end{align}
In the first line, we added and subtracted the same term. We applied \eqref{eq:ELF4} to obtain the second line with $R$ bounded by \eqref{eq:Rbound}. The third line is obtained again by adding and subtracting the same term. The fourth line is obtained by rearranging and cancelling out portions of the integrals that overlap.

We now upper-bound each factor on the right-hand side of \eqref{eq:longchain}.
We first note the following preliminary inequality:
\begin{equation}
    \frac{1}{n_i}\le \frac{1}{\mu_i n - \mu_i n^{0.96}} = \left(\frac{1}{1-n^{-0.04}}\right)\frac{1}{\mu_i n}  \leq \frac{2}{\mu_i n}\le \frac{2}{(b-a)n^{0.1}\underline{f}},
    \label{eq:nin}
\end{equation}
where the first inequality follows from \eqref{eq:for21} (and again replacing $n$ with $n-2$ as the difference is negligible asymptotically), the
second inequality holds for $n$ sufficiently large,
and the last from \eqref{eq:mubds}.
Turning to the right-hand side of \eqref{eq:longchain},
first,
\begin{align*}
|R| &\le 10(b-a)^2n^{-1.8}n_i^{-1.5} \\
&\le \frac{2^{1.5}\cdot 10(b-a)^2n^{-1.8}}{(b-a)^{1.5}n^{0.15}\underline{f}^{1.5}} \\
&= \frac{2^{1.5}\cdot 10(b-a)^{1/2}}{n^{1.95}\underline{f}^{1.5}},\\
\end{align*}
where the first line comes from \eqref{eq:Rbound} and the second from \eqref{eq:nin}.

Next,
\begin{align*}
\left|\frac{2\mu_i}{n_i}-\frac{2}{n}\right|
&= \frac{2}{nn_i}|\mu_i n - n_i| \\
&\le \frac{2}{nn_i}\cdot(b-a)\overline{f}n^{0.06}
\\
&\le \frac{4}{n^{1.1}(b-a)\overline{f}}\cdot(b-a)\underline{f}n^{0.06} \\
&=\frac{4\overline{f}}{n^{1.04}\underline{f}},
\end{align*}
where the second line comes from \eqref{eq:nimui} and the third from \eqref{eq:nin}.

Next,
\begin{align*}
\int_{l_i}^{r_i}\frac{dx}{f(x)}
&\le
\int_{a_i}^{a_{i+1}}\frac{dx}{f(x)} \\
&\le
(b-a)n^{-0.9}/\underline{f}
\end{align*}
Thus, the product of the two factors of the second term on the right-hand side of \eqref{eq:longchain} is at most $4(b-a)\overline{f}/(n^{1.94}\underline{f}^2)$.

Finally,
\begin{align*}
\int_{a_i}^{l_i}\frac{dx}{f(x)}
&\le (l_i-a_i)/\underline{f} \\
&\le (b-a)n^{-0.97}/\underline{f},
\end{align*}
where the second line uses \eqref{eq:liairi}.
The same bound holds for $\int_{r_i}^{a_{i+1}}\frac{dx}{f(x)}$.

We thus see that the first on the right-hand side of \eqref{eq:longchain} is proportional to $n^{-1.95}$,  the middle term is proportional to $n^{-1.94}$, and the third to $n^{-1.97}$  Thus, the middle term dominates, so \eqref{eq:longchain} yields
\[
\left| \mathbb E[\mathrm{LF}(l_i,r_i)] -
\frac{2}{n}\int_{a_i}^{a_{i+1}}\frac{dx}{f(x)}\right|
\le \frac{5(b-a)\overline{f}}{n^{1.94}\underline{f}^2},
\]
where we increased the coefficient from 4 to 5 to account for the two lower-order terms.  Thus, adding all $n^{0.9}$ such terms yields:
\[
\sum_{i=1}^{n^{.9}}\left| \mathbb E[\mathrm{LF}(l_i,r_i)] -
\frac{2}{n}\int_{a_i}^{a_{i+1}}\frac{dx}{f(x)}\right|
\le \frac{5(b-a)\overline{f}}{n^{1.04}\underline{f}^2},
\]

Add this to \eqref{eq:hoef1}, apply the triangle inequality, and combine the two remainders on the right-hand side.  Notice that the previous remainder of $n^{-1.04}$ dominates the right-hand side of \eqref{eq:hoef1}, so we write only the above right-hand side and increase the coefficient to 6 to account for the other term:
\[
\left|\sum_{i=1}^{n^{0.9}}\left(\mathrm{LF}(l_i,r_i)-\frac{2}{n}\int_{a_i}^{a_{i+1}}\frac{dx}{f(x)}\right)\right|
\le \frac{6(b-a)\overline{f}}{n^{1.04}\underline{f}^2}
\]
Note that the integrals sum to a single integral over $[a,b]$, which in turn is $\mathbb E[\mathrm{LF}(a,b)]$. Thus, we have
\[
\left|\left(\sum_{i=1}^{n^{0.9}}\mathrm{LF}(l_i,r_i)\right)-
\mathbb E[\mathrm{LF}(a,b)]\right|
\le \frac{6(b-a)\overline{f}}{n^{1.04}\underline{f}^2}
\]

We now have an estimate for the first term in \eqref{eq:twolf}, so let us next dispense with the second term. We have already assumed with probability exponentially close to 1 that $l_{i+1}-r_i\le 2(b-a)n^{-0.97}$ in \eqref{eq:liairi}.  Thus,

\begin{align*}
\sum_{i=1}^{n^{0.9}-1}\mathrm{LF}(r_i,l_{i+1})&=
\sum_{i=1}^{n^{0.9}-1}(l_{i+1}-r_i)^2 \\
&\le n^{0.9}\cdot 4(b-a)^2n^{-1.94}\\
&= 4(b-a)^2/n^{1.04}.
\end{align*}
Thus, we can add the two previous displayed equations, apply the triangle inequality, and note that the two summations together yield $\mathrm{LF}(a,b)$ according to \eqref{eq:twolf} to obtain
\begin{equation}
|\mathrm{LF}(a,b)-\mathbb E[\mathrm{LF}(a,b)]|
\le \left(\frac{6(b-a)\overline{f}}{\underline{f}^2}+4(b-a)^2\right)n^{-1.04}.
\label{eq:mainhighprobineq}
\end{equation}
All the discarded cases occur with probability exponentially small as $n\rightarrow\infty$.
\end{proof}

\section{Proof of Theorem \ref{thm:uppbdlf_dimd}}
\label{app:LF_main_proof}
\begin{manualtheorem}{\ref{thm:uppbdlf_dimd}}
Assume $f$ is admissible.  Let $n$ points in $\R^d$ be sampled according to $f$.  Then, with probability exponentially close to $1$,  for any $i=1,\ldots, K$, for any two samples $\a,\b\in S_i$, $\LF(\a,\b)\le 36\ell n^{-1/d+\eta}$
for an arbitrarily small $\eta>0$.
\end{manualtheorem}

\vspace{.1in}

Fix a particular $i\in\{1,\ldots,K\}$.
Let $\mu$ denote the measure induced by $f$.
Fix $\eta>0$ such that $-1/d+\eta <0$.
Before proving the theorem, we present a sequence
of lemmas.  Construct an epsilon-net inside $S_i$, where
$\eps=n^{-1/d+\eta}$, using the following procedure.
\begin{tabbing}
  ++\=++\=\kill
  \> $U:=\emptyset$ \\
  \> $W:=\emptyset$ \\
  \> while $S_i\not\subseteq W$ \\
  \>\> Choose arbitrary $\x\in S_i\setminus W$ \\
  \>\> $U:=U\cup \{\x\}$ \\
  \>\> $W:=W \cup B(\x,\eps)$ \\
\end{tabbing}
\begin{lemma}
  The above process terminates, and
  $|U|\le 1/(\theta \cdot p(\eps/2)^d)$.  Furthermore,
  upon termination, for
  any $\y\in S_i$, there exists $\x\in U$ such that
  $\Vert \x-\y\Vert \le \eps.$
  \label{lem:epsnet}
\end{lemma}

\begin{proof}
  Each $\x\in U$ chosen by the procedure is distance at least $\eps$ from
  every other point in $U$.
  Therefore, the balls in the collection $V=\{B(\x,\eps/2):\x\in U\}$
  are disjoint.  By property 5, each ball in $V$ contains
  volume at least $p(\eps/2)^d$ of $S_i$.
  Since the $\vol(S_i)\le 1/\theta$ is finite, there
  are at most a finite number of such balls.  In particular,
  there can be at most the quantity claimed in the lemma.
  The second claim follows because if, during the procedure, there is $\y\in S_i$ such that $\dist(\y,U)>\eps$, then $\y\notin W$, and hence the procedure cannot terminate.
\end{proof}

\begin{lemma}
  Assume $\phi\in (-1,0)$ and $a>0$.  Then $(1-an^\phi)^n\rightarrow 0$
  exponentially fast as $n\rightarrow \infty$.
\end{lemma}
\begin{proof}
  Since $1+x\le e^x$, then $(1-an^\phi)\le \exp(-an^\phi)$
  and thus $(1-an^\phi)^n\le \exp(-an^{\phi+1})$.  Since $\phi+1>0$,
  this proves the result.
\end{proof}

\begin{lemma}
  Let $\x\in S_i$ be arbitrary.
  With probability exponentially close to $1$ as $n\rightarrow\infty$,
  $B(\x,\eps)$
  contains a data point.
\end{lemma}
\begin{proof}
  By property 5, $\vol(B(\x,\eps)\cap S_i)\ge p\eps^d$.
  Therefore, $\mu(B(\x,\eps))\ge \theta p\eps^d$.
  If $R$ denotes $S_i\setminus B(\x,\eps)$, then
  $\mu(R)\le 1-p\theta\eps^d$.  Thus, the probability that all $n$ data points
  lie in $R$ is at most $(1-p\theta\eps^d)^n=(1-p\theta (n^{-1/d+\eta})^d)^n=
  (1-p\theta n^{-1+\eta d})^n$.   By the preceding lemma, this tends
  to 0 exponentially fast.
\end{proof}

\begin{lemma}
With probability exponentially close to $1$ as
  $n\rightarrow\infty$,
  for every $\x\in U$,  $B(\x,\eps)$ contains a data point.
\end{lemma}

\begin{proof}
  This follows from the union bound: the probability that any particular
  ball fails to contain a point is exponentially small in $n$ by the preceding lemma, and the
  number of balls is at most $1/(\theta p (\eps/2)^d)=2^d/(\theta p n^{-1+\eta d})$
  which is polynomial in $n$.
\end{proof}

\begin{proof}[Proof of Theorem~\ref{thm:uppbdlf_dimd}]
We prove the second statement of the theorem first.  Let $\x\in S_i$ be arbitrary.  By the epsilon-net property, there exists a $\y\in U$ such that $\Vert \x-\y\Vert\le \eps$.  By the preceding lemma, there exists an $\a_i$ such that $\Vert \y-\a_i\Vert\le \eps$.  Therefore, $\Vert \x-\a_i\Vert\le 2\eps$.

  Let $\a,\b$ be two data points in $S_i$.  Let $P\subset S_i$ be a path from
  $\a$ to $\b$ of length at most $\ell$.  Choose evenly spaced points
  $\x_0(=\a),\ldots,\x_k(=\b)\in P$
  such that the distance between successive points is
  between $\eps$ and $2\eps$.
  Note that $k\le \ell/\eps$. By the argument in the previous paragraph, for each $\x_j$, there is a data point $\a_j$ such that
  $\Vert \a_j-\x_j\Vert\le 2\eps$.  Therefore, for any $j=0,\ldots,k-1$,
  by the triangle inequality,
  \begin{align*}
    \Vert \a_j-\a_{j+1}\Vert &\le \Vert \a_j-\x_j\Vert +\Vert\x_j-\x_{j+1}\Vert + \Vert \a_{j+1}-{\x_{j+1}} \Vert\\
    &\le 2\eps +2\eps + 2\eps  \\
    &=6\eps
  \end{align*}
  Consider the path of data points from $\a$ to $\b$ given by $\a,\a_1,\a_2,\ldots,
  \a_{k-1}\b$.  Then $\mathrm{LF}(\a,\b)$ is at most the length determined by this
  particular path:
  \begin{align*}
    \mathrm{LF}(\a,\b)&\le k\cdot(6\eps)^2 \\
    &\le (\ell/\eps)\cdot 36\eps^2 \\
    &=36\ell\eps
  \end{align*}
\end{proof}
Note that the bound $36 \ell n^{-1/d+\eta}$ claimed in Theorem~\ref{thm:uppbdlf_dimd} means that the leapfrog distance between any two data points in $S$ is $o(1)$ as $n\rightarrow\infty.$

\section{Derivation of properties of $\b_i$'s}
\label{app:bi}
Our analysis uses the following steps:
\begin{enumerate}
    \item Decompose the leapfrog distance matrix $D$ into a low-rank leapfrog distance matrix $\bar D$ and a noise matrix $E$;
    \item Construct re-embeddings $\{\bar \b_i\}_{i=1}^n$ from the low-rank leapfrog distance matrix $\bar D$ using multidimensional scaling;
    \item Prove that the clusters of the re-embeddings $\{\bar \b_i\}_{i=1}^n$ are easily identifiable;
    \item Prove that the original re-embeddings $\{\b_i\}_{i=1}^n$ are a good proxy for $\{\bar \b_i\}_{i=1}^n$ with high probability, and conclude $\{\b_i\}_{i=1}^n$ is also 
    identifiable.
\end{enumerate}


Notationally, we let $\norm{\cdot}_p$ denote the operator $p$-norm with $\norm{A}_p = \sup_{\|v\|_p=1} \norm{Av}_p$. 
We use $\norm{\cdot}_{p,q}$ denote the entry-wise matrix norm defined by $\norm{A}_{p,q}=\norm{\left (\norm{\a_1}_p, \dots, \norm{\a_d}_p \right)}_q$. Finally, we use $\norm{\cdot}_F$ to denote the Frobenius norm, that is, $\|A\|_F = \sqrt{\sum_{i}\sum_{j} A(i,j)^2}$.



\subsection{Decomposition of $D$}
\label{sec:decom_D}

Let $\bmu_k$ denote an arbitrary data point that lies in $C_k$, $k=1,\ldots,K$.  (Note that the result of this section is a high-probability result, so therefore we can simply discard the exponentially rare case that one of the $C_k$'s is empty.)
We define $\bar D$ according to the formula
\[
\bar D(i,j):=\mathrm{LF}(\bmu_{c(i)},\bmu_{c(j)})^2,
\]
where $c(i)$ denotes the $k$ such that $\bm{a}_i\in C_k$ and similarly for $c(j)$.
As the columns corresponding to the same cluster are identical, we have that $\mathrm{rank}(\bar D) \le K$.
Note that $\bar D$ cannot be algorithmically constructed since the clusters are not known in advance, but nonetheless, we prove that in the limit $n\rightarrow\infty$, it approximates $D$.

Next, define a noise matrix $E:=D-\bar D$ to decompose $D$ into the sum of a low-rank matrix $\bar D$ and a noise matrix $E$. Since the Gram matrix computation is a linear operator, the Gram matrix of $D$ can be rewritten as follows:
\begin{equation}
    G(D) = G(\bar D + E) = G(\bar D) + G(E),
    \label{eq:new_Gram}
\end{equation}
We analyze the Gram matrix $G(D)$ through the lenses of $G(\bar D)$ and $G(E)$.

\subsubsection*{Structure of $G(\bar D)$}

\begin{lemma}
The rank of $G(\bar D)$ is at most $K-1$.
  Aside from the $n-(K-1)$ zero eigenvalues, the remaining
  eigenvalues of $G(\bar D)$, with probability exponentially close
  to $1$ as $n\rightarrow \infty$, are either
  \begin{enumerate}[label=(\alph*)]
      \item bounded below in magnitude by $\sigma n$ and bounded above in magnitude by $\tau n$, where
      $\sigma>0, \tau>0$ do not depend on $n$
      and are defined in   \eqref{eq:sigmadef} and \eqref{eq:taudef} respectively, or
      \item bounded above in magnitude by $O(n^{2/3})$.
  \end{enumerate}
  \label{lem:eigGbarDa}
\end{lemma}

\begin{proof}
  Let
  \begin{equation}
  p_m:=\int_{S_m}f(\x)\,d\x,
  \label{eq:pmdef}
  \end{equation}
  for $m=1,\ldots,K$.  By assumption, $p_m>0$ for $m=1,\ldots,K$ and $p_1+\cdots+p_K=1$. Let
  $A_0$ be the $K\times K$
  symmetric hollow matrix whose $(i,j)$th entry is $\mathrm{LF}(\bmu_i,\bmu_j)^2$ defined above.
  Let $G(A_0)$, as usual, be the Gram matrix corresponding to $A_0$.
  Let $F$ be the diagonal matrix whose $i$th diagonal entry is
  $\sqrt{p_i}$.  We will show that the eigenvalues
  of $G(\bar D)$ are closely related to those of the symmetric matrix
  $FG(A_0)F$.  Note that the $(i,j)$th entry of $FG(A_0)F$ is
  $\sqrt{p_ip_j}G(A_0)(i,j)$.  Let $L$ denote the number of nonzero eigenvalues of $G(A_0)$.
  Since $G(A_0)$ has $L$ nonzero
  eigenvalues, so does $FG(A_0)F$.
Define
  \be
  \sigma:=\min\{|\lambda|:\mbox{ $\lambda$ is a nonzero eigenvalue of $FG(A_0)F$}\}/2.
  \label{eq:sigmadef}
  \ee
  and
  \be
  \tau:=\max\{|\lambda|:\mbox{ $\lambda$ is a nonzero eigenvalue of $FG(A_0)F$}\}\cdot 2.
  \label{eq:taudef}
  \ee

  Suppose the first $n_1$ nodes are chosen from cluster 1, $n_2$ from cluster 2,
  etc., up to $n_K$.    Then $n_1+\cdots+n_K=n$.
  Note that $\bar D$ has the following structure. The first
  $n_1$ columns of $\bar D$ are identical, as are the next $n_2$ columns,
  and so on.  Thus, $\bar D$ is a $K\times K$ block matrix in which the
  $(i,j)$ block consists of an $n_i\times n_j$ submatrix of identical
  entries, which are equal to $A_0(i,j)$.
    Since the first row and column
  also agree blockwise with $A_0$, it follows that $G(\bar D)$ is also
  composed of $K\times K$ blocks, with constant values in each block.
Therefore, $G(\bar D)$ has only $K$ distinct columns, so its rank is at most $K$, i.e., $n-K$ of its eigenvalues equal 0. In fact, more strongly, $\rank(G(\bar D))\le K-1$ since all entries in the first block are zeros. This proves the first statement of the lemma.

  For the second part of the lemma, consider a candidate eigenvector $\v$ of $G(\bar D)$
  that has entry $u_1$ repeated $n_1$ times, $u_2$ repeated $n_2$ times, and so
  on up to $u_K$.  It follows that $\w:=G(\bar D)\v$ is composed of $K$ blocks
  with identical entries in each block.  In particular, the entries of the
  first block all equal $n_2G(A_0)(1,2)u_2+n_3G(A_0)(1,3)u_3+\cdots+n_KG(A_0)(1,K)u_K$,
  and so on for the remaining entries.  Thus, $\v$ is an eigenvector iff there
  exists a $\lambda$ such that
  \begin{align*}
    \lambda u_1 &= n_2G(A_0)(1,2)u_2+\cdots+n_KG(A_0)(1,K)u_K, \\
    &\vdots \\
    \lambda u_K &=n_1G(A_0)(K,1)u_1 +\cdots + n_{K-1}G(A_0)(K,K-1)u_{K-1}.
  \end{align*}
  By multiplying the $i$th equation by $\sqrt{n_i}$, this system of
  equations may be rewritten
  \begin{align*}
    \sqrt{n_1}\lambda u_1 &= \sqrt{n_1n_2}G(A_0)(1,2)(\sqrt{n_2}u_2)+\cdots\\
    &\hphantom{\le}\quad\mbox{} +
       \sqrt{n_1n_K}G(A_0)(1,K)(\sqrt{n_K}u_K), \\
    &\vdots \\
       \sqrt{n_K}\lambda u_K &=\sqrt{n_1n_2}G(A_0)(K,1)(\sqrt{n_1}u_1) +\cdots \\
    &\hphantom{\le}\quad\mbox{}+
       \sqrt{n_{K-1}n_K}G(A_0)(K,K-1)(\sqrt{n_{K-1}}u_{K-1}).
  \end{align*}
  In other words, the vector
  $$\left(\begin{array}{c}
  \sqrt{n_1}u_1 \\
  \vdots\\
  \sqrt{n_K}u_K
  \end{array}\right)$$
  is an eigenvector of the symmetric hollow
  matrix $\tilde A$ whose $(i,j)$th entry is $\sqrt{n_i n_j}G(A_0)(i,j)$.  Conversely, any eigenvalue/eigenvector pair of $\tilde A$ gives rise to an eigenvalue/eigenvector pair of $G(\bar D)$ by the same construction.
  One of these eigenvectors, $[1,\ldots,1,0,\ldots,0]^T$, has eigenvalue 0.
  The remaining $K$ possibly nonzero eigenvalues of $G(\bar D)$ arise from this formula.

  By using Chernoff bounds, with probability
  exponentially close to 1, $p_i n - n^{2/3} \le n_i \le p_i n+n^{2/3}$
  for all $i=1,\ldots, K$, which means $p_i-n^{1/3}\le  n_i/n \le p_i+n^{1/3}.$
  In other words, with probability exponentially close to 1,
  $$\|\tilde A/n - FG(A_0)F \|_F = O(n^{-1/3}).$$
  Since eigenvalues are continuous with respect to small perturbations,
  we conclude that for $n$ sufficiently large, $L$
  eigenvalues of $\tilde A$
  have magnitude at least equal to half those of $nFG(A_0)F$
  and at most twice those of $nFG(A_0)F$. The remaining $K-L$ eigenvalues have magnitude $O(n^{2/3})$.
\end{proof}



\subsubsection*{Structure of $G(E)$}
\begin{lemma}
Suppose the $f(\x)$ satisfies the admissibility conditions in Section $\ref{sec:LF_Rd}$. Then with probability exponentially close to $1$,
\begin{equation}
\begin{aligned}
&\norm{G(E)}_{2, \infty} \le 3 c_g n^{1/2-\gamma} = o(\sqrt{n}), \\
&\norm{G(E)}_{F} \le 3 c_g n^{1-\gamma} = o(n), \\
&\norm{G(E)}_{2} = 3 c_g n^{1-\gamma} = o(n),
\end{aligned}
\label{eq:E_F_2_norm}
\end{equation}
where $c_g$ and $\gamma$ do not depend on $n$
and are defined below by \eqref{eq:cgdef}.
\label{lem:E_2_norm}
\end{lemma}

\begin{proof}
Recall that each entry of $E$ is
\begin{align*}
    E(i,j)&=\mathrm{LF}(\a_i,\a_j)^2-\mathrm{LF}(\bmu_{c(i)},\bmu_{c(j)})^2\\
    &=(\mathrm{LF}(\a_i,\a_j)+\mathrm{LF}(\bmu_{c(i)},\bmu_{c(j)})) \\
    &\hphantom{\le}\quad\mbox{} \cdot  (\mathrm{LF}(\a_i,\a_j)-\mathrm{LF}(\bmu_{c(i)},\bmu_{c(j)}))\\
    &\le 2\,\mathrm{diam}(S)^2 (\mathrm{LF}(\a_i,\a_j)-\mathrm{LF}(\bmu_{c(i)},\bmu_{c(j)}))\\
    &= 2\,\mathrm{diam}(S)^2(\mathrm{LF}(\a_i,\a_j)- \mathrm{LF}(\a_i,\bmu_{c(j)}) \\
    &\hphantom{\le}\quad\mbox{} + \mathrm{LF}(\a_i,\bmu_{c(j)}) -\mathrm{LF}(\bmu_{c(i)},\bmu_{c(j)}))\\
    &\le 2\,\mathrm{diam}(S)^2(\mathrm{LF}(\a_i,\bmu_{c(i)})+\mathrm{LF}(\a_j,\bmu_{c(j)}))\\
    &\le 72\, \mathrm{diam}(S)^2\ell n^{-1/d+\eta}.
\end{align*}
In the third line, we used the fact that the leapfrog distance between any two points is bounded by the square of their Euclidean distance. Note that $\mathrm{diam}(S)$ is finite since $S$ is the union of $K$ compact sets.   Also, $\ell$ is the upper bound on path length introduced in Section~\ref{sec:LF_Rd}, and $\eta$ is the arbitrarily small positive number from Theorem~\ref{thm:uppbdlf_dimd}.
Define
\begin{equation}
c_g:= 72\,\mathrm{diam}(S)^2\ell
\mbox{ and } \gamma:=1/d-\eta.
\label{eq:cgdef}
\end{equation}

Hence, the absolute value of each entry of $E$ is at most $c_g n^{-\gamma}$. By the construction of $G(E)$, we get that the absolute value of each entry of $G(E)$ is upper bounded by $3 c_g n^{-\gamma}$. Thus, the $(2,\infty)$-norm has the following upper bound:
\begin{align*}
&\norm{G(E)}_{2,\infty} \le \\
&\max_{i=1,\dots,n}\sqrt{n \cdot (3 c_g n^{-\gamma})^2} = 3 c_g n^{1/2-\gamma} = o(\sqrt{n}),
\end{align*}
the Frobenius norm of $G(E)$ has the following upper bound:
\[
\norm{G(E)}_{F} \le \sqrt{n^2 \cdot (3 c_g n^{-\gamma})^2} = 3 c_g n^{1-\gamma} = o(n).
\]
Moreover, the operator 2-norm of $G(E)$ is upper bounded by:
%
\[
\norm{G(E)}_2 \le \norm{G(E)}_{F} \le 3 c_g n^{1-\gamma} = o(n).
\]
\end{proof}

\subsubsection*{Structure of $G(D)$}

\begin{lemma}
Let $L=\mathrm{rank}(G(\bar{D}))$. With probability exponentially close to $1$ as $n \to \infty$, the matrix $G(D)$ has $L$ eigenvalues on the order of $\Theta(n)$. The remaining $n-L$ eigenvalues are on the order of $o(n)$.
\label{lem:eigenval_G(D)}
\end{lemma}

\begin{proof}[Proof of Lemma \ref{lem:eigenval_G(D)}]
By Weyl's inequality and the decomposition of $G(D)$, we have the following bound for the $\ell$-th largest eigenvalue (in magnitude) of $G(D)$:
\[
\lambda_\ell (G(\bar D)) + \lambda_n (G(E)) \le \lambda_\ell(G(D)) \le \lambda_\ell (G(\bar D)) + \lambda_1 (G(E)).
\]
If $\ell \le L$, then with probability exponentially close to 1 asymptotically, we obtain
\[
\Theta(n) - o(n) \le \lambda_\ell(G(D)) \le \Theta(n) + o({n}),
\]
by Lemmas \ref{lem:E_2_norm} and \ref{lem:eigGbarDa}. We have verified the statement for the top $L$ eigenvalues of $G(D)$. Otherwise,
\[
0 - o({n}) \le \lambda_\ell(G(D)) \le 0 + o({n}).
\]
We conclude that the remaining eigenvalues are on the order of $o({n})$ as desired.
\end{proof}

As a consequence, we could heuristically determine the embedding dimension $L$, that is, $\mathrm{rank}(G(\bar D))$ by the procedure described in Section \ref{sec:MDS}.  Lemma \ref{lem:eigenval_G(D)} states that there are $L$ eigenvalues of $G(D)$ on the order of $\Theta(n)$ and $n-L$ eigenvalues on the order of $o({n})$.  More strongly, we have shown that the small eigenvalues are at most $O(n^{1-\gamma})$.  Therefore, asymptotically, a cutoff of $n^{1-\gamma/2}$ will correctly distinguish the zero from nonzero entries and determine $L$.  In practice, a heuristic that looks for a sharp dropoff is the eigenvalues should be used.


\subsubsection*{Structure of $\hat G$}
Recall that $\hat G$ was defined by \eqref{eq:hatGdef}.
\begin{lemma}
  The following two upper bounds hold for $\hat G$
  \begin{align}
    \Vert G(\bar D)-\hat G\Vert_F &\le 6c_gn^{1-\gamma},
    \label{eq:Ghatfro}\\
    \left \Vert \sqrt{|G(\bar D)|}-\sqrt{|\hat G|}\right\Vert_F &\le \frac{6c_gn^{1-\gamma}}{\sqrt{\sigma n}} =c_g'n^{1/2-\gamma}
    \label{eq:cpnh}
  \end{align}
\end{lemma}

\begin{proof}
  Note that $\Vert G(D)-\hat G\Vert_F \le \Vert G(D)- G(\bar D)\Vert_F = \Vert G(E) \Vert_F$
  since $\hat G$ is the optimal rank-$L$ approximation to $G(D)$ in the
  Frobenius norm, whereas $G(\bar D)$ is another rank-$L$ matrix.
  Thus, $\Vert G(D)-\hat G\Vert_F\le 3c_gn^{1-\gamma}$ by Lemma \ref{lem:E_2_norm}.
  Now \eqref{eq:Ghatfro} follows from applying the triangle inequality
  to the bound in the previous sentence together with Lemma \ref{lem:E_2_norm}.
\end{proof}

If $X$ is a symmetric matrix, we let $|X|$ denote the matrix-absolute value,
i.e., if $X=Q\Lambda Q^T$ then $|X|=Q\Diag(|\lambda_1|,\ldots,|\lambda_n|)Q^T$.  Similarly,
$\sqrt{X}$ denotes the matrix square root of a positive semidefinite symmetric
matrix, and $\sqrt{|X|}$ means the composition of the two operations $\sqrt{\cdot}$
and $|\cdot|$ for a symmetric matrix $X$.

The following lemma is a special case of a theorem from Bhatia \cite{bhatia2013matrix} about matrix functions.

\begin{lemma}
  For any two $n\times n$ symmetric matrices $X,Y$, $\left \Vert |X|-|Y|\right\Vert_F\le
  \Vert X-Y\Vert_F$.
  \label{lem:Bhat}
\end{lemma}
Therefore, we have:
\begin{equation}
\left \Vert |\hat G| - | G(\bar D)| \right \Vert_F \le 6c_gn^{1-\eps},
\label{eq:absdiff}
\end{equation}
due to \eqref{eq:Ghatfro}.
The following lemma is a sharper version of Bhatia's result for the special case of square root.
\begin{lemma}
  Let $X,Y$ be two $n\times n$ symmetric positive
  semidefinite matrices both of rank $k$.
  Let $\lambda_{\min^+}(X)>0$ (resp., $\lambda_{\min^+}(Y)$)
  denote the minimum positive eigenvalue of $X$ (resp., $Y$).
  Then
  \begin{equation}
  \left\Vert \sqrt{X}-\sqrt{Y}\right\Vert_F \le
  \frac{\Vert X-Y\Vert_F}
       {\min (\sqrt{\lambda_{\min^+}(X)}, \sqrt{\lambda_{\min^+}(Y)})}
       \label{eq:sqrtAB}
  \end{equation}
  \label{lem:sqrtAB}
\end{lemma}
\begin{proof}
  Let $Q_X\Lambda_XQ_X^T$ and $Q_Y\Lambda_YQ_Y^T$ be the diagonalizations
  of $X$ and $Y$ respectively, with the eigenvalues listed in decreasing
  order.  Then
  \begin{align*}
    \left \Vert \sqrt{X}-\sqrt{Y}\right\Vert_F &= \left\Vert Q_X\sqrt{\Lambda_X}Q_X^T -
    Q_Y\sqrt{\Lambda_Y}Q_Y^T \right\Vert_F \\
    & = \left \Vert\sqrt{\Lambda_X}Q_X^TQ_Y-Q_X^TQ_Y\sqrt{\Lambda_Y}\right\Vert_F,
  \end{align*}
  where the previous line follows from multiplying on the left by $Q_X^T$
  and on the right by $Q_Y$.  Let $Q=Q_X^TQ_Y$.  Then, continuing the above
  chain of equations and using the notation $\lambda_{X,i}$ (resp., $\lambda_{Y,i}$)
  for the $i$th eigenvalue of $X$ (resp., $Y$), we have

  \begin{align*}
    &\left \Vert \sqrt{X}-\sqrt{Y}\right\Vert_F^2 \\
    &=
    \left \Vert\sqrt{\Lambda_X}Q-Q\sqrt{\Lambda_Y}\right\Vert_F^2 \\
      &=\sum_{i=1}^n\sum_{j=1}^n Q(i,j)^2(\sqrt{\lambda_{X,i}}-\sqrt{\lambda_{Y,j}})^2 \\
      &=\sum_{\substack{i=1,\dots,n\\ j = 1, \dots, n \\ \min(i,j)\le k}} Q(i,j)^2(\sqrt{\lambda_{X,i}}-\sqrt{\lambda_{Y,j}})^2 \\
      &=\sum_{\substack{i=1,\dots,n\\ j = 1, \dots, n \\ \min(i,j)\le k}} Q(i,j)^2\left(\frac{\lambda_{X,i}-\lambda_{Y,j}}
      {\sqrt{\lambda_{X,i}}+\sqrt{\lambda_{Y,j}}}\right)^2 \\
      &  \le
        \frac{1}{(\min (\sqrt{\lambda_{\min^+}(X)}, \sqrt{\lambda_{\min^+}(Y)}))^2}
        \cdot
        \sum_{\substack{i=1,\dots,n\\ j = 1, \dots, n \\ \min(i,j)\le k}} Q(i,j)^2(\lambda_{X,i}-\lambda_{Y,j})^2 \\
      &=
        \frac{1}{(\min (\sqrt{\lambda_{\min^+}(X)}, \sqrt{\lambda_{\min^+}(Y)}))^2}
        \cdot
        \sum_{i=1}^n\sum_{j=1}^n Q(i,j)^2(\lambda_{X,i}-\lambda_{Y,j})^2 \\
      &=
        \frac{1}{(\min (\sqrt{\lambda_{\min^+}(X)}, \sqrt{\lambda_{\min^+}(Y)}))^2}
        \cdot
        \Vert \Lambda_XQ - Q\Lambda_Y\Vert_F^2 \\
        &=
        \frac{1}{(\min (\sqrt{\lambda_{\min^+}(X)}, \sqrt{\lambda_{\min^+}(Y)}))^2}
        \cdot
        \Vert \Lambda_XQ_X^T Q_Y - Q_X^T Q_Y \Lambda_Y\Vert_F^2 \\
      &=
        \frac{1}{(\min (\sqrt{\lambda_{\min^+}(X)}, \sqrt{\lambda_{\min^+}(Y)}))^2}
        \cdot
        \Vert X - Y\Vert_F^2.
  \end{align*}
\end{proof}

Recall
Lemma~\ref{lem:eigGbarDa}, which states that $\lambda_{\min^+}(|G(\bar D)|)\ge \sigma n$.
The same lower bound applies to $|\hat G|$ since, by
\eqref{eq:Ghatfro} and Lemma~\ref{lem:Bhat}, $|G(D)|$ and $|\hat G|$ are very
close and both have rank $L$.  Thus, applying Lemma~\ref{lem:sqrtAB}
to \eqref{eq:absdiff}
yields
\begin{equation*}
  \left \Vert \sqrt{|G(\bar D)|}-\sqrt{|\hat G|}\right\Vert_F \le \frac{6c_gn^{1-\gamma}}{\sqrt{\sigma n}} =c_g'n^{1/2-\gamma}
\end{equation*}
where $c_g'=6c_g/\sqrt{\sigma}$.  This holds with probability exponentially close to 1.

\subsection{Construction and analysis of $\bar \b_i$}
\label{sec:bar_b}

We define embeddings $\bar\b_1,\ldots,\bar\b_n$ by applying the same formulas to $\bar D$.  In other words, define the Gram matrix $G(\bar D) = \bar D - \bar D(1,:) \mathbf{1}^T - \mathbf{1} \bar D(1,:)^T$, then determine the eigendecomposition of $G(\bar D)=: \bar Q \bar \Lambda \bar Q^T$ and let $[\bar \b_1, \bar \b_2, \dots, \bar \b_n] \equiv \bar B := |\bar \Lambda_L|^{\frac{1}{2}} \bar Q_L$.
Due to the choice of $L$, we have $L=\mathrm{rank}(G(\bar{D}))$ and hence we could embed $G(\bar D)$ exactly by points $\bar \b_i \in \mathbb R^L$, i.e., $|G(\bar D)| = \bar B^T \bar B$.

Moreover, two points $i,j$ belong to the same cluster if and only if the re-embeddings of their canonical representatives $\bmu_{c(i)}, \bmu_{c(j)}$ coincide,  i.e. $\bar \b_i = \bar \b_{i'}$ if and only if $c(i)=c(i')$ as shown in the following two lemmas.

\begin{lemma}
  Suppose $i,j$ are in the same cluster.  Then $\bar\b_i=\bar\b_j$.
  \label{lem:identical_embeddings}
\end{lemma}
\begin{proof}
Since $i,j$ are in the same cluster, they
have the same entries in $\bar D$, i.e., $\bar D(i,:)=\bar D(j,:)$.   By
construction of $G(\cdot)$, this implies  $G(\bar D)(i,:)=G(\bar D)(j,:)$.
Let the eigendecomposition of $G(\bar D)$ be $G(\bar D)=\bar Q \bar \Lambda \bar Q^T$, in
which only the first $L$ diagonal entries of $\bar \Lambda$ are nonzero.
Then $G(\bar D) \bar Q= \bar Q \bar \Lambda$.  Since rows $i,j$ of $G(\bar D)$ are identical,
so are rows $i,j$ of $G(\bar D) \bar Q$, i.e., rows $i$ and $j$ of
$\bar Q \bar \Lambda=\bar Q_L\bar\Lambda_L$.  Since $\bar \Lambda_L$ is diagonal and invertible, this
implies that $\bar Q_L(i,:)= \bar Q_{L}(j,:)$.

It suffices to prove that columns $i$ and $j$ of $\bar B=\sqrt{|\bar \Lambda_L|} \bar Q_L^T$ coincide due to the construction of $\bar \b_i, \bar \b_j$.
  Observe that columns $i$ and $j$ of $Q_L^T$ agree by the result at the end
  of the last paragraph.  Multiplying on the
  left by $\sqrt{|\bar{\Lambda}_L|}$ does not alter this agreement,
  so we conclude that columns $i$ and $j$
  of $\bar B$ coincide.
\end{proof}

\begin{lemma}
  Suppose $i,j$ are in different clusters.  Then,
  with probability exponentially close to $1$, $\Vert \bar \b_i-\bar \b_j\Vert\ge c$,
  where $c>0$  is independent of $n$.
  \label{lem:bbar_inter_dist}
\end{lemma}
\begin{proof}
  As in the preceding proof, assume $G(\bar D)=\bar Q \bar \Lambda \bar Q^T$ with the
  $L$ nonzero diagonal entries of $\bar \Lambda$ listed first.

  Let $\delta:=\min\{\mathrm{dist}(S_k,S_k'):1\le k<k'\le K\}$.  Note $\delta>0$ by asssumptions.
  Suppose $i\in C_1$ and $j\in C_2$.
  Consider the entries of row $i$ of $\bar D$ corresponding to $C_2$; these are at
  least $\delta^4$
  (squared once because of the definition of LF distance; squared again because
  $\bar D$ has squared distances).  On the other hand, these entries in row $j$ of $\bar D$
  are all 0's.
  Conversely, the entries corresponding to $C_1$ in row $j$ are at least
  $\delta^4$, while they are 0 in row $i$.
  By the Chernoff bound,
    with probability exponentially close to 1,
    $|C_k|\ge \beta n$ for all $k=1,\ldots,K$, where
    \begin{equation}
    \beta:=(2/3)\min(p_1,\ldots,p_K),
    \label{eq:betadef}
    \end{equation}
    with $p_1,\ldots,p_K$ given by \eqref{eq:pmdef}.
  Thus, there are at least $\beta n$
  entries of row $i$ exceeding the corresponding entries of row $j$ by $\delta^4$,
  and vice versa.

  In forming $G(\bar D)$
  from $\bar D$ we subtract the same
  row from every row, which does not affect the relationship in the last paragraph.
  We also subtract $\bar D(i,1)$ from each entry of row $i$ and $\bar D(j,1)$ from each entry
  of row $j$.  Depending on whether $\bar D(i,1)> \bar D(j,1)$, either one relationship
  or the other in the preceding paragraph still holds.  Thus, we
  conclude that, either way,
  $$\Vert G(\bar D)(i,:)-G(\bar D)(j,:)\Vert \ge \sqrt{\beta n}\delta^4.$$
  Since $\bar Q$ is an isometry, it follows that rows $i$ and $j$ of
  $G(\bar D)\bar Q$, and hence of $\bar Q \bar \Lambda$, differ by at least this amount.
  Note that all columns of $\bar Q \bar \Lambda$ after $L$ are zeros, so we can
  drop those columns and conclude
  that rows $i$ and $j$ of $\bar Q_L \bar \Lambda_L$ differ by at least
  $\sqrt{\beta n}\delta^4.$  Since the maximum absolute entry of
  $\bar \Lambda_L$ is $\tau n$,
  we conclude that rows $i$ and $j$ of $\bar Q_L$ differ by at least
  $\sqrt{\beta n}\delta^4/(\tau n)=\sqrt{\beta}\delta^4/(\tau\sqrt{n})$ in Euclidean distance.

  As in the previous proof, write the coordinates as
  columns of $$\bar B = \sqrt{|\bar \Lambda_L|} \bar Q_L^T.$$  We just showed
  that columns $i$ and $j$ of $\bar Q_L^T$ differ by at least
  $c'/\sqrt{n}$, where $c'=\sqrt{\beta}\delta^4/\tau$.  Furthermore, the
  minimum entry of $\sqrt{|\bar\Lambda_L|}$ is at least $\sqrt{\sigma n}$
  by Lemma~\ref{lem:eigGbarDa}.  Thus, the minimum distance between
  these two columns of
  $\sqrt{|\bar\Lambda_L|}\bar Q_L^T$ is at least
  $\sqrt{\sigma n}\cdot c'/\sqrt{n}$, which simplifies to $\Omega(1)$ independent
  of $n$.
\end{proof}

\subsection{Properties of $\b_i$}
\label{sec:b}

The geometry of the $\bar \b_i$'s is desirable for clustering: points in the same cluster coincide and distinct clusters are well separated. In this case,  one can find the perfect clustering by inspection. However, the algorithm has access only to the $\b_i$'s rather than the $\bar\b_i$'s.  Furthermore, it is not possible to claim that $\Vert B-\bar B\Vert$ is small because the $Q$ factor in an eigendecomposition is not stable under small perturbations.  Therefore, more elaborate arguments are required.

\begin{manualtheorem}{\ref{thm:intra_dist_b_disjoint}} [Intracluster distance.]
Suppose $i,i'\in C_k$ for some $k=1,\ldots,K$. Then with probability exponentially close to $1$ as $n \to \infty$, there holds $\norm{\b_i - \b_{i'}}_2 \le \frac{6 c_g \sqrt{L}}{\sqrt{\sigma}} n^{-\gamma}  = o(1)$, where $c_g, \gamma, \sigma$ do not depend on $n$ and are defined by \eqref{eq:cgdef} and \eqref{eq:sigmadef}.
\end{manualtheorem}


\begin{proof}
Recall $B=|\Lambda_L|^{1/2}Q_L^T$, where $Q\Lambda Q^T$ is the eigendecomposition of $G(D)$, $Q_L=Q({:,1:L})$ and $\Lambda_L=\mathrm{Diag}(\lambda_1,\ldots,\lambda_L)$, where $\lambda_1,\ldots,\lambda_L$ are the eigenvalues of $G(D)$ with the largest magnitudes.
We can write the $\ell$th entry 
explicitly:
$$
\b_i(\ell) = \sqrt{|\lambda_\ell|}\q_\ell(i), \quad \b_{i'}(\ell) = \sqrt{|\lambda_\ell|}\q_\ell(i'), \quad \forall \ell \in [L].
$$
Using the properties of eigenpairs, one can rewrite $\b_i$ as follows:
\begin{equation}
    \begin{aligned}
\b_i(\ell) &= \sqrt{|\lambda_\ell|}\q_\ell(i) \\
&= \frac{\text{sign}(\lambda_\ell)}{\sqrt{|\lambda_\ell|}}\sum_j G(D)(i,j) \q_\ell(j)\\
&=\frac{\text{sign}(\lambda_\ell)}{\sqrt{|\lambda_\ell|}}\sum_j G(D)(i',j) \q_\ell(j) + \frac{\text{sign}(\lambda_\ell)}{\sqrt{|\lambda_\ell|}}\sum_j (G(D)(i,j) - G(D)(i',j)) \q_\ell(j)\\
&=\sqrt{|\lambda_\ell|}\q_\ell(i')  + \frac{\text{sign}(\lambda_\ell)}{\sqrt{|\lambda_\ell|}}\sum_j (G(D)(i,j) - G(D)(i',j)) \q_\ell(j) \\
&= \b_{i'}(\ell) + \underbrace{\frac{\text{sign}(\lambda_\ell)}{\sqrt{|\lambda_\ell|}}\sum_j (G(D)(i,j) - G(D)(i',j)) \q_\ell(j)}_{T_\ell}\\
\end{aligned}
\label{eq:bil-bil'}
\end{equation}
Recall the construction of $G(\bar D)$, we have $G(\bar D)(i,:) = G(\bar D)(i',:)$ as $i,i'$ lie in the same support. Hence, we could further upper-bound $T_\ell$ by adding $G(D)(i,j) - G(D)(i',j)$ to each term inside the summation.
\begin{align*}
    T_\ell &= \frac{\text{sign}(\lambda_\ell)}{\sqrt{|\lambda_\ell|}}\sum_j (G(D)(i,j) - G(\bar D)(i,j) + G(\bar D)(i',j) - G(D)(i',j)) \q_\ell(j)\\
    &=\frac{\text{sign}(\lambda_\ell)}{\sqrt{|\lambda_\ell|}}\sum_j (G(E)(i,j) - G(E)(i',j)) \q_\ell(j)\\
    &\le \frac{1}{\sqrt{|\lambda_\ell|}} \norm{G(E(i,:)-G(E(i',:)}_{2} \norm{\q_\ell}_2 \\
    &\le \frac{2}{\sqrt{|\lambda_\ell|}} \norm{G(E)}_{2,\infty} \norm{\q_\ell}_2 \\
    &\le \frac{2}{\sqrt{\sigma n}} \cdot 3 c_g n^{1/2-\gamma}\\
    &= \frac{6 c_g }{\sqrt{\sigma}} n^{-\gamma}
\end{align*}
where the last few lines are due to the Cauchy-Schwarz and triangle inequalities, and the last inequality is due to Lemma \ref{lem:E_2_norm}.
Hence, the difference between $\b_i(\ell),\b_{i'}(\ell)$ is bounded as follows:
$$
\norm{\b_i - \b_{i'}}_2 = \sqrt{\sum_{\ell=1}^L T_\ell^2} \le \frac{6 c_g \sqrt{L}}{\sqrt{\sigma}} n^{-\gamma} = o(1).
$$
We still get $o(1)$ as the dimension of the $\b_i$'s and $\bar\b_i$'s is $L$, which is independent of $n$.
\end{proof}

\begin{manualtheorem}{\ref{thm:inter_dist_X_disjoint}}
[Intercluster distance.]
Suppose $i \in C_m, j\in C_{m'}$ with $m\ne m'$. Then with probability exponentially close to $1$ as $n \to \infty$, there holds $\norm{\b_i - \b_{j}}_2 = \Omega(1)$.
\end{manualtheorem}

\begin{proof}
Before we dive into the proof, let us recall that the construction of $B=|\Lambda_L|^{1/2} Q_L$ and $\bar B=|\bar \Lambda_L|^{1/2} \bar Q_L$ implies:
\[
B^T B = |\hat G| = |\hat G|^{1/2} |\hat G|^{1/2}, \qquad \bar B^T \bar B = |G(\bar D)| = |G(\bar D)|^{1/2} |G(\bar D)|^{1/2},
\]
\newcommand{\Bpadzero}[1]{\left(\begin{array}{c} #1 \\0\end{array}\right)}
which implies the existence of $n\times n$ orthogonal matrices $R, \bar R$ such that
\[
\Bpadzero{B} = R |\hat G|^{1/2}
\quad\mbox{and} \quad
\Bpadzero{\bar B} = \bar R |G(\bar D)|^{1/2},
\]
where the two blocks of zeros are $(n-L)\times n$.  Applying inequality \eqref{eq:cpnh} yields
\begin{equation}
    \left\Vert \Bpadzero{B} - R^T \bar R \Bpadzero{\bar B} \right\Vert_F = \left\Vert R\Bpadzero{B} - \bar R \Bpadzero{\bar B} \right\Vert_F = \left \Vert |G(\bar D)|^{1/2}-|\hat G|^{1/2}\right\Vert_F \le c_g' n^{1/2-\gamma}.
    \label{eq:B-barB}
\end{equation}
Without loss of generality, take $m=1$ and $m'=2$ in the Theorem.  Suppose $i$ lies in $C_1$ and $j$ lies in $C_2$. With probability exponentially close to 1 asymptotically, there are at least $\beta n$ data points in each of $C_1$ and $C_2$, where $\beta$ is defined in \eqref{eq:betadef}. Pick an arbitrary subset of $beta n$ indices from $C_1, C_2$ and call these subsets $F_1, F_2$ respectively. Let $B_{F_1}$ and $B_{F_2}$ denote the columns of $B$ indexed by $F_1, F_2$. As a consequence of \eqref{eq:B-barB}, we obtain
\begin{equation}
    \left\Vert \Bpadzero{B_{F_k}} - R^T \bar R \Bpadzero{\bar B_{F_k}} \right\Vert_F \le \left\Vert \Bpadzero{B} - R^T \bar R \Bpadzero{\bar B} \right\Vert_F \le c_g' n^{1/2-\gamma}, \quad k = 1,2.
    \label{eq:B_S1-barB_S1}
\end{equation}
Moreover, we have proven that $\Vert\bar \b_i - \bar \b_j\Vert > c$ for some $c$ that does not depend on $n$ for each pair $(i,j)\in C_1\times C_2$ in Lemma \ref{lem:bbar_inter_dist}, which implies
\begin{equation}
    \left\Vert R^T \bar R \Bpadzero{\bar B_{F_1}} - R^T \bar R \Bpadzero{\bar B_{F_2}} \right\Vert_F = \left\Vert \Bpadzero{\bar B_{F_1}} - \Bpadzero{\bar B_{F_2}} \right\Vert_F > c\sqrt{\beta n}.
    \label{eq:barB_S1-barB_S2}
\end{equation}
Adding and subtracting
\[
R^T \bar R \Bpadzero{\bar B_{F_1}} - R^T \bar R \Bpadzero{\bar B_{F_2}},
\]
we obtain
\begin{align*}
    \Vert B_{F_1} - B_{F_2} \Vert_F &=
    \left\Vert \Bpadzero{B_{F_1}} - \Bpadzero{B_{F_2}}\right\Vert_F
    \\
    &=
    \left\Vert \Bpadzero{B_{F_1}} - R^T \bar R
    \Bpadzero{\bar B_{F_1}} + R^T \bar R
    \Bpadzero{\bar B_{F_1}} - R^T \bar R
    \Bpadzero{\bar B_{F_2}} \right. \\
    &\qquad\left.\mbox{} + R^T \bar R
    \Bpadzero{\bar B_{F_2}}- \Bpadzero{B_{F_2}} \right\Vert_F\\
    &\ge \left\Vert R^T \bar R \Bpadzero{\bar B_{F_1}} - R^T \bar R \Bpadzero{\bar B_{F_2}} \right\Vert_F -\left\Vert\Bpadzero{ B_{F_1}} - R^T \bar R \Bpadzero{\bar B_{F_1}} \right\Vert_F \\
    &\qquad\mbox{}- \left\Vert R^T \bar R \Bpadzero{\bar B_{F_2}} - \Bpadzero{B_{F_2}} \right\Vert_F\\
    &> c\sqrt{\beta n} - 2c_g' n^{1/2-\gamma},
\end{align*}
where the last inequality is due to \eqref{eq:B_S1-barB_S1} and \eqref{eq:barB_S1-barB_S2}. Squaring both sides of the inequality, we have
\[
\Vert B_{F_1} - B_{F_2} \Vert_F^2 > c'\beta n,
\]
for $c'=c/2$ (valid as $n \to \infty$). Since there are $\beta n$ columns in $B_{F_1} - B_{F_2}$, then there exist some $i' \in F_1 \subseteq{C_1}, j' \in F_2 \subseteq{C_2}$ such that $\Vert \b_{i'} - \b_{j'} \Vert_2 > c'$. Coupled with the intra-cluster distance (Theorem \ref{thm:intra_dist_b_disjoint}), there also holds $\Vert \b_{i} - \b_{i'}\Vert_2 \le\frac{6 c_g \sqrt{L}}{\sqrt{\sigma}} n^{-\gamma}$ for $i,i'\in C_1$. Adding and subtracting $\b_{i'} - \b_{j'}$ from $\b_{i} - \b_{j}$, we obtain
\begin{align*}
\Vert \b_{i} - \b_{j} \Vert_2 &= \Vert \b_{i} - \b_{i'} +  \b_{i'} - \b_{j'} + \b_{j} - \b_{j} \Vert_2\\
&\ge \Vert \b_{i'} - \b_{j'} \Vert_2 - \Vert \b_{i} - \b_{i'}\Vert_2 - \Vert \b_{j} - \b_{j} \Vert_2  \\
&\ge c' - \frac{12 c_g \sqrt{L}}{\sqrt{\sigma}} n^{-\gamma}.
\end{align*}
Asymptotically, we get $\Vert \b_{i} - \b_{j} \Vert_2 = \Omega(1)$ as desired.
\end{proof}

\section{Derivation of Recovery results}
\label{app:recovery}
\subsection{Recovery of non-convex clusters on disjoint supports (proof of theorem \ref{thm:disjoint_recovery})}
\label{app:recovery_disjoint}
\begin{manualtheorem}{\ref{thm:disjoint_recovery}}
Suppose data $\a_1, \dots, \a_n$ are independent and identically distributed with a common law $f$, which is admissible and supported on the disjoint union of path-connected components $S_1, \dots, S_K$.
Then there exists $\lambda$ such \eqref{eq:son-clustering}
applied to the re-embeddings $\b_1, \dots, \b_n$
achieves perfect recovery of clusters $C_1, \dots, C_K$,
where $C_k=\{i:\a_i\in S_k\}$ for $k=1,\ldots,K$,
with probability exponentially close to $1$ as $n\rightarrow \infty$. 
\end{manualtheorem}


\begin{proof}
By Theorems \ref{thm:intra_dist_b_disjoint} and \ref{thm:inter_dist_X_disjoint}, we have the following properties for intra-cluster and inter-cluster distances of the re-embeddings $\b_1, \dots, \b_n$:
\begin{align*}
\norm{\b_i-\b_j} &= o(1) \; \forall i,j \in C_k, \\
 \norm{\b_i-\b_j} &= \Omega(1) \; \forall i\in C_k, j \notin C_k,
\end{align*}
for all  $k = 1, \dots, K$.
Thus, the upper bound of $\lambda$ stated in Equation \eqref{eq:son_lambda_ub} is on the order of $\frac{\norm{
\b_i - \b_j}}{2(n-1)} = \Omega\left(\frac{1}{n}\right)$.  Moreover, as argued in the proof of Lemma \ref{lem:bbar_inter_dist}, with probability exponentially close to 1 as $n \to \infty$, each cluster has size at least $\gamma n$ for some $\gamma > 0$. Thus, the lower bound of $\lambda$ stated in \eqref{eq:son_lambda_lb} is on the order of $\frac{\norm{\b_i - \b_j}}{|C_k|} = o\left(\frac{1}{n}\right)$.
Therefore,
\[
\frac{\norm{\b_i - \b_j}}{|C_k|} < \frac{\norm{\b_i - \b_j}}{2(n-1)},
\]
which implies the existence of $\lambda$ that satisfies the lower and upper bounds stated in \eqref{eq:son_lambda_lb} and \eqref{eq:son_lambda_ub}. As the existence of such $\lambda$ is a sufficient condition for successful clustering as stated in Theorem \ref{thm:recovery}, we obtain perfect recovery of clusters $C_1, \dots, C_K$.
\end{proof}

\subsection{Recovery of mixture of Gaussians in 1D (proof of Theorem \ref{thm:1DGaussiansRecovery})}
\label{app:recovery_gaussian}
\begin{manualtheorem}{\ref{thm:1DGaussiansRecovery}}
Let the vertices $a_1,\ldots,a_n$ be chosen i.i.d.~according to a Lipschitz-continuous probability density function $f(x)$ that is positive for all $x\in\R$.  Assume SON clustering is applied to this data after re-embedding according to new coordinates $b_1,\ldots,b_n$. Let $S_1$, \ldots, $S_K$ be disjoint bounded closed intervals ordered from left to right.
 Let $\rho_m$, $m=1,\ldots,K$, denote $\int_{S_m}f(x)\,dx$.
For a particular $m\in\{1,\ldots,K\}$, for any $\eps>0$, all the data points in $S_m$ will be clustered together with probability exponentially close to $1$ as $n\rightarrow\infty$ (that also depends on $\epsilon$) provided that
\begin{equation}
\lambda \ge      \frac{2\int_{S_m}(1/f(x))\,dx}{(\rho_m-\epsilon)n^2} + O(n^{-2.04}).
\label{eq:lambda_lb_a}
\end{equation}
Furthermore, the cluster associated with $S_m$ is distinct from the cluster associated with $S_{m'}$, $1\le m<m'\le K$, with probability exponentially close to $1$ as $n\rightarrow\infty$  provided that
\begin{equation}
\lambda < \frac{\min_{m=1,\ldots,K-1}\int_{T_m}(1/f(x))\,dx}{n^2} + O(n^{-2.04}),
    \label{eq:lambda_ub_a}
\end{equation}
where $T_m$ for $m\in\{1,\ldots,K-1\}$ denotes the interval comprising the gap between $S_m$ and $S_{m+1}$, i.e., $T_m=[\max_{x\in S_m} x, \min_{x\in S_{m+1}} x]$.
\end{manualtheorem}

\begin{proof}
Let $V_m:=\{i:a_i\in I_m\}$ for $m=1,\ldots,K$.
In this proof, we will show that the two bounds in Theorem \ref{thm:1DGaussiansRecovery} satisfy both \eqref{eq:son_lambda_lb} and \eqref{eq:son_lambda_ub}, which verifies the sufficiency of Theorem \ref{thm:recovery} for the recovery of a mixture of Gaussians.

Let $\epsilon>0$ be fixed.  Fix an $m\in\{1,\ldots,K\}$.
First, we show that all the points indexed by $V_m$ are in the same
cluster.
By definition of $b_i$'s, we have that with exponentially high probability for $a_i,a_j\in V_m$, $a_i<a_j$,
\begin{align}
|b_i-b_j| &=
\mathrm{LF}(a_i,a_j)\notag\\
&=\frac{2}{n}\int_{a_i}^{a_j}\frac{dx}{f(x)} +O(n^{-1.04})
\notag \\
&\le\frac{2}{n}\int_{S_m}\frac{dx}{f(x)} +O(n^{-1.04}),
\label{eq:ED_same_cluster}
\end{align}
due to Theorem \ref{thm:ELF_U} and \ref{thm:LF_main}. Furthermore, with probability exponentially close to 1 that depends also on $\epsilon$, we have
\[
|V_m|\ge (\rho_m-\eps)n,
\]
by the Chernoff bound and Lemma~\ref{lemma:Chernoff}, since the points lie in $S_1,\ldots,S_K$ according to a multinomial distribution. Thus, by Theorem \ref{thm:recovery}, provided
$$\lambda \ge
     \frac{ (2/n)\int_{S_m}(1/f(x))\,dx+ O(n^{-1.04})}{(\rho_m-\epsilon)n}
     $$
the points indexed by $V_m$ are in the same cluster with probability exponentially close to 1 as $n \to \infty$.

For the second part of the theorem,
let $a_i,a_j$ be data points in $S_m$, $V_{m'}$ for some $m<m'$.  Then with probability exponentially close to 1,
\begin{align*}
\mathrm{LF}(a_i,a_j)&=\frac{2}{n}\int_{a_i}^{a_j}\frac{dx}
{f(x)} + O(n^{-1.04}) \\
&\ge \frac{2}{n}\int_{T_{m}\cup\cdots\cup T_{m'-1}}
\frac{dx}{f(x)}+O(n^{-1.04})
\end{align*}
Therefore, by \eqref{eq:lambda_ub_a}, with probability exponentially close to 1, $a_i$ and $a_j$ will be in different clusters provided
\[
\lambda \le \frac{(2/n)\min_{m=1,\ldots,K-1}\int_{T_m}(1/f(x))\,dx}{2n} + O(n^{-2.04})
\]
where we have replaced $n-1$ by $n$ in the denominator since the difference is smaller than the remainder term.
\end{proof}

\subsection{Remarks and setup of numerical comparison on $\lambda$}
It is important to make two observations about this testing procedure.
First, neither the bounds on $\lambda$ of Theorem~\ref{thm:1DGaussiansRecovery} nor those in \cite{jiang2020recovery} are tight.  Therefore, we are comparing only sufficient conditions rather than exact conditions.  This is in keeping with our goal for this section to show that re-embedding strengthens the guarantees on the clusters (rather than the actual computed clusters).  For results on determination of actual clusters, we carry out experiments on data, which is the subject of Section \ref{sec:exper}.

Second, the bounds here as well as the bounds in \cite{jiang2020recovery} hold only for large $n$, and the computations in this section do not provide insight into valid values of $n$.

For a Gaussian mixture model with $S_m$ defined as above, we have the following lower bound of the quantity $\rho_m$ appearing in \eqref{eq:lambda_lb}:
\begin{align*}
\rho_m &= \int_{S_m}f(x)\,dx \\
&=\sum_{p=1}^K \frac{w_p}{\sigma_p\sqrt{2\pi}}\int_{S_m}\exp(-(x-\mu_p)^2/(2\sigma_p^2))\,dx\\
&\ge  \frac{w_m}{\sigma_m\sqrt{2\pi}}\int_{S_m}\exp(-(x-\mu_m)^2/(2\sigma_m^2))\,dx \\
&=\frac{w_m}{\sigma_m\sqrt{2\pi}}\int_{\mu_m-\theta\sigma_m}^{\mu_m+\theta\sigma_m}\exp(-(x-\mu_m)^2/(2\sigma_m^2))\,dx \\
&=\frac{w_m}{\sqrt{2\pi}}\int_{-\theta}^{\theta}\exp(-y^2/2)\,dy \\
&=w_m\mathrm{erf}(\theta).
\end{align*}

In our computations, instead of evaluating the quantity $\rho_m$ appearing in \eqref{eq:lambda_lb} with numerical quadrature, we use the lower bound of  $\rho_m$ above
in order to be consistent with the derivation of the bound in \cite{jiang2020recovery}.
On the other hand, we use numerical quadrature to evaluate the integrals of $1/f(x)$ since there are no obvious good upper and lower bounds on this quantity.  Finally, we take $\epsilon=0$ since we are interested only at asymptotic ranges.

\section{Additional figures}
\label{app:figures}

\subsection{Sum-of-norms Clustering}
\label{app:figure_son}
\begin{figure}[H]
\subfloat[][Original points]{
\includegraphics[scale = 0.5]{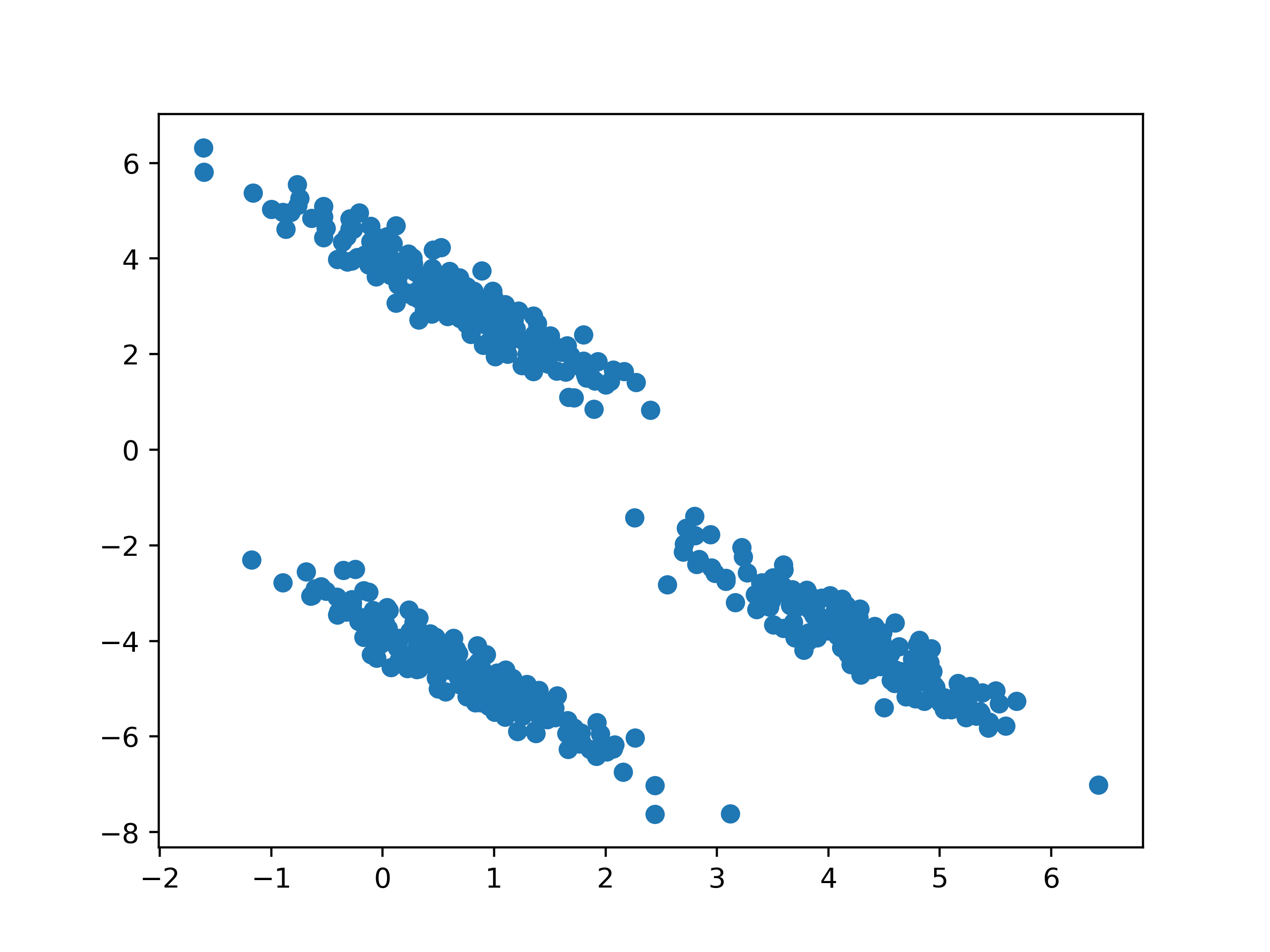}}
\subfloat[][Embedded points]{
\includegraphics[scale = 0.5]{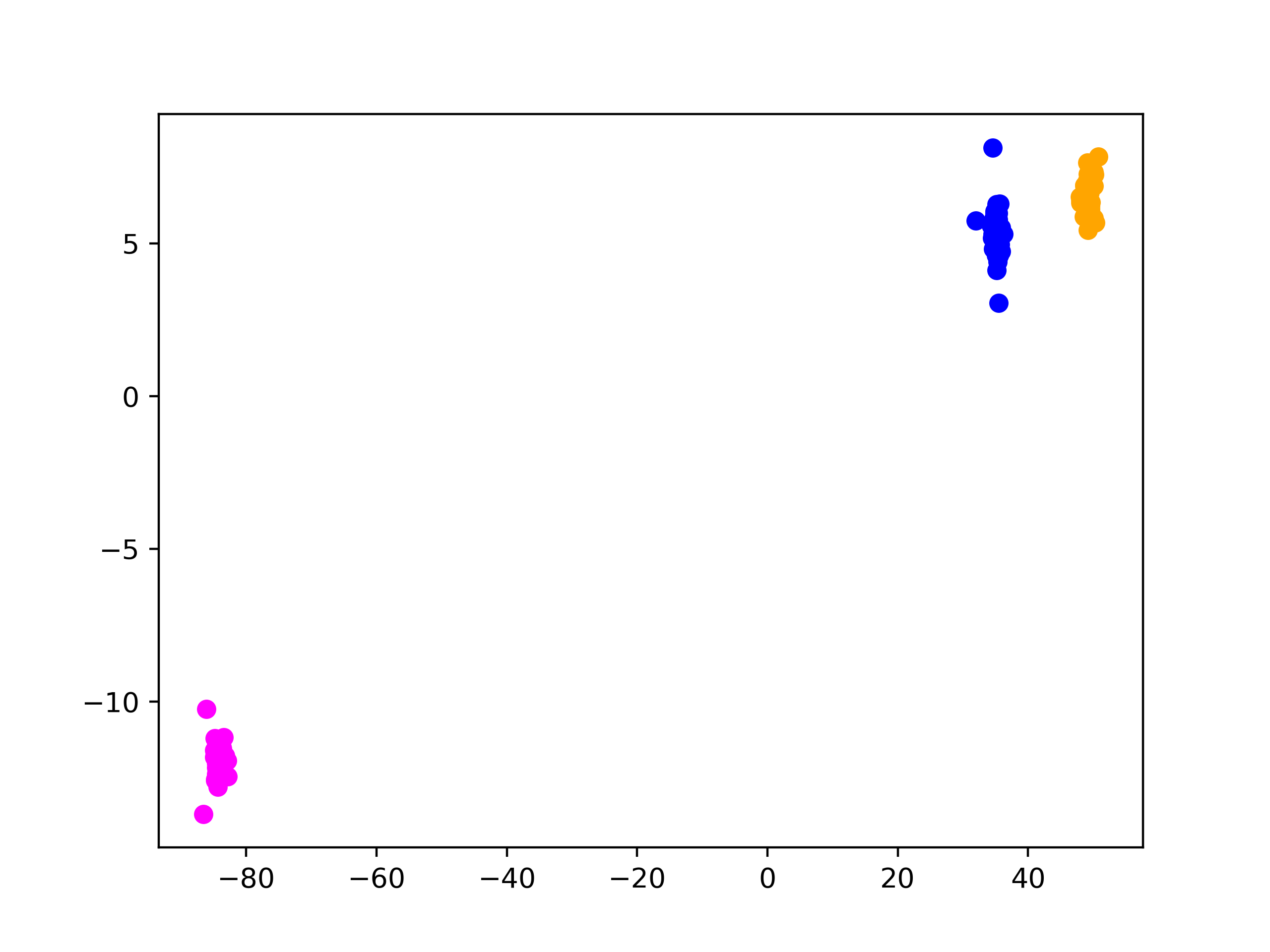}}
\caption{We apply the linear transformation $\begin{pmatrix} 0.6 & -0.6 \\ -0.4 & 0.8 \end{pmatrix}$ to the \texttt{make\_blobs} function in \texttt{sci-kit learn} with 600 points. (a) The original blobs in 2D space. (b) The embedded points, colored by the clusters as determined via sum-of-norms in this space. Perfect recovery is achieved.}
\label{fig:aniso}
\end{figure}

\begin{figure}[H]
\subfloat[][Original points]{
\includegraphics[scale = 0.5]{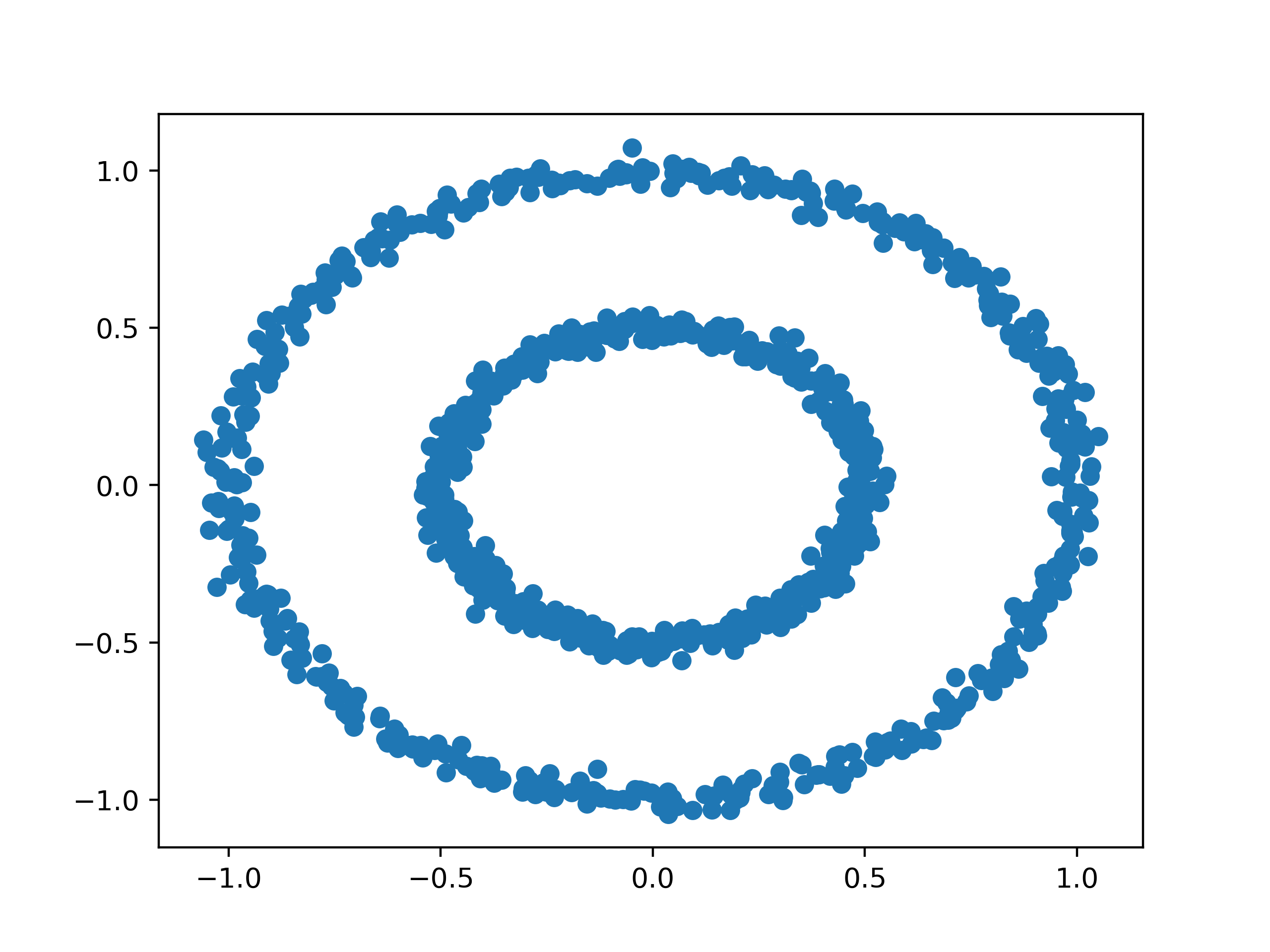}}
\subfloat[][Embedded points]{
\includegraphics[scale = 0.5]{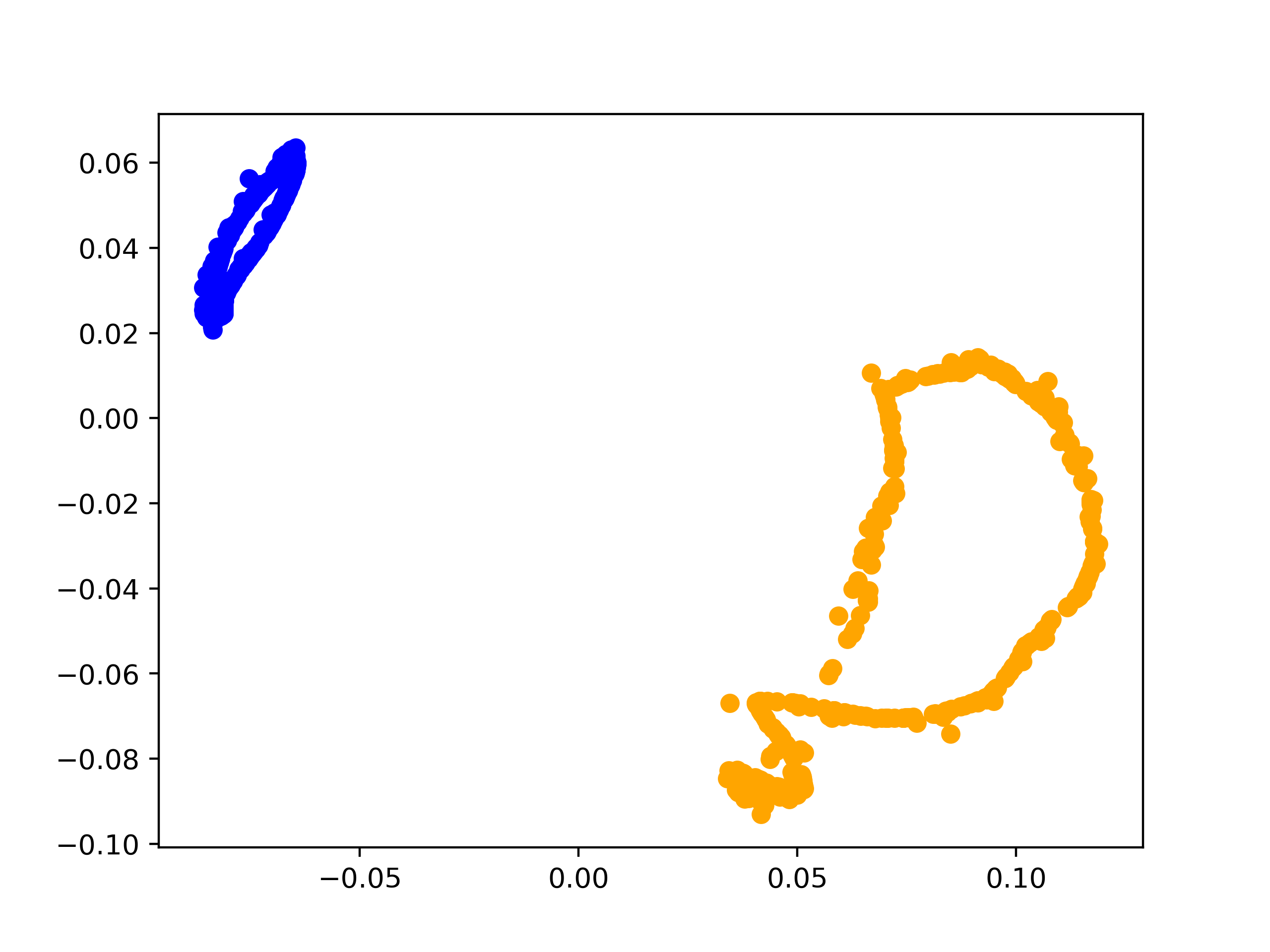}}
\caption{We use the \texttt{make\_circles} function in \texttt{sci-kit learn} with 1000 points, \texttt{noise = 0.025}, \texttt{factor = 0.5}. (a) The original moons in 2D space. (b) The re-embedded points, colored by the clusters as determined via sum-of-norms in this space. Perfect recovery is achieved.}
\label{fig:circ}
\end{figure}

\subsection{Agglomerative Clustering}
\label{app:figure_agg_clustering}
\begin{figure}[H]
\subfloat[][Original points]{
\includegraphics[scale = 0.5]{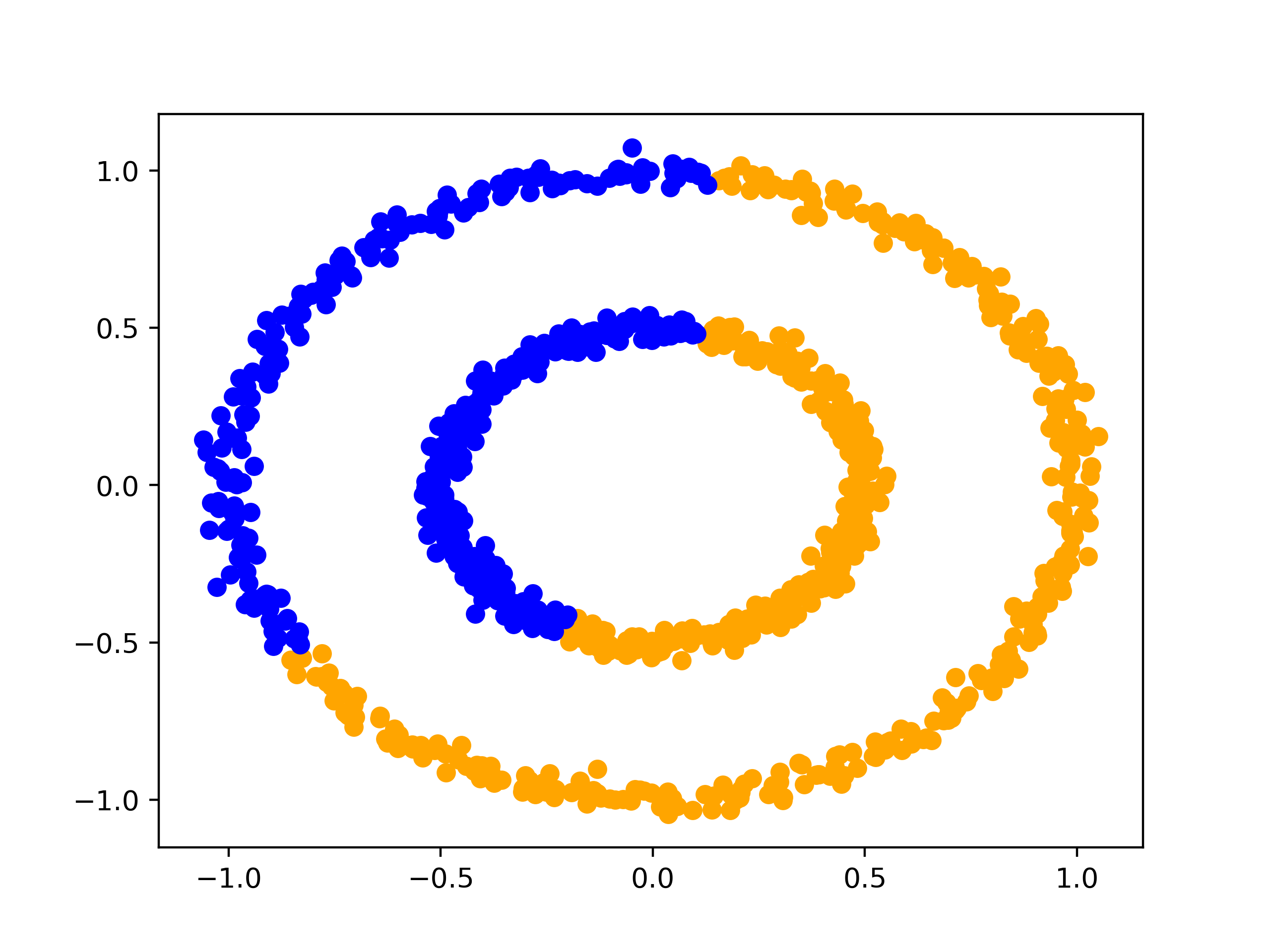}}
\subfloat[][Embedded points]{
\includegraphics[scale = 0.5]{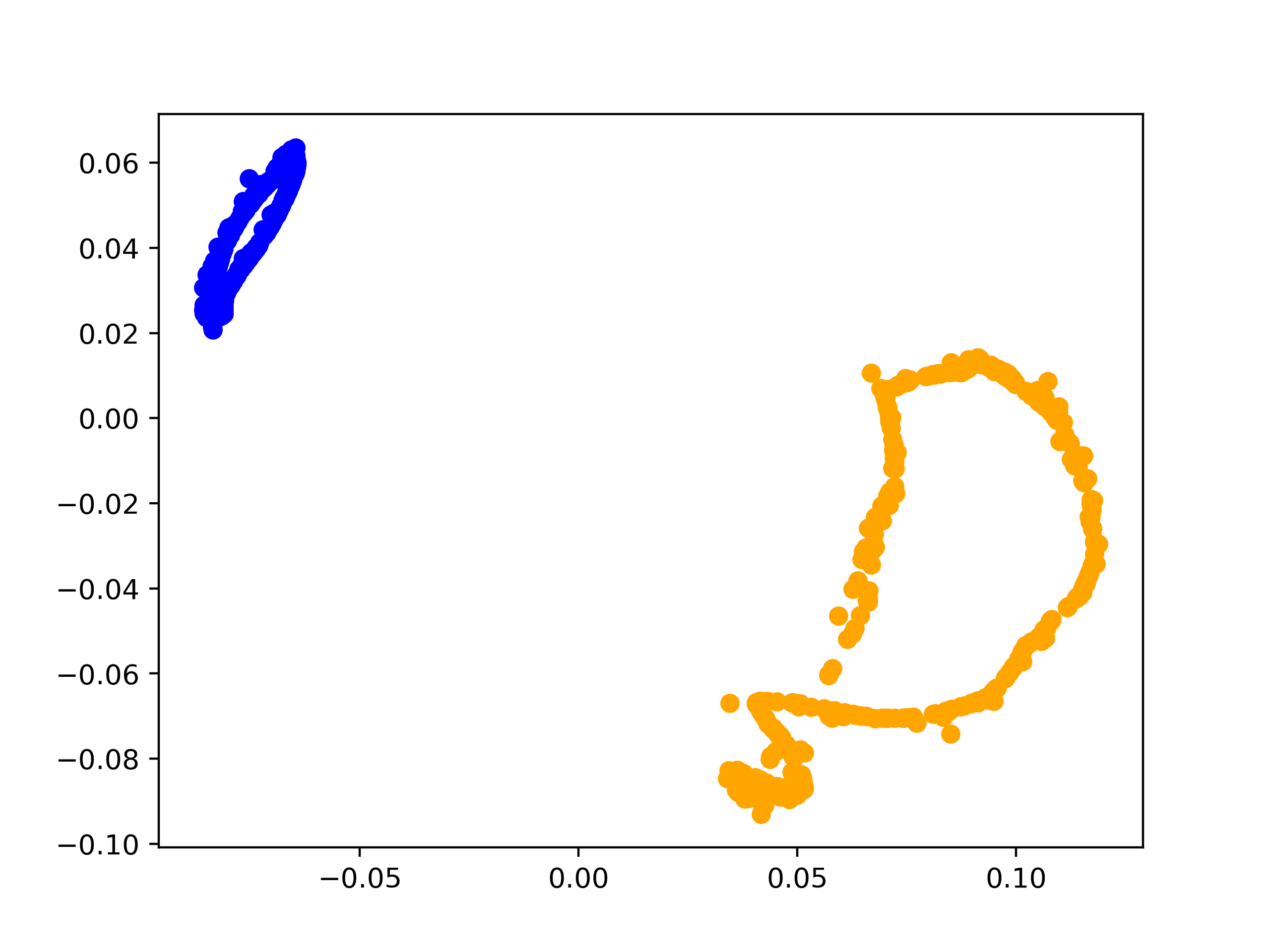}}
\caption{The same circles dataset as in Figure \ref{fig:circ}, but we use agglomerative clustering as implemented in \texttt{scikit-learn}. (a) The original circles in 2D space, with colors denoting the two clusters. (b) The re-embedded points, colored by the clusters as determined via sum-of-norms in this space. Perfect recovery is achieved.}
\label{fig:circ_agg}
\end{figure}

\end{document}